\documentclass[10pt]{article}

\usepackage[utf8]{inputenc}
\usepackage[T1]{fontenc}

\usepackage{epsf}
\usepackage{amsmath}

\allowdisplaybreaks

\usepackage[showframe=false]{geometry}
\usepackage{changepage}

\usepackage{epsfig}
\usepackage{amssymb}

\usepackage{amsthm}
\usepackage{setspace}
\usepackage{cite}
\usepackage{mcite}

\usepackage{algorithmic}  
\usepackage{algorithm}

\usepackage{shadow}
\usepackage{fancybox}
\usepackage{fancyhdr}

\usepackage{color}
\usepackage[usenames,dvipsnames,svgnames,table]{xcolor}
\newcommand{\bl}[1]{\textcolor{blue}{#1}}
\newcommand{\red}[1]{\textcolor{red}{#1}}

\definecolor{mypurple}{rgb}{.4,.0,.5}

\usepackage[hyphens]{url}

\usepackage[colorlinks=true,
            linkcolor=black,
            urlcolor=blue,
            citecolor=purple]{hyperref}

\usepackage{breakurl}

\def\y{{\bf y}}

\def\x{{\bf x}}

\def\x{{\mathbf x}}

\def\x{{\bf x}}
\def\y{{\bf y}}
\def\z{{\bf z}}
\def\q{{\bf q}}
\def\k{{\bf k}}

\def\a{{\bf a}}
\def\b{{\bf b}}
\def\c{{\bf c}}

\def\cB{{\cal B}}

\def\be{\begin{equation}}
\def\ee{\end{equation}}
\def\ba{\left[\begin{array}}
\def\ea{\end{array}\right]}

\def\x{{\bf x}}
\def\y{{\bf y}}
\def\z{{\bf z}}
\def\q{{\bf q}}
\def\a{{\bf a}}
\def\b{{\bf b}}
\def\c{{\bf c}}

\def\p{{\bf p}}

\def\1{{\bf 1}}

\def\G{{\bf G}}

\def\g{{\bf g}}
\def\0{{\bf 0}}

\def\erfc{\mbox{erfc}}

\def\mR{{\mathbb R}}
\def\mZ{{\mathbb Z}}

\def\mS{{\mathbb S}}
\def\mB{{\mathbb B}}
\def\mP{{\mathbb P}}

\def\lp{\left (}
\def\rp{\right )}

\def\y{{\bf y}}

\def\x{{\bf x}}

\def\x{{\mathbf x}}

\def\x{{\bf x}}
\def\y{{\bf y}}
\def\z{{\bf z}}
\def\q{{\bf q}}
\def\a{{\bf a}}
\def\b{{\bf b}}
\def\c{{\bf c}}

\def\cB{{\cal B}}

\def\be{\begin{equation}}
\def\ee{\end{equation}}
\def\ba{\left[\begin{array}}
\def\ea{\end{array}\right]}

\def\x{{\bf x}}
\def\y{{\bf y}}
\def\z{{\bf z}}
\def\q{{\bf q}}
\def\a{{\bf a}}
\def\b{{\bf b}}
\def\c{{\bf c}}

\def\p{{\bf p}}

\def\({\left (}
\def\){\right )}

\def\1{{\bf 1}}

\def\q{{\bf q}}

\def\G{{\bf G}}

\def\g{{\bf g}}
\def\0{{\bf 0}}

\def\cX{{\mathcal X}}

\def\cS{{\mathcal S}}
\def\cI{{\mathcal I}}
\def\cZ{{\mathcal Z}}

\def\cX{{\mathcal X}}
\def\cQ{{\mathcal Q}}

\usepackage{xcolor}
\usepackage{color}

\definecolor{darkgreen}{rgb}{0, 0.4,0}

\definecolor{purplebrown}{rgb}{0.5,0.1,0.6}

\definecolor{ultclupcol}{rgb}{0.1,0.5,0.5}

\definecolor{mytrycolor}{rgb}{0.5,0.7,0.2}

\definecolor{ultclupcola}{rgb}{.5,0,.5}

\definecolor{shadebrown}{rgb}{0.1,0.1,0.9}
\definecolor{lightblue}{rgb}{0.2,0,1}

\usepackage{fancybox}
\usepackage{graphicx}
\usepackage{epstopdf}
\usepackage{epsfig}
\usepackage{wrapfig}
\usepackage{subfigure}

\usepackage{xcolor}
\usepackage{tcolorbox}

\newtcbox{\xmybox}{on line,
arc=7pt,
before upper={\rule[-3pt]{0pt}{10pt}},boxrule=0pt,
boxsep=0pt,left=6pt,right=6pt,top=0pt,bottom=0pt,enhanced, coltext=blue, colback=white!10!yellow}

\newtcbox{\xmyboxa}{on line,
arc=7pt,
before upper={\rule[-3pt]{0pt}{10pt}},boxrule=0pt,
boxsep=0pt,left=6pt,right=6pt,top=0pt,bottom=0pt,enhanced, colback=white!10!yellow}

\newtcbox{\xmyboxb}{on line,
arc=7pt,
before upper={\rule[-3pt]{0pt}{10pt}},boxrule=1pt,colframe=darkgreen!100!blue,
boxsep=0pt,left=6pt,right=6pt,top=0pt,bottom=0pt,enhanced, colback=white!10!yellow}

\newtcbox{\xmyboxc}{on line,
arc=7pt,
before upper={\rule[-3pt]{0pt}{10pt}},boxrule=.7pt,colframe=blue!100!blue,
boxsep=0pt,left=6pt,right=6pt,top=0pt,bottom=0pt,enhanced, coltext=blue, colback=white!10!yellow}

\newtcbox{\xmytboxa}{on line,
arc=7pt,
before upper={\rule[-3pt]{0pt}{10pt}},boxrule=.0pt,colframe=pink!50!yellow,
boxsep=0pt,left=6pt,right=6pt,top=0pt,bottom=0pt,enhanced, coltext=white, colback=blue!40!red}

\newtcbox{\xmytboxb}{on line,
arc=7pt,
before upper={\rule[-3pt]{0pt}{10pt}},boxrule=.0pt,colframe=pink!50!yellow,
boxsep=0pt,left=6pt,right=6pt,top=0pt,bottom=0pt,enhanced, coltext=white, colback=white!40!green}

\makeatletter
\newcommand\subsubsubsection{\@startsection{paragraph}{4}{\z@}{-2.5ex\@plus -1ex \@minus -.25ex}{1.25ex \@plus .25ex}{\normalfont\normalsize\bfseries}}
\newcommand\subsubsubsubsection{\@startsection{subparagraph}{5}{\z@}{-2.5ex\@plus -1ex \@minus -.25ex}{1.25ex \@plus .25ex}{\normalfont\normalsize\bfseries}}
\makeatother

\newtheorem{theorem}{Theorem}

\newtheorem{conjecture}{Conjecture}

\begin{document}

\begin{singlespace}

\title {Ultrametric OGP - parametric RDT  \emph{symmetric} binary perceptron connection  
}
\author{
\textsc{Mihailo Stojnic
\footnote{e-mail: {\tt flatoyer@gmail.com}}
}}
\date{}
\maketitle

\centerline{{\bf Abstract}} \vspace*{0.1in}

In \cite{Stojnicalgbp25,Stojnicalgsbp26,Stojniccluphop25}, a \emph{parametric fully lifted random duality theory} (fl-RDT) framework is introduced as a viable analytical method for characterizing \emph{statistical computational gaps} (SCGs). For example, a study of the \emph{symmetric binary perceptron} (SBP) \cite{Stojnicalgsbp26} obtained an \emph{algorithmic} constraint density threshold estimate of $\alpha_a\approx \alpha_c^{(7)}\approx 1.6093$ at the 7th level of lifting (for the canonical margin/threshold $\kappa=1$). With a predicted convergence tendency toward the $1.59-1.60$ range, these results closely approach the corresponding $\alpha_{LE}\approx 1.58$ replica method local entropy (LE) prediction from \cite{Bald20}.

In this paper, we further connect parametric RDT to overlap gap properties (OGPs), another key geometric feature of the solution space. Specifically, for any positive integer $s$, we consider $s$-level ultrametric OGPs ($ult_s$-OGPs) and rigorously upper-bound the associated constraint densities $\alpha_{ult_s}$. To achieve this, we develop an analytical union-bounding program consisting of combinatorial and probabilistic components. By casting the combinatorial part as a convex problem and the probabilistic part as a nested integration, we conduct numerical evaluations and obtain that the tightest bounds at the first two levels, $\bar{\alpha}_{ult_1} \approx 1.6578$ and $\bar{\alpha}_{ult_2} \approx 1.6219$, closely approach the 3rd and 4th lifting level parametric RDT estimates, $\alpha_c^{(3)} \approx 1.6576$ and $\alpha_c^{(4)} \approx 1.6218$. We also observe excellent agreement across other key parameters, including overlap values and the relative sizes of ultrametric clusters.

Based on these observations, we propose several conjectures linking $ult$-OGP and parametric RDT. Specifically, we conjecture that $\alpha_a=\lim_{s\rightarrow\infty} \alpha_{ult_s} = \lim_{s\rightarrow\infty} \bar{\alpha}_{ult_s} = \lim_{r\rightarrow\infty} \alpha_{c}^{(r)}$, and $\alpha_{ult_s} \leq \alpha_{c}^{(s+2)}$ (with possible equality for some (maybe even all) $s$). Finally, we discuss the potential existence of a full isomorphism connecting all key parameters of $ult$-OGP and parametric RDT.

\vspace*{0.25in} \noindent {\bf Index Terms: Symmetric binary perceptrons; ultrametric OGP; fl-RDT; statistical-computational gaps}.

\end{singlespace}

\section{Introduction}
\label{sec:back}

The significant progress in AI over the last two decades is largely driven by theoretical and algorithmic developments in machine learning and neural networks from previous decades. Classical perceptrons, both binary (BP) and spherical (SP), have been essential as both architectural components and prototype models for complex structures. Given their analytical tractability, ability to emulate artificial reasoning, and remarkable simplicity, they remain a primary focus of AI research.

Among the various features of perceptrons, storage or classifying capacity is perhaps the most significant. This critical data density, $\alpha_c$, determines a perceptron's success as a storing memory. The relevance of capacity and its analytical characterizations, recognized already in early pattern recognition studies \cite{Wendel62,Winder,Cover65}, continues to be vital today. Research into these concepts has historically interconnected fields such as cognitive science, logic, psychology, information theory, optimization, and statistical physics, providing the robust theoretical support necessary for modern AI.

\subsection{Perceptrons key milestones}
\label{sec:analdiff}

A significant body of influential work \cite{Gar88,GarDer88,SchTir02,SchTir03,Tal05,Talbook11a,Talbook11b,StojnicGardGen13,StojnicGardSphNeg13} has built upon initial SP considerations \cite{Wendel62,Winder,Cover65}. As overall understanding has deepened, focus has shifted toward different architectures. For example, the positive spherical perceptrons (PSPs) are now very well understood; analytical studies \cite{SchTir02,SchTir03,StojnicGardGen13} have utilized convexity/strong deterministic duality, provided a substantial upgrade on \cite{Wendel62,Winder,Cover65}, and rigorously proved statistical physics replica predictions \cite{Gar88,GarDer88}.  

In contrast, the negative spherical perceptrons (NSPs) lack these features, presenting an analytical challenge \cite{StojnicGardSphNeg13,FPSUZ17,FraHwaUrb19,FraPar16,FraSclUrb19,FraSclUrb20,AlaSel20,BMPZ23}. Simple  shift to a negative threshold removes the convexity advantage, ultimately requiring more sophisticated approaches \cite{Stojnicsflgscompyx23,Stojnicnflgscompyx23,Stojnicflrdt23,Stojnicnegsphflrdt23}.

Similarly, convexity and strong deterministic duality are absent in BPs. Simple replica symmetric (RS) predictions \cite{Gar88,GarDer88,StojnicDiscPercp13} do not apply to asymmetric BPs (ABP), which instead require replica symmetry breaking (RSB) frameworks \cite{KraMez89,DingSun19,NakSun23,BoltNakSunXu22,Huang24,Stojnicbinperflrdt23}. However, symmetric BPs (SBP) offer a unique case: while RS predictions still do not hold, favorable underlying combinatorics allow for elegant analytical characterizations \cite{AbbLiSly21b,PerkXu21,AbbLiSly21a,AubPerZde19,GamKizPerXu22}.

\subsubsection{Satisfiability and algorithmic thresholds}
\label{sec:satalg}

The capacity $\alpha_c$ discussed in a majority of the above papers relates to achievability limits.
It represents the \emph{satisfiability} threshold for maximal data density given unlimited computational resources. On the other hand, the introduction of an alternative \emph{algorithmic} threshold, $\alpha_a$, allows one to define the density achievable through computationally efficient methods. While $\alpha_a \leq \alpha_c$ always holds, the existence of a (nonzero) statistical-computational gap (SCG), defined as $SCG = \alpha_c - \alpha_a$, measures the actual utilization of the perceptron's predicted power.

Determining the size of this gap remains a significant challenge. While worst-case complexity theory suggests that ABP and SBP are NP-complete \cite{Ama91}, such a complexity assessment may often fail to adequately reflect typical algorithmic solvability. For instance, in ABPs, while $\alpha_c \approx 0.8331$, current algorithms suggest $\alpha_a \approx 0.75 - 0.77$ \cite{BaldassiBBZ07,Bald15,BMPZ23} (for existence of efficient algorithms in a wide range of $\alpha<\alpha_c$, see also  \cite{BrZech06,BaldassiBBZ07,Hubara16,KimRoc98}). This indicates a gap that is notably more favorable than what NP theory might predict (for algorithmic implications of similar type in other problems (including planted ones) see, e.g., \cite{MMZ05,GamarSud14,GamarSud17,GamarSud17a,AchlioptasR06,AchlioptasCR11,GamMZ22,BarbierKMMZ18,KMSSZ12a}). Conversely, for SBPs in the $\alpha \rightarrow 0$ regime, the best known algorithms \cite{BanSpen20} are in $\kappa\sim\sqrt{\alpha}$ range (where $\kappa$ is the SBPs margin/threshold) and  remain significantly below the corresponding $\alpha_c$ scaling.

\subsection{Relevant prior work}
\label{sec:examples}

Despite extensive studies across various scientific fields over the last several decades, computational gaps continue to be a mystery. While significant progress has been made on specific problems, a generic resolution remains elusive.

In recent years, two concepts have emerged as key geometric properties of the solution space for hard optimization problems: the Overlap Gap Property (OGP) \cite{Gamar21,GamarSud14,GamarSud17,GamarSud17a,AchlioptasCR11,HMMZ08,MMZ05} and Local Entropy (LE) \cite{Bald15,Bald16,Bald20,BarbAKZ23,Stojnicabple25}. These properties are presumably correlated with the mismatch between satisfiability (information-theoretic) and algorithmic thresholds, as well as the existence of statistical computational gaps (SCG). While our primary interest is their relevance to developing efficient algorithms, we also emphasize that their focus on the clustering of typical and atypical solutions has yielded results of independent value. Below, we briefly review the aspects of these two concepts that have been of main interest over the last few years.

\subsubsection{Overlap gap properties (OGPs) }
\label{sec:prwrkogp}

The OGP based approaches toward resolving SCG mysteries \cite{Gamar21,GamarSud14,GamarSud17,GamarSud17a,AchlioptasCR11,HMMZ08,MMZ05} propose connecting algorithmic efficiency to gaps in the spectrum of attainable near-solutions overlaps. They suggests that the absence of such gaps implies the existence of efficient algorithms, viewing $\alpha_a$ as the maximal $\alpha$ that ensures the absence of OGP. For instance, in the case of SBP (which is of our interest here), the presence of OGP extends well below $\alpha_c$ \cite{GamKizPerXu22,Bald20} (see also \cite{GamKizPerXu23} for analogous discrepancy minimization results). If OGP is indeed algorithmically relevant, this would strongly support the existence of a SCG in SBPs. However, there are two key points to keep in mind: 1) The generic hardness implications of OGP have been disproved by the shortest path counterexample \cite{LiSch24} (earlier disproving examples were viewed as potential anomalies/exceptions due to their simple algebraic structure); and  2) While \cite{LiSch24} provides a counterexample, it does not disprove the relevance of OGP for specific algorithms and/or other problems.

The $2$-spin Ising Sherrington-Kirkpatrick (SK) model \cite{SheKir72} is perhaps the most prominent hard example where OGP concepts are rigorously confirmed as highly relevant. As demonstrated in \cite{Montanari19}, the absence of OGP directly implies the polynomial solvability of this model (more on analogous $p$-spin and  NSP considerations can be found in  \cite{AlaouiMS22,AlaouiMS21} and  \cite{AlaSel20,AMZ24}, respectively; closely related spherical SK models  results can be found in  \cite{Subag17,Subag17a,Subag21,Subag24} whereas the martingale based related reversals of Parisi functional are discussed in \cite{JCM25}; for more sophisticated OGPs and their relevance see \cite{Kiz23,HuangS22,HuangS22a,HuangSell24,HuangSell23} as well as a more thorough discussion presented later on in Section \ref{sec:ogpother}). While the role of OGP in generic algorithmic hardness is still being determined, its presence is known to prevent efficient implementations across several algorithmic classes \cite{GamKizPerXu22, HuangS22, HuangSell23, HuangSell24} (for example, OGP type of reasoning has been successfully utilized for the hardness analysis of a specific class of online algorithms \cite{DuGH25,BhaGG26,GhaKM25,GamKW25}). Furthermore, for many problems, practical algorithms exist within $\alpha$ ranges where OGP is absent \cite{RahVir17, GamarSud14, GamarSud17, GamAW24, Wein22}.

Regarding the OGP in the SBP context, excellent sets of results are obtained in \cite{Bald20,GamKizPerXu22}. In \cite{GamKizPerXu22}, the authors analyzed 2-OGP and 3-OGP, obtaining upper bounds on constraint densities $\alpha_{2ogp} \leq 1.71$ and $\alpha_{3ogp} \leq 1.667$, respectively. Higher OGPs were not pursued, as $\infty$-OGP yielded looser bounds than $1.667$. \cite{Bald20} investigated this by applying rigorous methods to 4-OGP (as well as 2- and 3-OGPs) and replica methods to higher OGPs. However, the conclusions remain consistent with \cite{GamKizPerXu22}: while these methodologies produce tight upper bounds within certain parameter ranges (specifically replica numbers and pairwise distances), they eventually become loose. This leaves a significant gap between OGP results and the local entropy algorithmic threshold prediction of approximately $1.58$.

\subsubsection{Local entropy (LE) }
\label{sec:prwrkle}

Unlike the OGP, \cite{Huang13, Huang14} connects algorithmic hardness to the entropy of \emph{typical} solutions. While a frozen isolation of typical solutions has been predicted and proven for the SBP \cite{PerkXu21, AbbLiSly21a, AbbLiSly21b}, an alternative approach using local entropy (LE) has been proposed to study \emph{atypical}, well-connected clusters \cite{Bald15, Bald16, Bald20}.

Even when predominant solutions are disconnected and presumably unreachable via local search \cite{Huang13, Huang14, PerkXu21, AbbLiSly21b}, rare atypical clusters may still exist. These rare clusters are predicted to be well connected and  precisely what efficient algorithms identify (recent justifications for this include sampling-type arguments for SBP \cite{ElAlGam24} and algorithmic considerations for the ABP \cite{GongHLS26}, which shows that no stable algorithms can find (with probability close to 1) isolated solutions and thereby suggests that reaching some type of cluster may be a prerequisite for algorithmic efficiency). If this portrayal is correct, it suggests a direct correlation between the properties of these rare clusters and the existence of SCGs.

\cite{Bald15,Bald16,Bald20} also speculate that LE features—specifically negativity, monotonicity, and breakdown—may reflect the impact of rare cluster structures on algorithmic hardness. This phenomenology is further supported by the findings of \cite{AbbLiSly21a}, which demonstrate that SBP maximal diameter clusters exist when $\alpha$ is sufficiently small. Additionally, \cite{AbbLiSly21a} showed that similar clusters of linear diameter exist for any $\alpha < \alpha_c$. It is worth noting that, subject to certain technical assumptions, these SBP results also translate to ABP.

Reconnecting LE to OGP, it should be noted that findings in \cite{BarbAKZ23} demonstrate that small $\alpha$ SBP LE results—including both rigorous contiguous and 1RSB versions—align with the OGP predictions outlined in \cite{GamKizPerXu22}. Aside from a logarithmic term, these also match the best-known algorithmic results achieved in \cite{BanSpen20}. Finally, a parametric RDT approach introduced in \cite{Stojnicalgsbp26} suggests that allowing for non-physical arbitrary ordering of the $\c$-sequence (a key RDT parameter) at the $r\rightarrow\infty$ lifting level leads to the algorithmic threshold. Numerical evaluations show a strong agreement with LE predictions from \cite{BarbAKZ23,Bald20}. Furthermore, results of \cite{Stojnicalgbp25,Stojniccluphop25} indicate that these parametric RDT concepts may be applicable to a broader range of problems.

The existence of an OGP — LE correspondence is a reassuring development that may help demystify the connection between these phenomena and SCGs. Below, we present a set of rigorous results that further substantiate this connection.

\subsection{Contextualization of our contributions}
\label{sec:context}

\subsubsection{Parametric RDT}
\label{sec:prwrkparrdt}

\cite{Stojnicalgbp25,Stojnicalgsbp26,Stojniccluphop25} developed a \emph{parametric random duality theory} (RDT) framework as an alternative method for characterizing SCGs. This framework relies on the paradigm that the $\c$ sequence—a key RDT parameter—should be of an arbitrary ordering type rather than a naturally decreasing one. While removing the decreasing restriction creates a counterintuitive and presumably unphysical scenario, no mathematical reasoning appears to directly prevent it. This remains a possible option, provided that an adequate interpretation can be established.

Studying several well known problems including both feasibility  (ABPs \cite{Stojnicalgbp25} and  SBPs \cite{Stojnicalgsbp26}) and  standard optimal objective seeking ones (negative Hopfield models \cite{Stojniccluphop25}),   parametric RDT links arbitrary $\c$ sequence ordering to algorithmic thresholds and SCGs.  For example, as stated above, studying SBP  and allowing for the arbitrary ordering of the  $\c$ sequence,  \cite{Stojnicalgsbp26} obtained that the critical constraint density estimate associated with the 7th RDT lifting level—in the canonical scenario with a margin/threshold of $\kappa=1$—is $\alpha_c^{(7)}\approx 1.6093$. Furthermore, a strong convergence tendency is observed with the predicated converging value, $\lim_{r\rightarrow \infty }\alpha_c^{(r)}$, expected to fall within the 1.59–1.60 range and match the algorithmic threshold $\alpha_a = \lim_{r\rightarrow \infty }\alpha_c^{(r)}$. These results closely approach the local entropy replica-based predictions of $\alpha_{LE}\approx 1.58$ from \cite{Bald20}.

An almost identical set of observations was the primary result of \cite{Stojnicalgbp25}, which studied ABPs. Specifically, in the ABP's canonical scenario (where the margin is 0, unlike the SBP), \cite{Stojnicalgbp25} obtained a critical constraint density estimate of $\alpha_c^{(5)}\approx 0.7764$ at the 5th RDT lifting level. A strong convergence tendency was observed, and the presumably converging value $\lim_{r\rightarrow \infty }\alpha_c^{(r)}$ is predicted to be in the $0.775-0.776$ range. This value is also expected to match the algorithmic threshold $\alpha_a = \lim_{r\rightarrow \infty }\alpha_c^{(r)}$. Furthermore, these results were observed to align well with $\alpha_{LE}\approx 0.77-0.78$, obtained via local entropy considerations in \cite{Stojnicabple25,Bald16}.
 
Similar conclusions are reached in \cite{Stojniccluphop25}, with the distinction that the underlying negative Hopfield models are classical optimal objective-seeking problems. In that context, the algorithmic thresholds differ slightly as they are associated with the algorithmically achievable value of the objective. 

Finally, utilizing the \emph{controlled loosening-up} (CLuP) optimizing formalism \cite{Stojnicclupsk25}, algorithmic CLuP variants have been designed for all problems considered in \cite{Stojnicalgbp25,Stojnicalgsbp26,Stojniccluphop25}. Their performance is shown to closely match the analytical algorithmic predictions.

\subsubsection{Our contributions}
\label{sec:contrib}

The results of \cite{Stojnicalgbp25,Stojnicalgsbp26} suggest that parametric RDT relates not only to computational gaps but also to local entropy. In this paper, we explore potential connections to Overlap Gap Properties (OGPs). Using the SBP as a convenient model, we consider its ultrametric OGPs ($ult$-OGP). Below is a summary of our $ult$-OGP contributions as presented throughout the paper.

\begin{itemize}
  \item For any positive integer $s$, we are considering $s$-level ultrametric OGPs ($ult_s$-OGPs). We associate these with $\alpha_{ult_s}$, defined as the smallest constraint density for which the solution space of the SBP exhibits $ult_s$-OGP. Within a standard statistical context and utilizing a union-bounding methodology, we provide rigorous (probabilistic) bounds, $\bar{\alpha}_{ult_s}$, on $\alpha_{ult_s}$.

  \item We develop an analytical program consisting of a combinatorial and a probabilistic component. By casting the former as a convex program and formulating the latter as a nested integration, we are able to conduct the underlying numerical evaluations.

  \item  Following our evaluations, we find $\bar{\alpha}_{ult_1}\approx 1.6578$ and $\bar{\alpha}_{ult_2}\approx 1.6219$ as the tightest bounds on the first and second levels of ultrametricity. These estimates align remarkably well with the critical constraint density estimates $\alpha_c^{(3)} \approx 1.6576$ and $\alpha_c^{(4)} \approx 1.6218$ obtained on the 3rd and 4th lifting levels of parametric RDT in \cite{Stojnicalgsbp26}.

   \item Furthermore, we observed excellent overall agreement across all key parameters, including overlap values and the relative sizes of the ultrametric clusters.

   \item  Several conjectures that directly link $ult$-OGP and parametric RDT are formulated. In particular we conjecture for any margin $\kappa$ (i.e., not only for the canonical $\kappa=1$)
      \begin{equation}\label{eq:inteq1}
        \lim_{s\rightarrow \infty} \alpha_{uls_s}(\kappa) 
= \lim_{s\rightarrow \infty}   \bar{\alpha}_{uls_s} (\kappa) 
= \lim_{r\rightarrow \infty} \alpha^{(r)}_c(\kappa),
      \end{equation} 
and even stronger
      \begin{equation}\label{eq:inteq2}
         \alpha_{uls_s}(\kappa) 
 \leq   \alpha^{(s+2)}_c(\kappa), \forall s\geq 1,
      \end{equation} 
with the last inequality possibly being equality for some (maybe even all) $s$.

\item We discuss the potential existence of a full isomorphism between $ult$-OGP and parametric RDT, where not only the $\alpha$ values but also all other critical parameters of both concepts align.

\item Since \cite{Stojnicalgsbp26} has already suggested a potential connection between parametric RDT and local entropy (LE), the $ult$-OGP – parametric RDT connection discussed in this paper further links $ult$-OGP to LE. This ultimately suggests that it may be possible for all three concepts to achieve accurate $\alpha_a$ characterizations.
     
\end{itemize}

\section{SBP -- basics}
 \label{sec:bprfps}

\subsection{Alternative feasibility/optimization SBP formulations}
 \label{sec:bprfps}

Let $m$ and $n$ be two positive integers. For $G\in\mR^{n\times n}$, $\b\in\mR^{m\times 1}$, and $\cX\in\mR^n$, the following linearly constrained class of \emph{feasibility} problems is of our interest
\begin{eqnarray}
\hspace{-1.5in}\mbox{$\mathbf{\mathcal F}(G,\b,\cX,\alpha)$:} \hspace{1in}\mbox{find} & & \x\nonumber \\
\mbox{subject to}
& & G\x\geq \b \nonumber \\
& & \x\in\cX. \label{eq:ex1}
\end{eqnarray}
We consider \emph{linear/proportional}  large dimensional regime with constraint density $\alpha\triangleq \lim_{n\rightarrow\infty} \frac{m}{n}$ remaining constant as $m$ and $n$ grow. Many perceptron types discussed earlier are special cases of (\ref{eq:ex1}). For example, taking $\cX=\{\x | \| \x \|_2=1\} \triangleq \mS^m$ gives PSP when $\b\geq 0$  \cite{StojnicGardGen13,GarDer88,Gar88,Schlafli,Cover65,Winder,Winder61,Wendel62,Cameron60,Joseph60,BalVen87,Ven86,SchTir02,SchTir03} and NSP when $\b< 0$ \cite{AMZ24,BMPZ23,FPSUZ17,Talbook11a,FraHwaUrb19,FraPar16,FraSclUrb19,FraSclUrb20,AlaSel20,StojnicGardSphNeg13,Stojnicnegsphflrdt23}; taking $\cX=\left \{-\frac{1}{\sqrt{n}},\frac{1}{\sqrt{n}} \right \}^n \triangleq \mB^n$ gives ABP \cite{Talbook11a,StojnicGardGen13,GarDer88,Gar88,StojnicDiscPercp13,KraMez89,GutSte90,KimRoc98,NakSun23,BoltNakSunXu22,PerkXu21,CXu21,DingSun19,Huang24,Stojnicbinperflrdt23,LiSZ24} and so on.

Our foucs is on SBPs \cite{AubPerZde19,AbbLiSly21a,AbbLiSly21b,Bald20,GamKizPerXu22,PerkXu21,ElAlGam24,SahSaw23,Barb24,djalt22,BarbAKZ23}. They are another specialization of (\ref{eq:ex1}) where $\cX=\mB^n$ and linear constraints are remodeled to become $|G\x |\leq \b$. Formally, SBP is the following variant of (\ref{eq:ex1})  
\begin{eqnarray}
\hspace{-1.5in}\mbox{$\mathbf{\mathcal S}(G,\b,\alpha)$:} \hspace{1in}\mbox{find} & & \x\nonumber \\
\mbox{subject to}
& & |G\x|\leq \b \nonumber \\
& & \x\in \left \{-\frac{1}{\sqrt{n}},\frac{1}{\sqrt{n}} \right \}^n \triangleq \mB^n. \label{eq:ex1a0a0}
\end{eqnarray}
The above is SBP with generic thresholds $\b$. Further specialization to $\b=\kappa \1$ (throughout the paper $\1$ stands for the column vector of all ones; the dimension of the vector will be clear form the context; here, for example, it is an $m$-dimensional vector of all ones) gives the traditional SBP formulation
\begin{eqnarray}
\hspace{-1.5in}\mbox{$\mathbf{\mathcal S}(G,\kappa,\alpha)$:} \hspace{1in}\mbox{find} & & \x\nonumber \\
\mbox{subject to}
& & |G\x|\leq \kappa \1 \nonumber \\
& & \x\in \left \{-\frac{1}{\sqrt{n}},\frac{1}{\sqrt{n}} \right \}^n \triangleq \mB^n. \label{eq:ex1a0}
\end{eqnarray}

SBPs are also directly related to the so-called discrepancy minimization problems typically seen in computer science literature \cite{KaKLO86,Spen85,LovMek15,GamKizPerXu23,Roth17,AlwLiuSaw21}
\begin{eqnarray}
\hspace{-1.5in}\mbox{$\mathbf{\mathcal D}(G,\alpha)$:} \hspace{1in}\min_{\x} & & |\G\x| \nonumber \\
\mbox{subject to}
 & & \x\in \left \{-\frac{1}{\sqrt{n}},\frac{1}{\sqrt{n}} \right \}^n \triangleq \mB^n. \label{eq:ex1a1}
\end{eqnarray}
It is useful to note that (\ref{eq:ex1a0}) and (\ref{eq:ex1a1})  are fundamentally different optimization problems. (\ref{eq:ex1a0}) is a feasibility problem and (\ref{eq:ex1a1})  is an objective minimization. Nonetheless, $\mathbf{\mathcal D}(G,\alpha)$ can be used to solve $\mathbf{\mathcal S}(G,\kappa,\alpha)$. The optimal objective in  (\ref{eq:ex1a1})  is practically the minimal $\kappa$ that ensures feasibility of $\mathbf{\mathcal S}(G,\kappa,\alpha)$.  In other words, solving  (\ref{eq:ex1a1}) and checking whether $\kappa$ in  (\ref{eq:ex1a0}) upper bounds the obtained objective optimum suffices to determine feasibility of (\ref{eq:ex1a0}). 

For what follows we find it useful to note  the following alternative optimization reformulation of (\ref{eq:ex1a0a0}) \cite{Stojnicalgsbp26}
\begin{eqnarray}
\xi_{SBP}
& \triangleq  &
 \min_{\x\in \mB^n, \| \z \|_{\infty} \leq \kappa} \max_{\y\in\mS_+^m}  \lp \y^TG\x - \kappa \y^T\z \rp,
 \label{eq:ex3}
\end{eqnarray}
where $\mS_+^m$ is the  positive orthant part of the $m$-dimensional unit sphere  (i.e., $\mS_+^m=\{\y|\|\y\|_2=1,\y\geq 0\}$). One 
then has 
\begin{eqnarray}
  \mathbf{\mathcal S}(G,\kappa,\alpha)  \mbox{ is infeasible }  \quad \Longleftrightarrow   \quad\xi_{SBP}>0 . \label{eq:ex1a3a0}
\end{eqnarray}
This further implies that for all practical purposes  the optimal objective seeking (\ref{eq:ex3}) is basically equivalent to  and can replace  $\mathbf{\mathcal S}(G,\kappa,\alpha)$.

\subsection{Information-theoretic and algorithmic capacities -- computational gap}
 \label{sec:caprole}

Depending on whether $G$ is random or not, one has statistical or deterministic perceptrons. Our focus will be on the classical Gaussian perceptrons, where the components of $G$ are independent standard normals.

We now mathematically formalize the storage/classifying capacity notion mentioned in previous sections. As a key measure of a perceptron's efficacy, storage capacity is defined as the minimal constraints density for which $\mathbf{\mathcal S}(G,\kappa,\alpha)$ is infeasible. The capacity can be a statistical or deterministic quantity depending on the choice of $G$. Our focus is on random $G$ and, following the practice established in \cite{StojnicGardGen13,StojnicDiscPercp13,Stojnicbinperflrdt23,Stojnicalgbp25,Stojnicalgsbp26}, we define the SBP's statistical storage capacity as
 \begin{eqnarray}
\alpha & = &    \lim_{n\rightarrow \infty} \frac{m}{n}  \nonumber \\
\alpha_c(\kappa) 
& \triangleq & \min \{\alpha |\hspace{.08in}  \lim_{n\rightarrow\infty}\mP_G\lp {\mathcal S}(G,\kappa,\alpha) \hspace{.07in}\mbox{is infeasible} \rp\longrightarrow 1\}
 \nonumber \\
& = &  \min \{\alpha |\hspace{.08in}  \lim_{n\rightarrow\infty}\mP_G\lp  \xi_{SBP}>0\rp\longrightarrow 1\}.
  \label{eq:ex4}
\end{eqnarray}
Throughout this presentation, we adopt the convention that subscripts next to $\mathbb{P}$ and/or $\mathbb{E}$ (if present) denote the source of randomness. Since our analytical considerations are statistical in nature, we may omit emphasizing that they hold with high probability when it is clear from the context.

Practically speaking, the capacity notion characterizes the critical constraint density below which SBP operates properly. It also hints at the phase-transitioning component of the underlying randomness. For example, when $\alpha > \alpha_c$, $\cS(G, \kappa, \alpha)$ is in the UNSAT phase with no feasible solutions. Conversely, when $\alpha < \alpha_c$, $\cS(G, \kappa, \alpha)$ is feasible and in the SAT phase. A transition between these two phases occurs precisely at $\alpha_c$, where an exponentially large set of solutions in the SAT phase shrinks to an empty set in the UNSAT phase \cite{KraMez89, NakSun23, BoltNakSunXu22, DingSun19, Huang24, Stojnicbinperflrdt23}. This \emph{satisfiability threshold} represents the theoretical limit of SBP's storage power.

The degree to which SBP's storage power is utilizable depends critically on the ability to efficiently determine weights $\mathbf{x}$ in (\ref{eq:ex1a0a0}). The discrete nature of the problem ensures that finding a feasible $\mathbf{x}$ is possible throughout the SAT phase ($\alpha < \alpha_c$). However, identifying the specific portions of the SAT regime where this process is also computationally efficient remains an extraordinary challenge.

To formalize this, we distinguish between two $\alpha$ SAT phase regimes, beginning with the \emph{algorithmic threshold}, $\alpha_a(\kappa)$
  \begin{equation}
 \alpha_a(\kappa) 
= \alpha_{\underline{a}}(\kappa) 
 \triangleq  \max \{\alpha |\hspace{.08in}  \lim_{n\rightarrow\infty}\mP_G\lp \mbox{a feasible $\x$ in $ {\mathcal S}(G,\kappa,\alpha)$ can be found in polynomial time} \rp\longrightarrow 1\}.
  \label{eq:ex4a0}
\end{equation}
We also set
 \begin{equation}
 \alpha_{\bar{a}}(\kappa) 
 \triangleq  \min \{\alpha |\hspace{.08in}  \lim_{n\rightarrow\infty}\mP_G\lp \mbox{no feasible $\x$ in $ {\mathcal S}(G,\kappa,\alpha)$ can be found in polynomial time} \rp\longrightarrow 1\},
  \label{eq:ex4a0a0}
\end{equation}
and observe
\begin{eqnarray}
 \alpha_a(\kappa)  =  \alpha_{\underline{a}}(\kappa)  \leq  \alpha_{\bar{a}}(\kappa)  \leq \alpha_c(\kappa).
  \label{eq:ex4a1}
\end{eqnarray}
Under concentrations one may expect $\alpha_{\underline{a}}(\kappa)  =  \alpha_{\bar{a}}(\kappa) $ which allows to switch focus on
\begin{eqnarray}
 \alpha_a(\kappa)    \leq \alpha_c(\kappa).
  \label{eq:ex4a1a0}
\end{eqnarray}
The primary challenge involves determining whether the remaining inequality in (\ref{eq:ex4a1a0}) is strict. If so, the satisfiability threshold, $\alpha_c(\kappa)$, and the algorithmic threshold, $\alpha_a(\kappa)$, are not equal. This would indicate the presence of a nonzero statistical-computational gap (SCG)
 \begin{eqnarray}
SCG \triangleq  \alpha_c(\kappa) -\alpha_a(\kappa).
  \label{eq:ex4a2}
\end{eqnarray}
The above question and its potential resolution have direct consequences for the classical P vs. NP problem and its stronger statistical variants. Utilizing the fully lifted random duality theory (fl-RDT) machinery from \cite{Stojnicflrdt23}, a parametric RDT approach was proposed in \cite{Stojnicalgsbp26} to tackle this challenge.

The results provided in Table \ref{tab:tab1} are obtained in \cite{Stojnicalgsbp26} and represent the estimates of $\alpha_c(1)$ at different lifting levels. It was observed that $\alpha_c^{(2)}(1)$ matches $\alpha_c(1)$, achieving the storage capacity or satisfiability threshold. Furthermore, \cite{Stojnicalgsbp26} notes that the estimates on the first two lifting levels are obtained by restricting the c-parameter—a key parametric RDT component—to a decreasingly ordered sequence. By removing this restriction on higher lifting levels ($r > 2$) and allowing for arbitrary $\c$-sequence ordering, \cite{Stojnicalgsbp26} obtained the values listed in Table \ref{tab:tab1}. These results suggest that the $\alpha_c^{(r)}(1)$ sequence likely converges at a fast rate. The potentially converging value, $\lim_{r\rightarrow\infty} \alpha_c^{(r)}(1)$, appears to be in the 1.59–1.60 range, which aligns remarkably well with earlier algorithmic threshold predictions obtained via replica analysis in \cite{Bald20}. Finally, in the small constraint density regime ($\alpha \rightarrow 0$), \cite{Stojnicalgsbp26} obtains an explicit $\alpha_a, \kappa$ connection on the third lifting level: $\kappa \approx 1.2385\sqrt{\frac{\alpha_a}{-\log(\alpha_a)}}$. This qualitatively matches the OGP-based predictions of \cite{GamKizPerXu22} and identically matches the local entropy-based predictions of \cite{BarbAKZ23}, including the value of the scaling constant.
 
\begin{table}[h]
\caption{fl-RDT estimates of  $\alpha_c(1)$ as functions of lifting level $r$}\vspace{.1in}
\centering
\def\arraystretch{1.2}
 \begin{tabular}{||c||c||c||c||c||c||c||c||}\hline\hline
 \hspace{-0in}$r$                                              & $1$    &  $2$   &  $3$   &  $4$   &  $5$   &  $6$   &  $7$    \\ \hline\hline
 $\alpha_c^{(r)}(1)$  & \bl{$\mathbf{4.2250}$}  & \bl{$\mathbf{1.8159}$}  & \red{$\mathbf{1.6576}$}  & \red{$\mathbf{1.6218}$}  & \red{$\mathbf{1.6093}$}  & \red{$\mathbf{1.6041}$}  & \red{$\mathbf{1.6021}$}  \\ \hline\hline
  \end{tabular}
\label{tab:tab1}
\end{table}

The proposed idea and parametric RDT approach are not restricted to SBPs. In \cite{Stojnicalgbp25}, a similar approach was applied to ABP, yielding results that closely match the local entropy predictions from \cite{Bald16} and \cite{Stojnicabple25}. While both SBP and ABP are feasibility problems, this methodology applies successfully to a wider range of contexts. For instance, in \cite{Stojniccluphop25}, it is used to characterize the computational gap of the negative Hopfield model—a classical optimization problem. Furthermore, by utilizing the controlled loosening-up (CLuP) formalism, practical algorithmic variants have been designed for these cases that closely approach the predicted performance.

\section{Ultrametric OGP}
\label{sec:ultogp}

Building on the \cite{Stojnicalgsbp26}'s established connection between parametric RDT and the local entropy of atypical clusters \cite{BarbAKZ23, Bald20}, we are investigating whether a similar connection exists with the Overlap Gap Property (OGP). To explore this, we consider a specific OGP type called \emph{ultrametric} OGP ($ult$-OGP). Our findings demonstrate that its associated constraint densities align closely with parametric RDT predictions \cite{Stojnicalgsbp26}.

\subsection{Ult-OGP preliminaries}
\label{sec:ultogpprel}

We first recall the practical prototype of the OGP definition and the associated constraint density. We focus on the essential underlying mathematical concepts here; for full definitions and an introductory survey, see, e.g., \cite{Gamar21}. For an extensive list of algorithmic hardness/solvability implications, see, e.g.,     \cite{BreHuang21,Kiz23,GamJag21,Gametal23,GamAW24,GamarnikJW20,GamarSud14,GamarSud17,GamarSud17a,RahVir17,Wein22,CGPR19,HuangSell24,HuangSell23,HuangS22}).  

We start with the simplest possible 2-OGP and consider $\x^{(1)}$ and $\x^{(2)}$ as two solutions of (\ref{eq:ex1a0}) that have overlap $q$. For a fixed $q\in(0,1)$ let 
\begin{equation}\label{eq:ultogp1}
  \cX_0(q) = \left\{ \x^{(1)},\x^{(2)} \hspace{.03in} | \x^{(1)},\x^{(2)}\in\cB, ,  \lp\x^{(1)}\rp^T\x^{(2)} =q  \right \}.  
\end{equation}
Then 
\begin{equation}\label{eq:ultogp2}
 \min_{q\in(0,1)} \lim_{n\rightarrow \infty }\mP_G(\exists \x^{(1)},\x^{(2)}  \in\cX_0(q) \hspace{.03in} |  \hspace{.03in} |G\x^{(1)}|\leq \kappa\1, |G\x^{(2)}|\leq \kappa\1)\rightarrow 0 
\quad  \implies \quad 
  \cS(G,\kappa,\alpha) \mbox{ exhibits $2$-OGP}.   
\end{equation}
The above practically means that if $2$-OGP is present then there will be an overlap that no two solutions likely have.  One then associates with (\ref{eq:ultogp2}) the following critical constraint density 
\begin{equation}\label{eq:ultogp2a0}
 \alpha_{2ogp}(\kappa) \triangleq \min\left \{\alpha \hspace{.03in}  | \hspace{.03in} 
  \cS(G,\kappa,\alpha) \mbox{ exhibits $2$-OGP} \right \}.   
\end{equation}

Generalizations to $k$-OGP where $k$ is an integer larger than $2$ are also possible. Define set $\cQ$ as
\begin{equation}\label{eq:ultogp3}
\cQ(k) = \left \{ Q\hspace{.03in} | \hspace{.03in} Q\in \mR^{k\times k}, Q = (1-q)I + q\1\1^T , q\in(0,1) \right \}, 
\end{equation}
where $I$ is an identity matrix (throughout the paper $I$ may have different size which will be clear from the context). For a $Q$ from $\cQ(k)$ define set $\cX(k;Q)$ as
\begin{equation}\label{eq:ultogp4}
  \cX(k;Q) = \left\{ X \hspace{.03in} |  \hspace{.03in}  X\in \mR^{n\times k}, X^TX =Q, \mbox{ and } \forall i, X_{:,i}\in\cB \right \},  
\end{equation}
where $X_{:,i}$ is the $i$-th column of $X$.
Then 
\begin{equation}\label{eq:ultogp5}
 \min_{Q\in\cQ(k)} \lim_{n\rightarrow \infty }\mP_G(\exists X \in \cX(k;Q) \mbox{ such that } \forall i, |GX_{:,i}|\leq \kappa\1)\rightarrow 0 
\quad  \implies \quad 
  \cS(G,\kappa,\alpha) \mbox{ exhibits $k$-OGP}.   
\end{equation}
Analogously to (\ref{eq:ultogp2a0}),  one then associates with (\ref{eq:ultogp5}) the following critical constraint density 
\begin{equation}\label{eq:ultogp5a0}
 \alpha_{kogp}(\kappa) \triangleq \min\left \{\alpha \hspace{.03in}  | \hspace{.03in} 
  \cS(G,\kappa,\alpha) \mbox{ exhibits $k$-OGP} \right \}.   
\end{equation}

In this paper we are interested in a particular upgrade of $k$-OGP where matrix $Q$ is ultrametric. For a positive integer $s$ (that denotes the level of ultrametricity) we consider a decreasing real $\q$-sequence ($\q=[q_0,q_1,\dots,q_s]$)
\begin{equation}\label{eq:ultogp6}
1 =q_0 > q_1> q_2 >\dots > q_s > q_{s+1}=0, 
\end{equation}
and an increasing integer $\k$-sequence  ($\k=[k_0,k_1,\dots,k_s]$)
\begin{equation}\label{eq:ultogp7}
1 =k_0 <  k_1 < k_2 < \dots < k_s = k, \mbox{ where } \frac{k_{i+1}}{k_i}\in\mZ.
\end{equation}
With $\q$ and $\k$ as in (\ref{eq:ultogp6}) and (\ref{eq:ultogp7}), respectively, we define the following set of $s$-level ultrametric matrices
 \begin{equation}\label{eq:ultogp8}
\cQ_{ult_s}(\k) = \left \{ Q\hspace{.03in} | \hspace{.03in} Q\in \mR^{k\times k}, Q =  \sum_{i=0}^{s} (q_{i}-q_{i+1}) I_{\frac{k}{k_{i}}\times\frac{k}{k_{i}}} \1_{k_{i}}\1_{k_{i}}^T \right \}.
\end{equation}
For a $Q$ from $\cQ_{ult_s}(\k)$ define set  
\begin{equation}\label{eq:ultogp9}
  \cX_{ult_s}(\k;Q) = \left\{ X \hspace{.03in} |  \hspace{.03in}  X\in \mR^{n\times k}, X^TX =Q, \mbox{ and } \forall i, X_{:,i}\in\cB\right \}.  
\end{equation}
Then we say
\begin{equation}\label{eq:ultogp10}
\min_{\k} \min_{Q\in\cQ_{ult_s}(\k)} \lim_{n\rightarrow \infty }\mP_G(\exists X\in \cX_{ult_s}(\k;Q) 
\hspace{.03in} | \hspace{.03in}
 \forall i, |GX_{:,i}|\leq \kappa\1)\rightarrow 0 
\quad  \implies \quad 
  \cS(G,\kappa,\alpha) \mbox{ exhibits $ult_s$-OGP}.   
\end{equation}
Analogously to (\ref{eq:ultogp2a0}),  one then associates with (\ref{eq:ultogp5}) the following critical constraint density 
\begin{equation}\label{eq:ultogp10a0}
 \alpha_{ult_s}(\kappa) = \min\left \{\alpha \hspace{.03in}  | \hspace{.03in} 
  \cS(G,\kappa,\alpha) \mbox{ exhibits $ult_s$-OGP} \right \}.   
\end{equation}

It is not that difficult to see that the above sets $\cX_0(q)$, $\cX(k;Q)$, and $\cX_{ult_s}(\k;Q)$ are closely connected. In particular, $\cX_0(q) = \cX(2;Q)$ and $\cX(k;Q) = \cX_{ult_1}(\k;Q)$.   Given that $2$-OGP is less general concept than $k$-OGP, which itself is less general than $ult_s$-OGP, one expects
\begin{equation}\label{eq:ultogp11}
 \alpha_{ult_s}(\kappa) \leq \alpha_{ult_1}(\kappa) =  \lim_{k\rightarrow\infty}   
 \alpha_{kogp}(\kappa) \leq  \alpha_{2ogp}(\kappa).   
\end{equation}
In the following sections we obtain a series of results that support such an expectation.

\subsection{First level of ultrametricity -- $ult_1$-OGP}
\label{sec:ult1}

To provide $\alpha_{ult_s}$ characterizations we will rely on union-bounding strategy. Pursuing this avenue may come as surprise as it rarely works in precise analyses. However, the SBP is a notable exception. Completely remarkably, the SBP's  underlying combinatorial considerations happen to be so favorable that a simple union-bounding is already sufficiently powerful to provide \emph{exact} characterization of $\alpha_c(\kappa)$. In particular, \cite{AubPerZde19} obtained
\begin{eqnarray}\label{eq:ult1eq1}
\alpha_c(\kappa) = -\frac{\log(2)}{\log\lp   \frac{1}{2}\erfc\lp -\frac{\kappa}{\sqrt{2}} \rp 
 -
 \frac{1}{2}\erfc\lp \frac{\kappa}{\sqrt{2}} \rp \rp} \quad \mbox{ and }\quad 
\alpha_c(1)   \approx 1.8159.
    \end{eqnarray}
As shown in Table \ref{tab:tab1}, this was matched on the second partial level of lifting via RDT in \cite{Stojnicalgsbp26}. To see how the basic combinatorial argument indeed produces a tight upper-bound, one observes that for an $\x\in\cB$ 
\begin{eqnarray}\label{eq:ult1eq2}
\mP(|G\x|\leq \kappa\1 ) =    \lp  \frac{1}{2}\erfc\lp -\frac{\kappa}{\sqrt{2}} \rp 
 -
 \frac{1}{2}\erfc\lp \frac{\kappa}{\sqrt{2}} \rp \rp^m.
   \end{eqnarray}
   Then 
\begin{eqnarray}\label{eq:ult1eq3}
\mP(\exists \x\in\cB \mbox{ such that } |G\x|\leq \kappa\1 ) 
& \leq &
 \sum_{\x\in\cB}   \mP(|G\x|\leq \kappa\1 ) 
=
  \sum_{i=1}^{2^n}   \mP(|G\x^{(i)}|\leq \kappa\1 ) 
\nonumber \\
& = &  2^n \lp 
 \frac{1}{2}\erfc\lp -\frac{\kappa}{\sqrt{2}} \rp 
 -
 \frac{1}{2}\erfc\lp \frac{\kappa}{\sqrt{2}} \rp \rp^m
\nonumber \\
& = &  \lp2\lp 
 \frac{1}{2}\erfc\lp -\frac{\kappa}{\sqrt{2}} \rp 
 -
 \frac{1}{2}\erfc\lp \frac{\kappa}{\sqrt{2}} \rp \rp^{\alpha} \rp^n.
   \end{eqnarray}
Moreover,
\begin{eqnarray}\label{eq:ult1eq4}
2\lp 
 \frac{1}{2}\erfc\lp -\frac{\kappa}{\sqrt{2}} \rp 
 -
 \frac{1}{2}\erfc\lp \frac{\kappa}{\sqrt{2}} \rp \rp^{\alpha} \leq 1
 \quad \implies \quad 
 \lim_{n\rightarrow \infty}\mP(\exists \x\in\cB \mbox{ such that } |G\x|\leq \kappa\1 ) \rightarrow 0.
   \end{eqnarray}
One then also has
\begin{eqnarray}\label{eq:ult1eq5}
\alpha> -\frac{\log(2)}
{ 
 \frac{1}{2}\erfc\lp -\frac{\kappa}{\sqrt{2}} \rp 
 -
 \frac{1}{2}\erfc\lp \frac{\kappa}{\sqrt{2}} \rp }
 \quad \implies \quad 
 \lim_{n\rightarrow \infty}\mP(\exists \x\in\cB \mbox{ such that } |G\x|\leq \kappa\1 ) \rightarrow 0,
   \end{eqnarray}
which based on the definition (\ref{eq:ex4}) gives
\begin{eqnarray}\label{eq:ult1eq6}
\alpha_c(\kappa)\leq  -\frac{\log(2)}
{ 
 \frac{1}{2}\erfc\lp -\frac{\kappa}{\sqrt{2}} \rp 
 -
 \frac{1}{2}\erfc\lp \frac{\kappa}{\sqrt{2}} \rp }.
\end{eqnarray}
\cite{AubPerZde19}  then  proceeds with the second moment analysis and shows a matching lower bound. It is also worth noting that the above argument is spelled out in details in \cite{StojnicGardGen13} for ABP as well. However, only in certain ranges of $\kappa$ it is competitive with RDT and more advanced methodologies. 

The above is simple but instructive and we will try to emulate it in what follows to as large extent as possible. However, the underlying sets of favorable $\x$ will get increasingly more complex and the analysis will become less and less trivial. While one can immediately start with very generic results, we choose a systematic approach that first goes through simpler steps and then further builds towards the more complex ones. There are three main reasons that motivate such a choice: 1) it allows us to gradually, level-by-level, establish connection with parametric RDT (and conjecture that the connection might not be restricted only to the limiting scenarios);  2) it also allows to point at possible connections between RDT and results already available in literature; and 3) it makes more involved technical parts of the presentation appear smoother, easier, and/or more natural to follow.

\subsubsection{$2$-OGP}
\label{sec:ult12ogp}

We now look at $2$-OGP as a warmup application of the above principle. To that end we observe the following analogue of (\ref{eq:ult1eq3})

\begin{eqnarray}\label{eq:ult1eq7}
& &  \hspace{-.4in}\mP(\exists \hspace{.03in}\x^{(1)},\x^{(2)}\in\cX_0(q) \mbox{ such that } |G\x^{(1)}|\leq \kappa\1,|G\x^{(2)}|\leq \kappa\1  ) 
\nonumber \\
& \leq &
 \sum_{\x^{(1)},\x^{(2)}\in\cX_0(q)}  
 \mP\lp |G\x^{(1)}|\leq \kappa\1,|G\x^{(2)}|\leq \kappa\1, \lp\x^{(1)}\rp^T\x^{(2)}=q  \rp \nonumber \\
& = &
|\cX_0(q)|  
 \mP\lp |G\x^{(1)}|\leq \kappa\1,|G\x^{(2)}|\leq \kappa\1, \lp\x^{(1)}\rp^T\x^{(2)}=q  \rp
  \nonumber \\
&  =  & \lp2^{h_q(2)} p_q(2)^{\alpha}\rp^n,
    \end{eqnarray}
    where $h_q(2)$ and  $p_q(2)$ are the combinatorial and probabilistic factors, respectively, given as
\begin{eqnarray}\label{eq:ult1eq8}
h_q(2) &  = & \frac{\log_2\lp  |\cX_0(q)|  \rp}{n} \nonumber \\
p_q(2) & = &   \lp  \mP\lp |G\x^{(1)}|\leq \kappa\1,|G\x^{(2)}|\leq \kappa\1, \lp\x^{(1)}\rp^T\x^{(2)}=q  \rp
 \rp^{\frac{1}{m}} ,
    \end{eqnarray}    
 and $|\cX_0(q)| $ is the cardinality of set $\cX_0(q)$.
     
Compared to the upper-bounding of the capacity itself presented above, one already here for $2$-OGP recognizes appearance of two main analytical sources of difficulty that will present themselves in more complex forms as we progress below. One is counting the number of elements in $\cX_0(q)$ and the other is the evaluation of a slightly more complicated correlated Gaussian likelihood. However, both of them can be handled. 

For any $a\in(0,1)$ we define scalar entropy function
\begin{eqnarray}\label{eq:ult1eq8a0}
h(a) & = &  -(1-a)\log_2(1-a)-a\log_2(a).
\end{eqnarray}
Also, for any vector $\a$ with elements $\a_i\in(0,1)$ such that $\sum_{i}\a_i=1$, we define corresponding vector entropy function
\begin{eqnarray}\label{eq:ult1eq8a1}
h_v(\a) & = &  -\sum_{i} \a_i\log_2(\a_i).
\end{eqnarray}
We then further set
\begin{eqnarray}\label{eq:ult1eq9}
q_{sx1} & = & (1+q)/2 \nonumber \\
q_{s1} & = &  q_{sx1} 
\nonumber \\
h_q^{(0)}(2)  & = & 1
\nonumber \\
h_q^{(1)}(2)  & = & h(q_{s1}). 
    \end{eqnarray}    
and through entropic combinatorial considerations obtain
\begin{eqnarray}\label{eq:ult1eq10}
h_q(2) &  = &   h_q^{(0)}(2) + h_q^{(1)}(2).
    \end{eqnarray}    
To see where (\ref{eq:ult1eq9}) comes from, one first observes 
\begin{eqnarray}\label{eq:ult1eq11}
2^{h_q(2) n} &  = &  2^{ h_q^{(0)}(2) n} 2^{ h_q^{(1)}(2) n}.
    \end{eqnarray}    
One then notes that $ 2^{ h_q^{(0)}(2) n}=2^n$ is the number of possible choices for $\x^{(1)}$. For any fixed such choice $\x^{(1)}$ one has precisely $2^{ h_q^{(1)}(2) n}$ choices for $\x^{(2)}$ where the number of overlapping signs between the two is $q_{s1} n$. On the other hand, keeping in mind that the elements of $G$ are independent standard normals, one has for $p_q(2)$ 
\begin{eqnarray}\label{eq:ult1eq12}
 p_q(2) & = &    \mP\lp |G_{1,:}\x^{(1)}|\leq \kappa,|G_{1,:}\x^{(2)}|\leq \kappa, \lp\x^{(1)}\rp^T\x^{(2)}=q  \rp,
    \end{eqnarray}    
where $G_{i,:}$ stands for the $i$-th row of $G$. Setting
\begin{eqnarray}\label{eq:ult1eq12a0}
g_1 & = & G_{1,:}\x^{(1)} \nonumber \\
g_2 & = & G_{1,:}\x^{(2)},
\end{eqnarray}
and recognizing that $g_1$ and $g_2$ are $q$-correlated centered Gaussians, one finds
\begin{eqnarray}\label{eq:ult1eq13}
 p_q(2) & = &   \frac{1}{\lp 2\pi\rp^{\frac{2}{2}} \det(A)^{\frac{1}{2}} }\int_{|g_1|\leq \kappa,|g_2|\leq \kappa} e^{-\frac{1}{2}\begin{bmatrix}
                  g_1 \\ g_2        \end{bmatrix}^T\begin{bmatrix}
                                                     1 & q  \\
                                                     q & 1   
                                                   \end{bmatrix}^{-1}\begin{bmatrix}
                  g_1 \\ g_2        \end{bmatrix}}dg_1dg_2.
    \end{eqnarray}    
Setting  further
\begin{eqnarray}\label{eq:ult1eq14}
\g  =  \begin{bmatrix}
                  g_1 \\ g_2        \end{bmatrix} 
\quad                  \mbox{ and } \quad
              Q  =     \begin{bmatrix}
                                                     1 & q  \\
                                                     q & 1   
                                                   \end{bmatrix},
\end{eqnarray}
one can rewrite (\ref{eq:ult1eq13}) in a more compact form
\begin{eqnarray}\label{eq:ult1eq15}
 p_q(2) & = &   \frac{1}{\lp 2\pi\rp^{\frac{2}{2}} \det(Q)^{\frac{1}{2}} }\int_{|\g|\leq \kappa} e^{-\frac{1}{2}\g^TQ^{-1}\g}d\g.
    \end{eqnarray}    
A combination of (\ref{eq:ult1eq9}), (\ref{eq:ult1eq10}), and  (\ref{eq:ult1eq15})  is then sufficient to utilize (\ref{eq:ult1eq7}) to obtain the following upper bound
\begin{eqnarray}\label{eq:ult1eq16}
 \alpha_{2ogp}(\kappa) &  \leq  & \min_{q} \lp - \frac{h_q(2) \log(2) }{\log(p_q(2))} \rp \triangleq  \bar{\alpha}_{ult_1}(\kappa;[1,2]).
    \end{eqnarray}
To ensure concreteness we consider canonical $\kappa=1$ scenario. After numerical evaluations one obtains results given in Table \ref{tab:tab2}. These should be identical to what was obtained in \cite{Bald20} and similar to $1.71$ estimate from \cite{GamKizPerXu22}.
\begin{table}[h]
\caption{Union upper bound on  $\alpha_{2ogp}(1)$  }\vspace{.1in}
\centering
\def\arraystretch{1.2}
 \begin{tabular}{||c||c||}\hline\hline
 \hspace{-0in}                                               $k$     &  $2$   \\ \hline\hline
    $q$  & $ 0.9689$ \\ \hline\hline 
  $\bar{\alpha}_{ult_1}(1;[1,k])$ &   \bl{$\mathbf{1.7001}$}    \\ \hline\hline
  \end{tabular}
\label{tab:tab2}
\end{table}

\subsubsection{$3$-OGP}
\label{sec:ult13ogp}

The compactness of the above forms allows a quick extension to $3$-OGP. One first writes the following analogue to (\ref{eq:ult1eq7})
\begin{eqnarray}\label{eq:ult13ogpeq7}
& &  \hspace{-.4in}\mP(\exists \hspace{.03in}\x^{(i)}\in\cX(3;Q) \mbox{ such that } |G\x^{(i)}|\leq, i=1,2,3 ) 
\nonumber \\
 & 
 \leq &
|\cX(3;Q)|  
 \mP\lp \forall i, |G\x^{(i)}|\leq \kappa\1, \lp\x^{(i)}\rp^T\x^{(j)}=q, i\neq j  \rp
  \nonumber \\
&  =  & \lp2^{h_q(3)} p_q(3)^{\alpha}\rp^n,
    \end{eqnarray}
    where $h_q(3)$ and  $p_q(3)$ are again the combinatorial and probabilistic factors, respectively, given as
\begin{eqnarray}\label{eq:ult13ogpeq8}
h_q(3) &  = & \frac{\log_2\lp  |\cX(3;Q)|  \rp}{n} \nonumber \\
p_q(3) & = &   \lp  \mP\lp \forall i, |G\x^{(i)}|\leq \kappa\1, \lp\x^{(i)}\rp^T\x^{(j)}=q, i\neq j  \rp
 \rp^{\frac{1}{m}} ,
    \end{eqnarray}    
 and $|\cX(3;Q)| $ is the cardinality of set $\cX(3;Q)$.

Evaluation of the combinatorial factor is now a bit more involved but still doable. One sets 
\begin{eqnarray}\label{eq:ult13ogpeq9}
q_{sx1} & = & (1+q)/2 \nonumber \\
q_{sx2} & = & (1+q)/2 \nonumber \\
q_{s1} & = &  q_{sx1}
 \nonumber \\ 
    q_{s2} &  = &  q_{sx2} - (1-q_{s1})/2 
\nonumber \\
h_q^{(0)}(3)  & = & 1
\nonumber \\
h_q^{(1)}(3)  & = & h(q_{s1})
\nonumber 
\\
q_{s21} & = & q_{s2} /q_{s1} 
\nonumber 
\\
 q_{s22} & = & ( q_{sx2} - q_{s2} )/(1-q_{s1}) 
 \nonumber \\
h_q^{(2)}(3) & = &    q_{s1}h(q_{s21}) +  (1-q_{s1})h(q_{s22})  .
    \end{eqnarray}    
    and through entropic combinatorial considerations obtains
\begin{eqnarray}\label{eq:ult13ogpeq10}
h_q(3) &  = &   h_q^{(0)}(3) + h_q^{(1)}(3) + h_q^{(2)}(3) .
    \end{eqnarray}    
The reasoning is as for $2$-OGP; $ 2^{ h_q^{(0)}(3) n}=2^n$ is the number of possible choices for $\x^{(1)}$ and for any fixed choice $\x^{(1)}$, $2^{ h_q^{(1)}(3) n}$ is the number of possible choices for $\x^{(2)}$ where the number of overlapping signs between $\x^{(1)}$  and $\x^{(2)}$   is $q_{s1} n$. For fixed $\x^{(1)}$ and $\x^{(2)}$ (with the number of overlapping signs $q_{s1} n$), $2^{ h_q^{(2)}(3) n}$ is the number of possible choices for $\x^{(3)}$ where the number of overlapping signs between $\x^{(1)}$  and $\x^{(3)}$  and $\x^{(2)}$  and $\x^{(3)}$  is also $q_{s1} n$. However, one has to account for possible options where these overlaps can happen. In particular, $q_{s1}n$ is the number of places where signs between   $\x^{(1)}$  and $\x^{(2)}$ match. Out of these $q_{s1}n$ places, there are $q_{s2}n$ places where they also match the sign of $\x^{(3)}$. On the other hand, out of the $(1-q_{s1})n$ places where $\x^{(1)}$  and $\x^{(2)}$ don't match, $\x^{(3)}$ matches both at exactly half, i.e., at $\frac{1-q_{s1}}{2}n$ places. $h_q^{(21)}(3) $ accounts for placement of $q_{s2} n$ signs within  $q_{s1} n$ locations, whereas $h_q^{(22)}(3) $ accounts for placement of $\frac{1-q_{s1}}{2}n$ signs  within  $(1-q_{s1})n$ locations. Together they give $h_q^{(2)}(3) $. (The above is written in a way that is a bit more complicated than necessary as one effectively has $\frac{1-q_{s1}}{2} =q_{sx2}-q_{s2}$; however, the reasons for the above generic writing will become clearer later on as we progress with the analyses of more complex structures.)

Analogously to (\ref{eq:ult1eq15}) we have
\begin{eqnarray}\label{eq:ult13ogpeq15}
 p_q(3) & = &   \frac{1}{\lp 2\pi\rp^{\frac{3}{2}} \det(Q)^{\frac{1}{2}} }\int_{|\g|\leq \kappa} e^{-\frac{1}{2}\g^TQ^{-1}\g}d\g,
    \end{eqnarray}    
where analogously to (\ref{eq:ult1eq14})
\begin{eqnarray}\label{eq:ult13ogpeq14}
\g  =  \begin{bmatrix}
                  g_1 \\ g_2 \\g_3        \end{bmatrix} 
\quad                  \mbox{ and } \quad
              Q  =     \begin{bmatrix}
                                                     1 & q & q \\
                                                     q & 1 & q \\ 
                                                     q & q & 1
                                                   \end{bmatrix}.
\end{eqnarray}
We now observe 
\begin{eqnarray}\label{eq:ult13ogpeq15a0}
Q^{-1}= \lp (1-q)I +q\1\1^T \rp^{-1}
= \frac{1}{1-q}I -\frac{1}{(1-q)^2}\1\1^T\lp \frac{1}{q} +\frac{1}{1-q}\1^T\1 \rp^{-1} 
= c_{w1} I + c_{w2}\1\1^T,
\end{eqnarray}
where
\begin{eqnarray}\label{eq:ult13ogpeq15a1}
c_{w1} = \frac{1}{1-q} \quad \mbox{ and } \quad  c_{w2} = -\frac{1}{(1-q)^2}\lp \frac{1}{q} +\frac{3}{1-q} \rp^{-1} .
\end{eqnarray}
From (\ref{eq:ult13ogpeq15}) we then also have
\begin{eqnarray}\label{eq:ult13ogpeq15a2}
 p_q(3) & = &   \frac{1}{\lp 2\pi\rp^{\frac{3}{2}} \det(Q)^{\frac{1}{2}} }\int_{|\g|\leq \kappa} e^{-\frac{1}{2}\g^T\lp    c_{w1} I + c_{w2}\1\1^T  \rp \g}d\g
 \nonumber \\
 & = &   \frac{1}{\lp 2\pi\rp^{\frac{3}{2}} \det(Q)^{\frac{1}{2}} }\int_{|\g|\leq \kappa} e^{-\frac{c_{w1}}{2}\g^T\g  -\frac{ c_{w2}}{2}\g^T\lp \1\1^T  \rp \g}d\g
 \nonumber \\
 & = &   \frac{1}{\lp 2\pi\rp^{\frac{3}{2}} \det(Q)^{\frac{1}{2}} }\int_{|\g|\leq \kappa} e^{-\frac{c_{w1}}{2}\g^T\g  }  \frac{1}{\sqrt{2\pi}}\int_z e^{\sqrt{-c_{w2}}\1^T \g z -\frac{z^2}{2}} dz d\g
 \nonumber \\
 & = &   \frac{1}{\det(Q)^{\frac{1}{2}} }
\frac{1}{\sqrt{2\pi}} \int_z
\lp
\frac{1}{\lp 2\pi\rp^{\frac{3}{2}} }
 \int_{|\g|\leq \kappa} e^{-\frac{c_{w1}}{2}\g^T\g  +\sqrt{-c_{w2}}\1^T \g z } d\g 
 \rp
 e^{-\frac{z^2}{2}} dz
 \nonumber \\
 & = &   \frac{1}{\det(Q)^{\frac{1}{2}} }
\frac{1}{\sqrt{2\pi}} \int_z
\lp
\frac{1}{\lp 2\pi\rp^{\frac{1}{2}} }
 \int_{|g_1|\leq \kappa} e^{-\frac{c_{w1}}{2}g_1^2 + \sqrt{-c_{w2}}g_1 z } dg_1 
 \rp^3
 e^{-\frac{z^2}{2}} dz.
    \end{eqnarray}    
After setting
\begin{eqnarray}
\label{eq:ult13ogpeq15a3}
D = \frac{\sqrt{-c_{w2}} }  {\sqrt{c_{w1}} } z \quad \mbox{ and } \quad  E = \sqrt{c_{w1}}\kappa,
\end{eqnarray}
one finds
\begin{eqnarray}\label{eq:ult13ogpeq15a4}
 I_0 =
\frac{1}{\lp 2\pi\rp^{\frac{1}{2}} }
 \int_{|g_1|\leq \kappa} e^{-\frac{c_{w1}}{2}g_1^2 + \sqrt{-c_{w2}}g_1 z } dg_1 
 =
 -\frac{{\mathrm{e}}^{\frac{D^2}{2}}\,\left(\mathrm{erf}\left(\frac{\sqrt{2}\,D}{2}
 -\frac{\sqrt{2}\,E}{2}\right)-\mathrm{erf}\left(\frac{\sqrt{2}\,D}{2}
 +\frac{\sqrt{2}\,E}{2}\right)\right)}{2},
    \end{eqnarray}    
and
\begin{eqnarray}\label{eq:ult13ogpeq15a5}
 p_q(3) & = &   
    \frac{1}{\sqrt{c_{w1}}^3\det(Q)^{\frac{1}{2}} }
\frac{1}{\sqrt{2\pi}} \int_z
I_0^3
 e^{-\frac{z^2}{2}} dz.
    \end{eqnarray}    
A combination of (\ref{eq:ult13ogpeq9}), (\ref{eq:ult13ogpeq10}), (\ref{eq:ult13ogpeq15a1}), (\ref{eq:ult13ogpeq15a3}), and  (\ref{eq:ult13ogpeq15a5})  is then sufficient to utilize (\ref{eq:ult13ogpeq7}) to obtain the following upper bound
\begin{eqnarray}\label{eq:ult13ogpeq16}
 \alpha_{3ogp}(\kappa) &  \leq  & \min_{q} \lp - \frac{h_q(3) \log(2) }{\log(p_q(3))} \rp \triangleq  \bar{\alpha}_{ult_1}(\kappa;[1,3]).
    \end{eqnarray}
Numerical evaluations give results shown in Table \ref{tab:tab3}. As earlier, this closely matches estimates from \cite{Bald20,GamKizPerXu22}.
\begin{table}[h]
\caption{Union upper bound on  $\alpha_{3ogp}(1)$  }\vspace{.1in}
\centering
\def\arraystretch{1.2}
 \begin{tabular}{||c||c|c||}\hline\hline
 \hspace{-0in}                                               $k$     &  $2$    &  $3$   \\ \hline\hline
    $q$  & $ 0.9689$    & $ 0.9780$ \\ \hline\hline 
  $\bar{\alpha}_{ult_1}(1;[1,k])$ &   \bl{$\mathbf{1.7001}$} &   \bl{$\mathbf{1.6664}$}    \\ \hline\hline
  \end{tabular}
\label{tab:tab3}
\end{table}

\subsubsection{$4$-OGP}
\label{sec:ult14ogp}

Following what was done in the previous subsection we immediately write the following 
analogue to (\ref{eq:ult13ogpeq7})
\begin{eqnarray}\label{eq:ult14ogpeq7}
& &  \hspace{-.4in}\mP(\exists \hspace{.03in}\x^{(i)}\in\cX(4;Q) \mbox{ such that } |G\x^{(i)}|\leq \kappa\1, i=1,2,3,4 ) 
 \leq 
  \lp 2^{h_q(4)} p_q(4)^{\alpha}\rp^n,
    \end{eqnarray}
where $h_q(4)$ and $p_q(4)$ are corresponding combinatorial and probabilistic factors. We then 
immediately have
\begin{eqnarray}\label{eq:ult14ogpeq15a5b0}
h_q(4) &  = & \frac{\log_2\lp  |\cX(4;Q)|  \rp}{n} ,
\end{eqnarray}
and the following analogue to (\ref{eq:ult13ogpeq15a5})
\begin{eqnarray}\label{eq:ult14ogpeq15a5}
 p_q(4) & = &   
    \frac{1}{\sqrt{c_{w1}}^4\det(Q)^{\frac{1}{2}} }
\frac{1}{\sqrt{2\pi}} \int_z
I_0^4
 e^{-\frac{z^2}{2}} dz,
    \end{eqnarray}    
with $I_0$ as defined in (\ref{eq:ult13ogpeq15a3}) and (\ref{eq:ult13ogpeq15a4})  for $c_{w1}$ and $c_{w2}$  
\begin{eqnarray}\label{eq:ult14ogpeq15a1}
c_{w1} = \frac{1}{1-q} \quad \mbox{ and } \quad  c_{w2} = -\frac{1}{(1-q)^2}\lp \frac{1}{q} +\frac{4}{1-q} \rp^{-1} .
\end{eqnarray}

The fourth OGP is sufficiently high so that we can put forth a generic procedure to determine the combinatorial factor. To that end we set

\begin{eqnarray}\label{eq:ult14ogpeq15a1b0}
A^{(4)}_{eq} =\begin{bmatrix}
          1 & 1 & 1 & 0 & 0 & 0 & 1  \\
          0 & 0 & 1 & 1 & 1 & 0 & 1  \\
          0 & 1 & 0 & 1 & 0 & 1 & 1  \\
          0 & 0 & 1 & 0 & 0 & 1 & 0 \\
          0 & 1 & 0 & 0 & 1 & 0 & 0 \\
          1 & 0 & 0 & 1 & 0 & 0 & 0 
\end{bmatrix},
\end{eqnarray}
and
\begin{eqnarray}\label{eq:ult14ogpeq15a1b1}
\b^{(4)}_{eq} =\begin{bmatrix}
           q_{sx1} & q_{sx1} & q_{sx1} & q_{s1}-q_{s2} & 1/2(1-q_{s1})  & 1/2(1-q_{s1}) 
\end{bmatrix}^T,
\end{eqnarray}
Let the numbers of rows and columns of $A_{eq}^{(k)}$ be $l_{r,k}$  and $l_{c,k}$ , respectively (for $k=4$ one then has that $l_{r,4}$  and $l_{c,4}$ are the numbers of rows and columns of $A_{eq}^{(4)}$. Then
 \begin{eqnarray}\label{eq:ult14ogpeq15a1b3}
h_q(4) = 1 + \max_{\a} & & h_v(\a) 
\nonumber \\
\mbox{subject to} & &  A^{(4)}_{eq}\a =\b^{(4)}_{eq}
\nonumber \\
  & &  \sum_{i=1}^{l_{c,4}+1}\a_i =1 .
    \end{eqnarray}    
We note the key convex property of the above program. To see how/why the above program evaluates the combinatorial factor we look at the following scheme.

\begin{equation}\label{eq:ult14ogpeq15a1b4}
\begin{bmatrix}
  ++++++++ & q_{s3}n &  ++++++++ & q_{s2}n & ++++++++  & q_{s1}n &  ++++++++ \\
  ++++++++ & q_{s3}n &  ++++++++ & q_{s2}n & ++++++++  & q_{s1}n &  -------- \\
  ++++++++ & q_{s3}n &  ++++++++ & q_{s2}n & --------  & q_{s1}n &  ++++---- \\
           &         &           &         &           &   & \bar{\a}_2 \quad \quad \quad  \bar{\a}_1    \\
  ++++++++ & q_{s3}n &  -------- & q_{s2}n & ++++----  & q_{s1}n &  ++--++-- \\
    \a_7   &         &           &         & \a_3\quad\quad\quad \a_6 &  & \a_2\quad  \a_5 \quad \a_1\quad \a_4  
\end{bmatrix}.
\end{equation}
Before we get into the details, we emphasize that in all rows, quantities $q_{s1}n$, $q_{s2}n$, and $q_{s3}n$ are not the elements of the row. They serve as an orientation regarding the number of signs. For example, $q_{s3}n$ means that to the left of that number there are exactly $q_{s3}n$ elements in the row. Similarly for $q_{s2}n$, it denotes the number of elements in the row (starting from the left) up to the point where $q_{s2}n$ is located. Also, numbers $\a_i$ and $\bar{\a}_i$ count the number of signs in the group right above them. In other words, the only elements in the scheme are '+' and '-' signs and what matters is where they are located.

To get a felling about the meaning of the scheme, we start with the first row. It basically corresponds to $\x^{(1)}$ having all positive signs (due to symmetry any combination of signs is fine; for easiness of presentation we choose all pluses (this practically means $\x^{(1)}=\frac{\1}{\sqrt{n}}$)). Then we proceed to the second row and observe that it emulates possible options for $\x^{(2)}$. The signs between the first and the second row match at exactly $q_{s1}n$ locations and mismatch and $(1-q_{s1})n$, precisely as it should be. Third row emulates possible options for $\x^{(3)}$. Assume that  out of $q_{s1}n$ locations where $\x^{(1)}$ and $\x^{(2)}$ match, $\x^{(3)}$ also matches them at $q_{s2}n$ locations. Then the fourth (most right) portion of the row shows the necessary partition  of the remaining  $(1-q_{s1})n$ locations where $\x^{(1)}$ and $\x^{(2)}$ mismatch and where $\x^{(3)}$ has to match each of them to make up for the overall  $q_{s1}n$ matching locations. One immediately notes that the split must be such that both portions $\bar{\a}_2$ and $\bar{\a}_1$ are exactly half of the interval, i.e, such that they are equal to each other. In other words, one has $\bar{\a}_2=\bar{\a}_1=\frac{1}{2}(1-q_{s1})n$. The fourth row shows possible options for $\x^{(4)}$. Say there are $q_{s3}n$ locations where $\x^{(4)}$  matches all three, $\x^{(1)}$, $\x^{(2)}$, and $\x^{(3)}$ where they themselves also match (which is the first $q_{s2}n$ locations). One then needs to partition the remaining intervals in $+$ and $-$ portions. The length of these portions are precisely the elements of $\a$ as denoted in (\ref{eq:ult14ogpeq15a1b4}). Moreover, these lengths should be such that all overlaps are as needed.

We now check if constraints generated by rows of $A_{eq}^{(4)}$ indeed match the overlap constraints. The equality constraint in (\ref{eq:ult14ogpeq15a1b3}) associated with the first row of $A_{eq}^{(4)}$ gives
\begin{equation}\label{eq:ult14ogpeq15a1b6}
\a_1+\a_2+\a_3 +\a_7 = q_{sx1}.
\end{equation}
Looking at (\ref{eq:ult14ogpeq15a1b3})  this precisely gives the sign overlap between the first and the fourth row  (i.e., between $\x^{(1)}$ and $\x^{(4)}$). Constraint associated with the second row gives
\begin{equation}\label{eq:ult14ogpeq15a1b7}
\a_3+\a_4+\a_5 +\a_7 = q_{sx1}.
\end{equation}
Looking at (\ref{eq:ult14ogpeq15a1b3})  this precisely gives the sign overlap between the second and the fourth row (i.e., between $\x^{(2)}$ and $\x^{(4)}$). Constraint associated with the third row gives
\begin{equation}\label{eq:ult14ogpeq15a1b8}
\a_2+\a_4+\a_6 +\a_7 = q_{sx1}.
\end{equation}
which precisely gives the sign overlap between the third and the fourth row (i.e., between $\x^{(3)}$ and $\x^{(4)}$). The remaining three equations ensure that sums of new partitions match the length of the interval that are being newly partitioned. For example, the fourth constraint gives
\begin{equation}\label{eq:ult14ogpeq15a1b9}
\a_3+\a_6  = (q_{s1}-q_{s2})n,
\end{equation}
which is precisely how it should be. The remaining two give
\begin{equation}\label{eq:ult14ogpeq15a1b10}
\a_2+\a_5  = \a_1 +\a_4 = \frac{1}{2}(1-q_{s2})n=\bar{\a}_2=\bar{\a}_1,
\end{equation}
which is precisely how it should be.

Since all the constraints match the needed overlap properties, the only thing that is left is to ensure that $h_q(4)$ is indeed the entropic exponent. However, that follows immediately after one sets $h_q^{(i)}(4)=h_q^{(i)}(3)$ for i=1,2, defines 
\begin{eqnarray}\label{eq:ult14ogpeq15a1b11}
  q_{s31} & = & \a_7 /q_{s2} 
\nonumber 
\\
 q_{s32} & = & \a_3 /(\a_3+\a_6) 
\nonumber 
\\
 q_{s33} & = & \a_2 /(\a_2+\a_5) 
\nonumber 
\\
 q_{s34} & = & \a_1 /(\a_1+\a_4) 
\nonumber 
\\
   h_q^{(3)}(4) & = &  q_{s3}h(q_{s41})  +(\a_3+\a_6) h(q_{s42}) 
+(\a_2+\a_5)h(q_{s43})  +  (\a_1+\a_4)h(q_{s44}) ,
    \end{eqnarray}    
 and observes that entropy identities give
\begin{eqnarray}\label{eq:ult14ogpeq15a1b12}
   h_q(4) =  1+ h_v(\a) = 1 +  h_q^{(1)}(4) +  h_q^{(2)}(4) + h_q^{(3)}(4),
\end{eqnarray}
for $\a$ being the solution of (\ref{eq:ult14ogpeq15a1b3}).

A combination of  (\ref{eq:ult13ogpeq15a3}), (\ref{eq:ult13ogpeq15a4}) , (\ref{eq:ult14ogpeq15a5}), (\ref{eq:ult14ogpeq15a1}), and  (\ref{eq:ult14ogpeq15a1b3}) (or (\ref{eq:ult14ogpeq15a1b12}))  is then sufficient to utilize (\ref{eq:ult14ogpeq7}) to obtain the following upper bound
\begin{eqnarray}\label{eq:ult14ogpeq16}
 \alpha_{4ogp}(\kappa)&  \leq  &  \min_{q} \lp  - \frac{h_q(4) \log(2) }{\log(p_q(4))} \rp \triangleq  \bar{\alpha}_{ult_1}(\kappa;[1,4]).
    \end{eqnarray}
Numerical evaluations give results shown in Table \ref{tab:tab4}. Looking carefully at Table \ref{tab:tab1}, we now observe that the value given in Table \ref{tab:tab4} magically approaches the value $\approx 1.6576$ obtained on the third level of partial lifting in \cite{Stojnicalgsbp26} (we believe that procedure from \cite{Bald20} would give similar value as well).
\begin{table}[h]
\caption{Union upper bound on  $\alpha_{4ogp}(1)$  }\vspace{.1in}
\centering
\def\arraystretch{1.2}
 \begin{tabular}{||c||c|c|c||}\hline\hline
 \hspace{-0in}                                               $k$     &  $2$    &  $3$      &  $4$   \\ \hline\hline
    $q$  & $ 0.9689$ & $ 0.9780$  & $ 0.9840$ \\ \hline\hline 
  $\bar{\alpha}_{ult_1}(1;[1,k])$ &   \bl{$\mathbf{1.7001}$} &   \bl{$\mathbf{1.6664}$}  &   \bl{$\mathbf{1.6578}$}    \\ \hline\hline
  \end{tabular}
\label{tab:tab4}
\end{table}

\subsubsection{$5$-OGP}
\label{sec:ult15ogp}

The formalism from the previous subsection is sufficient to proceed with higher OGPs. However, one has to be careful as $\a$ partition determined for $4$-OGP might not be the optimal for $5$-OGP. Before getting to the point to discuss these intricacies we focus on the results that can be directly generalized from  $4$-OGP.

Analogously to (\ref{eq:ult14ogpeq7}) we now have
 \begin{eqnarray}\label{eq:ult15ogpeq7}
& &  \hspace{-.4in}\mP(\exists \hspace{.03in}\x^{(i)}\in\cX(5;q) \mbox{ such that } |G\x^{(i)}|\leq \kappa\1, i=1,2,3,4,5 ) 
 \leq 
  \lp 2^{h_q(5)} p_q(5)^{\alpha}\rp^n,
    \end{eqnarray}
where $h_q(5)$ and $p_q(5)$ are corresponding combinatorial and probabilistic factors. One can also write   the following analogue to (\ref{eq:ult14ogpeq15a5})
\begin{eqnarray}\label{eq:ult15ogpeq15a5}
 p_q(4) & = &   
    \frac{1}{\sqrt{c_{w1}}^5\det(Q)^{\frac{1}{2}} }
\frac{1}{\sqrt{2\pi}} \int_z
I_0^5
 e^{-\frac{z^2}{2}} dz,
    \end{eqnarray}    
with $I_0$ defined in (\ref{eq:ult13ogpeq15a3}) and (\ref{eq:ult13ogpeq15a4})  for $c_{w1}$ and $c_{w2}$  
\begin{eqnarray}\label{eq:ult15ogpeq15a1}
c_{w1} = \frac{1}{1-q} \quad \mbox{ and } \quad  c_{w2} = -\frac{1}{(1-q)^2}\lp \frac{1}{q} +\frac{5}{1-q} \rp^{-1} .
\end{eqnarray}
We again find it useful to look at the following scheme  

{\small\begin{equation}\label{eq:ult15ogpeq15a1b4}
\begin{bmatrix}
  ++++& q_{s4}n & ++++ & q_{s3}n &  ++++++++ & q_{s2}n & ++++++++  & q_{s1}n &  ++++++++ \\
  ++++& q_{s4}n & ++++ & q_{s3}n &  ++++++++ & q_{s2}n & ++++++++  & q_{s1}n &  -------- \\
  ++++& q_{s4}n & ++++ & q_{s3}n &  ++++++++ & q_{s2}n & --------  & q_{s1}n &  ++++---- \\
        &  &  &         &           &         &           &   & \bar{\a}_2 \quad \quad \quad  \bar{\a}_1    \\
  ++++& q_{s4}n & ++++ & q_{s3}n &  -------- & q_{s2}n & ++++----  & q_{s1}n &  ++--++-- \\
      & \a_{22} &    &     &           &         & \a_{18}\quad\quad\quad \a_{21} &  & \a_{17}\quad  \a_{20} \quad \a_{16}\quad \a_{19}  
           \\
  ++++& q_{s4}n &  ---- & q_{s3}n &  ++++---- & q_{s2}n & ++--++--  & q_{s1}n &  +-+-+-+- \\
    \a_{15} &   & &         &        \a_{7}\quad \quad \quad \a_{14}   &         & \a_6\quad \a_{13}\quad \a_5\quad \a_{12}  &  & {\small{\a_4  \a_{11} \a_3 \a_{10} \a_2   \a_9  \a_1  \a_8}}  
\end{bmatrix}.
\end{equation}}

The partition of the fifth row is completely analogous to the procedure described when we discussed $4$-OGP. However, one now has to keep in mind that $\a_{16},\a_{17},\dots,\a_{22}$ (which correspond to $\a_1,\a_2,\dots,\a_7$ from $4$-OGP are now again subject to optimization as the ones determined for $4$-OGP are not necesarily optimal for $5$-OGP).

Now, an analogous $A_{eq}$ matrix to the one in (\ref{eq:ult14ogpeq15a1b0}) would be
\begin{eqnarray}\label{eq:ult15ogpeq15a1b0}
A^{(5,1)}_{eq} =\begin{bmatrix}
 1 & 1 & 1 & 1 & 1 & 1 & 1 & 0 & 0 & 0 & 0 & 0 & 0 & 0 & 1  \\
 0 & 0 & 0 & 0 & 1 & 1 & 1 & 1 & 1 & 1 & 1 & 0 & 0 & 0 & 1  \\
 0 & 0 & 1 & 1 & 0 & 0 & 1 & 1 & 1 & 0 & 0 & 1 & 1 & 0 & 1  \\
 0 & 1 & 0 & 1 & 0 & 1 & 0 & 1 & 0 & 1 & 0 & 1 & 0 & 1 & 1  \\
 0 & 0 & 0 & 0 & 0 & 0 & 1 & 0 & 0 & 0 & 0 & 0 & 0 & 1 & 0  \\
 0 & 0 & 0 & 0 & 0 & 1 & 0 & 0 & 0 & 0 & 0 & 0 & 1 & 0 & 0  \\
 0 & 0 & 0 & 0 & 1 & 0 & 0 & 0 & 0 & 0 & 0 & 1 & 0 & 0 & 0  \\
 0 & 0 & 0 & 1 & 0 & 0 & 0 & 0 & 0 & 0 & 1 & 0 & 0 & 0 & 0  \\ 
 0 & 0 & 1 & 0 & 0 & 0 & 0 & 0 & 0 & 1 & 0 & 0 & 0 & 0 & 0  \\ 
 0 & 1 & 0 & 0 & 0 & 0 & 0 & 0 & 1 & 0 & 0 & 0 & 0 & 0 & 0  \\ 
 1 & 0 & 0 & 0 & 0 & 0 & 0 & 1 & 0 & 0 & 0 & 0 & 0 & 0 & 0  
\end{bmatrix}.
\end{eqnarray}
However, we will also need the following matrix
\begin{eqnarray}\label{eq:ult15ogpeq15a1b1}
A^{(5,2)}_{eq} =\begin{bmatrix}
   0 & 0 &  0 &  0 & 0 & 0  &   0 \\   
   0 & 0 &  0 &  0 & 0 & 0  &   0 \\   
   0 & 0 &  0 &  0 & 0 & 0  &   0 \\   
   0 & 0 &  0 &  0 & 0 & 0  &   0 \\   
   0 & 0 &  0 &  0 & 0 & 0  &   1 \\   
   0 & 0 & -1   & 0 & 0 &  0   &  0  \\  
    0 & 0 & 0  &   0 & 0 & -1 &   0   \\
   0 & -1 & 0   & 0 & 0 & 0    & 0  \\ 
   0 & 0 & 0   &  0 & -1 & 0  &  0  \\ 
   -1 & 0 & 0   & 0 & 0 & 0   &  0  \\ 
    0 & 0 & 0  &   -1 & 0 & 0  &  0   
\end{bmatrix}.
\end{eqnarray}
Then we set
\begin{eqnarray}\label{eq:ult15ogpeq15a1b2}
A^{(5)}_{eq} =\begin{bmatrix}
   A^{(5,1)}_{eq}  &   A^{(5,2)}_{eq} \\   
   \0 & A^{(4)}_{eq}\end{bmatrix},
\end{eqnarray}
and
{\small\begin{eqnarray}\label{eq:ult15ogpeq15a1b1}
\b^{(5)}_{eq} =\begin{bmatrix}
 q_{sx1} & q_{sx1} & q_{sx1} &  q_{sx1} &  q_{s2} & 0 & 0 & 0 & 0 & 0 & 0 &
           q_{sx1} & q_{sx1} & q_{sx1} & q_{s1}-q_{s2} & \frac{1-q_{s1}}{2}  &  \frac{1-q_{s1}}{2}  
\end{bmatrix}^T,
\end{eqnarray}}
We then have for the combinatorial factor
\begin{eqnarray}\label{eq:ult15ogpeq15a5b0}
h_q(5) &  = & \frac{\log_2\lp  |\cX(5;Q)|  \rp}{n} ,
\end{eqnarray}
the following program
\begin{eqnarray}\label{eq:ult15ogpeq15a1b3}
h_q(5) = 1 + \max_{\a} & & h_v(\a) 
\nonumber \\
\mbox{subject to} & &  A^{(5)}_{eq}\a =\b^{(5)}_{eq}
\nonumber \\
  & &  \sum_{i=1}^{l_{c,5}+1}\a_i =1 .
    \end{eqnarray}    
As was the program in (\ref{eq:ult14ogpeq15a1b3}), the above program is also convex. Following the methodology of the previous section, we now check if constraints generated by the rows of $A_{eq}^{(5)}$ indeed match the overlap constraints. The equality constraint in (\ref{eq:ult15ogpeq15a1b3}) associated with the first row of $A_{eq}^{(5)}$ gives
\begin{equation}\label{eq:ult15ogpeq15a1b6}
\a_1+\a_2+\a_3 +\a_4+\a_5+\a_6 +\a_7 +\a_{15}  = q_{sx1},
\end{equation}
which is precisely the sign overlap between the first and the fifth row in (\ref{eq:ult15ogpeq15a1b4}) (i.e., between $\x^{(1)}$ and $\x^{(5)}$). The second row constraint  gives
\begin{equation}\label{eq:ult15ogpeq15a1b7}
\a_5+\a_6+\a_7 +\a_8+\a_9+\a_{10} +\a_{11} +\a_{15}  = q_{sx1},
\end{equation}
which gives the sign overlap between the second and the fifth row (i.e., between $\x^{(2)}$ and $\x^{(5)}$). The third row constraint gives
\begin{equation}\label{eq:ult15ogpeq15a1b8}
\a_3+\a_4+\a_7 +\a_8+\a_9+\a_{12} +\a_{13} +\a_{15}  = q_{sx1},
\end{equation}
which   gives the sign overlap between the third and the fifth row (i.e., between $\x^{(3)}$ and $\x^{(5)}$).  The fourth row constraint gives
\begin{equation}\label{eq:ult15ogpeq15a1b8b0}
\a_2+\a_4+\a_6 +\a_8+\a_{10}+\a_{12} +\a_{14} +\a_{15}  = q_{sx1},
\end{equation}
which  is the sign overlap between the fourth and the fifth row (i.e., between $\x^{(4)}$ and $\x^{(5)}$).  The following $7$ constraints ensure that sums of new partitions match the length of the intervals of the partition from the previous row (these are now also variables, $\a_{16},\a_{17},\dots,\a_{22}$). For example, the fifth constraint gives
\begin{equation}\label{eq:ult15ogpeq15a1b9}
\a_7+\a_{14} +\a_{22} = q_{s2} n,
\end{equation}
which is also precisely how it should be. The remaining six give
\begin{equation}\label{eq:ult15ogpeq15a1b10}
\a_6+\a_{13} -\a_{18} = \a_5 +\a_{12} -\a_{21} = \a_4 +\a_{11} -\a_{17} = \a_3 +\a_{10} -\a_{20} = \a_2 +\a_{9} -\a_{16} = \a_1 +\a_8 -\a_{19} = 0,
\end{equation}
which also precisely matches what is needed. The remaining seven constraint relate to $\a_{16},\a_{17},\dots,\a_{22}$ and, by the recursive construction of $A_{eq}^{(5)}$ and $b_{eq}^{(5)}$, are automatically identical to the ones given for $4$-OGP as they should be.

As earlier, we again have that $h_q(5)$ is the entropic exponent. A combination of  (\ref{eq:ult13ogpeq15a3}), (\ref{eq:ult13ogpeq15a4}) , (\ref{eq:ult15ogpeq15a5}), (\ref{eq:ult15ogpeq15a1}), and  (\ref{eq:ult15ogpeq15a1b3})  is then sufficient to utilize (\ref{eq:ult15ogpeq7}) and obtain the following upper bound
\begin{eqnarray}\label{eq:ult15ogpeq16}
 \alpha_{5ogp }(\kappa)&  \leq  &  \min_{q} \lp  - \frac{h_q(5) \log(2) }{\log(p_q(5))}   \rp  \triangleq  \bar{\alpha}_{ult_1}(\kappa;[1,5]).
    \end{eqnarray}
After numerical evaluations we obtain the results shown in Table \ref{tab:tab5}.  One now observes that the obtained bound is larger than the one for $4$-OGP which indicates that it is loose (this is also qualitatively in agreement with what was observed in \cite{Bald20} while relying on replica methods).
\begin{table}[h]
\caption{Union upper bound on  $\alpha_{5ogp}(1)$  }\vspace{.1in}
\centering
\def\arraystretch{1.2}
 \begin{tabular}{||c||c|c|c|c||}\hline\hline
 \hspace{-0in}                                               $k$     &  $2$    &  $3$      &  $4$    &  $5$   \\ \hline\hline
    $q$  & $ 0.9689$ & $ 0.9780$  & $ 0.9840$ & $ 0.9881$ \\ \hline\hline 
  $\bar{\alpha}_{ult_1}(1;[1,k])$ &   \bl{$\mathbf{1.7001}$} &   \bl{$\mathbf{1.6664}$}  &   \red{$\mathbf{1.6578}$}    &   \bl{$\mathbf{1.6593}$}    \\ \hline\hline
  \end{tabular}
\label{tab:tab5}
\end{table}

\subsubsection{$k$-OGP}
\label{sec:ult1kogp}

We now have all the ingredients to utilize the above procedure  in a recursive fashion for any integer $k>5$.

To that end we first write analogously to (\ref{eq:ult15ogpeq7}) 
 \begin{eqnarray}\label{eq:ult1kogpeq7}
& &  \hspace{-.4in}\mP(\exists \hspace{.03in}\x^{(i)}\in\cX(k;Q) \mbox{ such that } |G\x^{(i)}|\leq \kappa\1, i=1,2,\dots, k) 
 \leq 
  \lp 2^{h_q(k)} p_q(k)^{\alpha}\rp^n,
    \end{eqnarray}
where as usual $h_q(k)$ and $p_q(k)$ are the corresponding combinatorial and probabilistic factors. One can also write   the following analogue to (\ref{eq:ult15ogpeq15a5})
\begin{eqnarray}\label{eq:ult1kogpeq15a5}
 p_q(k) & = &   
    \frac{1}{\sqrt{c_{w1}}^k\det(Q)^{\frac{1}{2}} }
\frac{1}{\sqrt{2\pi}} \int_z
I_0^k
 e^{-\frac{z^2}{2}} dz,
    \end{eqnarray}    
with $I_0$ defined in (\ref{eq:ult13ogpeq15a3}) and (\ref{eq:ult13ogpeq15a4})  for $c_{w1}$ and $c_{w2}$  
\begin{eqnarray}\label{eq:ult1kogpeq15a1}
c_{w1} = \frac{1}{1-q} \quad \mbox{ and } \quad  c_{w2} = -\frac{1}{(1-q)^2}\lp \frac{1}{q} +\frac{k}{1-q} \rp^{-1} .
\end{eqnarray}

We construct matrix $A^{(k,1)}_{eq}$ of size  $(k-1+2^{k-2}-1)\times (2^{k-1}-1)$ in the following way. Let $A^{(k,3)}_{eq}$ be a $(k-1)\times (2^{k-1}-1)$ matrix  such that
\begin{eqnarray}\label{eq:ult1kogpeq15a1c0}
A^{(k,3)}_{eq} = \mbox{binrep} 
\lp \begin{bmatrix}
2^{k-2} & 2^{k-2}+1 & \dots & 2^{k-1}-2 & 2^{k-2}-1 &  2^{k-2}-2 & \dots &  1 
   \end{bmatrix} \rp,
\end{eqnarray}
where $\mbox{binrep}(\cdot)$ stands for column binary representation; it replaces each of the elements in the vector by a column vector that is its $(k-1)$-bit binary representation. For example, that means that for $k=5$, $2^{k-2}$ is replaced by $\begin{bmatrix}
                            1 & 0 & 0 & 0 
                          \end{bmatrix}^T$. Also, we set
\begin{eqnarray}\label{eq:ult1kogpeq15a1c1}
A^{(k,4)}_{eq} =  \begin{bmatrix}
fI_{(2^{k-2}-1)\times (2^{k-2}-1)} & fI_{(2^{k-2}-1)\times (2^{k-2}-1)}
    \end{bmatrix},
\end{eqnarray}
with $fI_{(2^{k-2}-1)\times (2^{k-2}-1)}$ being a flipped identity matrix of size  $(2^{k-2}-1)\times (2^{k-2}-1)$ (matrix is flipped vertically so that $1$'s are on the diagonal starting in the lower left corner and ending in the upper right corner). Then we have
\begin{eqnarray}\label{eq:ult1kogpeq15a1c1}
A^{(k,1)}_{eq} =  \begin{bmatrix}
A^{(k,3)}_{eq} & \1_{(k-1)\times 1} \\
A^{(k,4)}_{eq} & \0_{2^{k-2}-1\times 1} 
    \end{bmatrix},
\end{eqnarray}
 where $\1$ and $\0$ are vectors of all ones/zeros with dimensions specified in the subscripts.

We set
\begin{eqnarray}\label{eq:ult1kogpeq15a1c1d0}
A^{(k,2)}_{eq} =  \begin{bmatrix}
\0_{k-1\times l_{c,(k-1)}}   \\
\begin{bmatrix}
 \begin{bmatrix}
 \begin{bmatrix}
\0_{1\times (2^{k-2}-2)} & (-1)  
    \end{bmatrix} \\
 \begin{bmatrix}
CC^{(1)}  &
CC^{(2)} &   \0_{(2^{k-2}-2)\times 1)}
    \end{bmatrix}
    \end{bmatrix}  &
\begin{bmatrix}
\begin{bmatrix}
\0_{(2^(k-2)-2) \times ((\sum_{j=2}^{k-3}   (2^j-1)) -1)}    &
-1   \end{bmatrix} \\
\0_{(2^(k-2)-2) \times \sum_{j=2}^{k-3}   (2^j-1)}   \end{bmatrix}  
    \end{bmatrix}
    \end{bmatrix}.
\end{eqnarray}
where 
  \begin{eqnarray}\label{eq:ult1kogpeq15a1c1d1}
CC = fI_{(2^(k-3)-1)\times (2^(k-3)-1)}.
\end{eqnarray}
and
 \begin{equation}\label{eq:ult1kogpeq15a1c1d1}
CC_{2(i-1)+1,:}^{(1)}  =  CC_{i,:}, \quad \mbox{ and } \quad 
CC_{2(i-1)+2,:}^{(1)}  =  \0_{1\times (2^{k-3}-1)}, \quad i=1,2,\dots,2^{k-3}-1,
\end{equation}
 and  
 \begin{equation}\label{eq:ult1kogpeq15a1c1d1}
CC_{2(i-1)+2,:}^{(2)}  =  CC_{i,:}, \quad \mbox{ and } \quad 
CC_{2(i-1)+1,:}^{(2)}  = \0_{1\times (2^{k-3}-1)}, \quad i=1,2,\dots,2^{k-3}-1,
\end{equation}
We recall that $W_{i,:}$ stands for the $i$-th row of matrix $W$ and construct matrix  $A^{(k)}_{eq}$ as
\begin{eqnarray}\label{eq:ult15ogpeq1ka1b2}
A^{(k)}_{eq} =\begin{bmatrix}
   A^{(k,1)}_{eq}  &   A^{(k,2)}_{eq} \\   
   \0 & A^{(k-1)}_{eq}\end{bmatrix}.
\end{eqnarray}
After setting
 \begin{eqnarray}\label{eq:ult1kogpeq15a1b1}
\b^{(k)}_{eq} =\begin{bmatrix}
 q_{sx1} \1_{1\times (k-1) } & 0_{1\times (2^{k-2}-1)} &
\lp\b^{(k-1)}_{eq}\rp^T             
\end{bmatrix}^T,
\end{eqnarray}
we have for the combinatorial factor
\begin{eqnarray}\label{eq:ult15ogpeq15a5b0}
h_q(k) &  = & \frac{\log_2\lp  |\cX(k;Q)|  \rp}{n} ,
\end{eqnarray}
the following program
 \begin{eqnarray}\label{eq:ult1kogpeq15a1b3}
h_q(k) = 1 + \max_{\a} & & h_v(\a) 
\nonumber \\
\mbox{subject to} & &  A^{(k)}_{eq}\a =\b^{(k)}_{eq}
\nonumber \\
  & &  \sum_{i=1}^{l_{c,k}+1}\a_i =1 .
    \end{eqnarray}    
Several observations might additionally highlight the key points of the above construction. 
\begin{itemize}
\item At each recursion step $k$ ($k>5$), one repartitions into two new partitions the 
'+/-' sign interval partitions from the previous $k-1$ step.
  \item Such a $k-1\rightarrow k$ recursion move results in having the dimension of the optimizing $\a$ increased by precisely $2^{k-1}-1$. These components correspond to the lengths of the new partitions introduced at the $k$-th step and are  precisely the first   $2^{k-1}-1$ components of $\a$, $\a_{1:(2^{k-1}-1)}$.
  \item Upper left corner, $A^{(k,1)}_{eq}$, corresponds to $\a_{1:(2^{k-1}-1)}$ portion of $\a$. On the other hand, bottom right corner, $A^{(k-1)}_{eq}$, corresponds to the remaining components of $\a$ which account for all partitions from previous steps.
  \item The upper right corner, $A^{(k,2)}_{eq}$, ensures that connections between the two parts of $\a$ are properly accounted for.
      \item Construction of  $\b^{(k)}_{eq}$ follows the same principle. The first two components, 
      $ q_{sx1} \1_{1\times (k-1) }$ and  $0_{1\times (2^{k-2}-1)}$ relate to partitions introduced at the $k$-th step. The bottom portion,  $\b^{(k-1)}_{eq} $, relates to the partitions present in the previous step.
\end{itemize}

As were the programs in (\ref{eq:ult14ogpeq15a1b3}) and (\ref{eq:ult15ogpeq15a1b3}), the above program (\ref{eq:ult1kogpeq15a1b3}) is also convex. A combination of  (\ref{eq:ult13ogpeq15a3}), (\ref{eq:ult13ogpeq15a4}) , (\ref{eq:ult1kogpeq15a5}), (\ref{eq:ult1kogpeq15a1}), and  (\ref{eq:ult1kogpeq15a1b3})  suffices to utilize (\ref{eq:ult1kogpeq7}) and obtain the following upper bound
\begin{eqnarray}\label{eq:ult1kogpeq16}
 \alpha_{kogp }(\kappa) &  \leq  & \min_{q} \lp  - \frac{h_q(k) \log(2) }{\log(p_q(k))} \rp  \triangleq  \bar{\alpha}_{ult_1}(\kappa;[1,k]).
    \end{eqnarray}
To check if the reversed trend observed for $k=5$ (where the bound gets looser as $k$ increases) continues, we conducted numerical evaluations for $k=6$ and $k=7$. For concreteness, we again consider canonical  $\kappa=1$ scenario. The obtained results are shown in Table \ref{tab:tab6}. They are also complemented with $k\leq 5$ results from previous sections and visualized in Figure \ref{fig:fig1}.  As can be seen, the trend indeed continues and the bounds are getting looser  as $k$ further increases beyond $k=4$. Utilizing replica methods, \cite{Bald20} arrived at the same conclusion. However, \cite{Bald20} also showed that in the regimes where the bounds decrease as $k$ increases (which is how the true values, $\alpha_{kogp }(\kappa)$ ,  that we are trying to bound behave), they also tend to be tight. If such a tendency is indeed true then $\lim_{k\rightarrow\infty}\alpha_{kogp}$ is smaller than what we have obtained, but possibly not that much smaller. In other words, this suggests that the tightest $\kappa=1$ bound that the above union-bounding methodology produces, $\bar{\alpha}_{ult_1}(1;4)\approx 1.6578$, might not be that far away from the true $\alpha_{ult_1}(1)$.

\begin{table}[h]
\caption{Union upper bound on  $\alpha_{kogp}(1)$, $k\leq 7$.  }\vspace{.1in}
\centering
\def\arraystretch{1.2}
 \begin{tabular}{||c||c|c|c|c|c|c||}\hline\hline
 \hspace{-0in}                                               $k$     &  $2$    &  $3$      &  $4$    &  $5$  &  $6$    &  $7$   \\ \hline\hline
    $q$  & $ 0.9689$ & $ 0.9780$  & $ 0.9840$ & $ 0.9881$ & $ 0.9909$ & $ 0.9930$ \\ \hline\hline 
  $\bar{\alpha}_{ult_1}(1;[1,k])$ &   \bl{$\mathbf{1.7001}$} &   \bl{$\mathbf{1.6664}$}  &   \red{$\mathbf{1.6578}$}    &   \bl{$\mathbf{1.6593}$}    &   \bl{$\mathbf{1.6650}$}    &   \bl{$\mathbf{1.6724}$}    \\ \hline\hline
  \end{tabular}
\label{tab:tab6}
\end{table}

 \begin{figure}[h]
\centering
\centerline{\includegraphics[width=.7\linewidth]{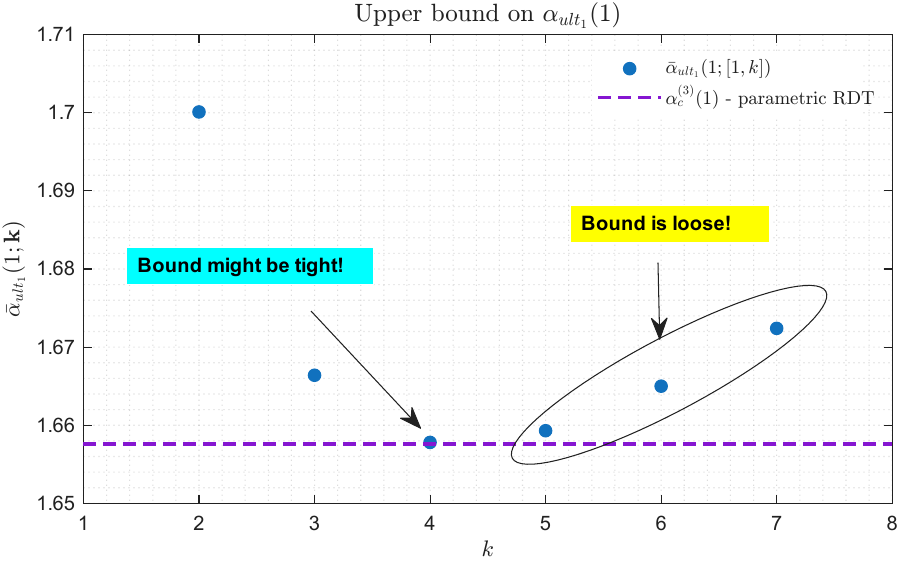}}
\caption{Upper bounds $\bar{\alpha}_{ult_1}(1;\k)$ on $\alpha_{kogp}(1)$ and $\alpha_{ult_1}(1) = \lim_{k\rightarrow \infty} \alpha_{kogp}(1)$.}
\label{fig:fig1}
\end{figure}

\subsection{Second level of ultrametricity -- $ult_2$-OGP}
\label{sec:ult2}

On the second level of ultametricity $s=2$. From (\ref{eq:ultogp6}) and (\ref{eq:ultogp7}), one then has $\q=[q_0,q_1,q_2,q_3]$ and $\k=[k_0,k_1,k_2]$ with
  \begin{equation}\label{eq:ult2ogp6}
1  =q_0 > q_1> q_2 > q_{3}=0, \quad \mbox{and } \quad 1 =k_0 <  k_1 < k_2 = k, \mbox{ where } \frac{k_{2}}{k_1},\frac{k_{1}}{k_0}\in\mZ.
\end{equation}
Recalling on (\ref{eq:ultogp5}), we observe that for a $Q\in\cQ_{ult_2}(\k)$ the key object of interest is
\begin{equation}\label{eq:ult2ogp5}
    \lim_{n\rightarrow \infty }\mP_G(\exists X\in\cX(\k;Q)\mbox{ such that } |GX_{:,i}|\leq \kappa \1, i=1,2,\dots,k).   
\end{equation}

\begin{eqnarray}\label{eq:ult2eq7}
& &  \hspace{-.4in}\mP_G(\exists X\in\cX(\k;Q)\mbox{ such that } |GX_{:,i}|\leq \kappa \1, i=1,2,\dots,k) 
\nonumber \\
& \leq &
 \sum_{X\in\cX(\k;Q)}  
 \mP\lp \forall i, |GX_{:,i}|\leq \kappa\1\mbox{ and } X^TX=Q  \rp \nonumber \\
& = &
|\cX(\k;Q)| 
 \mP\lp \forall i, |GX_{:,i}|\leq \kappa\1\mbox{ and } X^TX=Q  \rp  \nonumber \\
&  =  & \lp2^{h_q(\k)} p_q(\k)^{\alpha}\rp^n,
    \end{eqnarray}
    where, as usual,  $h_q(\k)$ and  $p_q(\k)$ are the combinatorial and probabilistic factors given as
\begin{eqnarray}\label{eq:ult2eq8}
h_{\q}^{(2)}(\k) &  = & \frac{\log_2\lp  |\cX(\k;Q)|  \rp}{n} \nonumber \\
p_{\q}^{(2)}(\k) & = &   \lp   \mP\lp \forall i, |GX_{:,i}|\leq \kappa\1\mbox{ and } X^TX=Q  \rp  \rp^{\frac{1}{m}} =      \mP\lp \forall i, |G_{1,:}X_{:,i}|\leq \kappa \mbox{ and } X^TX=Q  \rp ,
    \end{eqnarray}    
and the last equality follows after observing that the elements of $G$ are independent standard normals.
Following the methodology of previous sections, we then write
\begin{eqnarray}\label{eq:ult23ogpeq15}
 p_{\q}^{(2)}(\k) & = &   \frac{1}{\lp 2\pi\rp^{\frac{k}{2}} \det(Q)^{\frac{1}{2}} }\int_{|\g|\leq \kappa} e^{-\frac{1}{2}\g^TQ^{-1}\g}d\g,
    \end{eqnarray}    
where 
\begin{eqnarray}\label{eq:ult23ogpeq14}
\g  =  \begin{bmatrix}
                  g_1 \\ g_2 \\\vdots \\g_k        \end{bmatrix} 
\quad                  \mbox{ and } \quad
              Q  =  (1-q_1)I +(q_1-q_2)\cI^{(1)}   + q_2 \cI^{(2)}  ,
               \end{eqnarray}
               and
\begin{eqnarray}\label{eq:ult23ogpeq14a0}               
               \cI^{(1)} =  I_{\frac{k}{k_1}\times\frac{k}{k_1}} \1_{k_1}\1_{k_1}^T   \quad                  \mbox{ and } \quad              
              \cI^{(2)} =\1_k\1_k^T .
\end{eqnarray}
We first observe 
\begin{eqnarray}\label{eq:ult23ogpeq15a0}
\cZ =\lp (1-q_1)I +(q_1-q_2)\cI^{(1)}  \rp^{-1} 
= c_{w1} I + c_{w2}\cI^{(1)},
\end{eqnarray}
where
\begin{equation}\label{eq:ult23ogpeq15a1}
c_{w1} = \frac{1}{1-q_1} \quad \mbox{ and } \quad  c_{w2} = -\frac{1}{(1-q_1)^2} \lp \frac{1}{q_1-q_2} + \frac{k_1}{1-q_1}  \rp^{-1}  = -(c_{w1})^2\lp \frac{1}{q_1-q_2} + k_1 c_{w1} \rp^{-1} .
\end{equation}
One also has 
\begin{eqnarray}\label{eq:ult23ogpeq15a0b0}
Q^{-1}  &  = & \lp \cZ^{-1}  + q_2\cI^{(2)}  \rp^{-1}  \nonumber \\
& = & \cZ -\cZ\1_k \lp \frac{1}{q_2} + \1_k^T\cZ\1_k   \rp^{-1} \1_k^T\cZ^T
\nonumber \\
& = &  c_{w1}I + c_{w2}\cI^{(1)} - (c_{w1}+k_1c_{w2})^2\lp \frac{1}{q_2} +k(c_{w1}+k_1c_{w2})\rp^{-1} \cI^{(2)}
\nonumber \\
& = &  c_{w1}I + c_{w2}\cI^{(1)} +  c_{w3} \cI^{(2)}
,
\end{eqnarray}
with
\begin{eqnarray}\label{eq:ult23ogpeq15a1b0}
c_{w3} = -(c_{w1}+k_1 c_{w2})^2\lp \frac{1}{q_2} +k( c_{w1} + k_1 c_{w2} ) \rp^{-1}.
\end{eqnarray}
After partitioning $\g$ 
\begin{eqnarray}\label{eq:ult23ogpeq15a1b1}
\g  =  \begin{bmatrix}
                  \bar{\g}_1 \\ \bar{\g}_2 \\\vdots \\ \bar{\g}_{\frac{k}{k_1}}        \end{bmatrix} ,
                  \mbox{ where }
\bar{\g}_i  =  \begin{bmatrix}
                   g_{(i-1)k_1+1} \\ g_{(i-1)k_1+2} \\\vdots \\ g_{(i-1)k_1+k_1}        \end{bmatrix} ,  
\end{eqnarray}
we further transform (\ref{eq:ult23ogpeq15}) 
\begin{eqnarray}\label{eq:ult23ogpeq15a2}
 p^{(2)}_{\q}(\k) & = &   \frac{1}{\lp 2\pi\rp^{\frac{k}{2}} \det(Q)^{\frac{1}{2}} }\int_{|\g|\leq \kappa} e^{-\frac{1}{2}\g^T\lp    c_{w1} I + c_{w2}\cI^{(1)}  + c_{w3}\cI^{(2)}  \rp \g}d\g
 \nonumber \\
 & = &   \frac{1}{\lp 2\pi\rp^{\frac{k}{2}} \det(Q)^{\frac{1}{2}} }\int_{|\g|\leq \kappa} e^{-\frac{c_{w1}}{2}\g^T\g  -\frac{ c_{w2}}{2}\g^T \cI^{(1)} \g -\frac{ c_{w3}}{2}\g^T \cI^{(2)} \g  }d\g
 \nonumber \\
 & = &   \frac{1}{\lp 2\pi\rp^{\frac{k}{2}} \det(Q)^{\frac{1}{2}} }
 \nonumber \\
& & \times
 \int_{|\g|\leq \kappa} e^{-\frac{c_{w1}}{2}\g^T\g  }  
 \prod_{i=1}^{\frac{k}{k_1}}
 \frac{1}{\sqrt{2\pi}}\int_{\bar{z}_i} e^{\sqrt{-c_{w2}}\1^T \bar{\g}_i \bar{z}_i -\frac{\bar{z}_i^2}{2}} d\bar{z}_i 
 \frac{1}{\sqrt{2\pi}}\int_z e^{\sqrt{-c_{w3}}\1^T \g z -\frac{z^2}{2}} dz d\g
 \nonumber \\
 & = &   \frac{1}{\det(Q)^{\frac{1}{2}} } \frac{1}{\sqrt{2\pi}^{1+\frac{k}{k_1}}} 
\nonumber \\
& & \times
\int_z
 \prod_{i=1}^{\frac{k}{k_1}}
\int_{\bar{z}_i}
\lp
\frac{1}{\lp 2\pi\rp^{\frac{k_1}{2}} }
 \int_{|\bar{\g}_i|\leq \kappa} e^{-\frac{c_{w1}}{2}\bar{\g}_i^T\bar{\g}_i  +\sqrt{-c_{w2}}\1^T \bar{\g}_i \bar{z}_i   +\sqrt{-c_{w3}}\1^T \bar{\g}_i z  } d\bar{\g}_i 
 \rp
 e^{-\frac{\bar{z}_i^2}{2}} d\bar{z}_i
 e^{-\frac{z^2}{2}} dz
\nonumber \\
 & = &   \frac{1}{\det(Q)^{\frac{1}{2}} } \frac{1}{\sqrt{2\pi}^{1+\frac{k}{k_1}}} 
\nonumber \\
& & \times
 \int_z
 \prod_{i=1}^{\frac{k}{k_1}}
\int_{\bar{z}_i}
\lp
\frac{1}{\lp 2\pi\rp^{\frac{1}{2}} }
 \int_{|g_1|\leq \kappa} e^{-\frac{c_{w1}}{2}g_1^2  +\sqrt{-c_{w2}}  g_1 \bar{z}_i   +\sqrt{-c_{w3}} g_1 z  } dg_1 
 \rp^{k_1}
 e^{-\frac{\bar{z}_i^2}{2}} d\bar{z}_i
 e^{-\frac{z^2}{2}} dz
 \nonumber \\
 & = &   \frac{1}{\det(Q)^{\frac{1}{2}} } \frac{1}{\sqrt{2\pi}^{1+\frac{k}{k_1}}} 
   \int_z
 \prod_{i=1}^{\frac{k}{k_1}}
\int_{\bar{z}_i}
\lp
I_0^{(2)}(\bar{z}_i,z) \rp^{k_1}
 e^{-\frac{\bar{z}_i^2}{2}} d\bar{z}_i
 e^{-\frac{z^2}{2}} dz
 \nonumber \\
 & = &   \frac{1}{\det(Q)^{\frac{1}{2}} } \frac{1}{\sqrt{2\pi}} 
   \int_z
 \lp
 \int_{\bar{z}_1}
 \frac{1}{\sqrt{2\pi} }
\lp
I_0^{(2)}(\bar{z}_1,z) \rp^{k_1}
 e^{-\frac{\bar{z}_1^2}{2}} d\bar{z}_1
 \rp^{\frac{k}{k_1}}
 e^{-\frac{z^2}{2}} dz
 \nonumber \\
 & = &   \frac{1}{\det(Q)^{\frac{1}{2}} } \frac{1}{\sqrt{2\pi}} 
   \int_z
\lp I_1^{(2)}(z) \rp^{\frac{k}{k_1}} e^{-\frac{z^2}{2}} dz,
    \end{eqnarray}    
    where
\begin{eqnarray}\label{eq:ult23ogpeq15a2b0}
  I_0^{(2)}(\bar{z}_1,z) 
=
\frac{1}{\lp 2\pi\rp^{\frac{1}{2}} }
 \int_{|g_1|\leq \kappa} e^{-\frac{c_{w1}}{2}g_1^2  +\sqrt{-c_{w2}}  g_1 \bar{z}_1   +\sqrt{-c_{w3}} g_1 z  } dg_1 
     \end{eqnarray}    
and
\begin{eqnarray}\label{eq:ult23ogpeq15a2b1}
  I_1^{(2)}(z) 
=
 \frac{1}{\sqrt{2\pi} }
 \int_{\bar{z}_1}
\lp
I_0^{(2)}(\bar{z}_1,z) \rp^{k_1}
 e^{-\frac{\bar{z}_1^2}{2}} d\bar{z}_1.
      \end{eqnarray}    
After setting
\begin{eqnarray}
\label{eq:ult23ogpeq15a3}
D^{(2)} = \frac{\sqrt{-c_{w2}} }  {\sqrt{c_{w1}} } \bar{z}_1
+
 \frac{\sqrt{-c_{w3}} }  {\sqrt{c_{w1}} } z 
  \quad \mbox{ and } \quad  E^{(2)}  = \sqrt{c_{w1}}\kappa,
\end{eqnarray}
one finds
\begin{eqnarray}\label{eq:ult23ogpeq15a4}
 \tilde{I}_0^{(2)}(\bar{z}_1,z) =
 -\frac{{\mathrm{e}}^{\frac{\lp D^{(2)}\rp^2}{2}}\,\left(\mathrm{erf}\left(\frac{\sqrt{2}\,D^{(2)}}{2}
 -\frac{\sqrt{2}\,E^{(2)}}{2}\right)-\mathrm{erf}\left(\frac{\sqrt{2}\,D^{(2)}}{2}
 +\frac{\sqrt{2}\,E^{(2)}}{2}\right)\right)}{2},
    \end{eqnarray}    
\begin{eqnarray}\label{eq:ult23ogpeq15a4b0}
  \tilde{I}_1^{(2)}(z) 
=
 \frac{1}{\sqrt{2\pi} }
 \int_{\bar{z}_1}
\lp
\tilde{I}_0^{(2)}(\bar{z}_1,z) \rp^{k_1}
 e^{-\frac{\bar{z}_1^2}{2}} d\bar{z}_1,
      \end{eqnarray}    
and
\begin{eqnarray}\label{eq:ult23ogpeq15a5}
 p^{(2)}_{\q}(\k) & = &   
  \frac{1}{\sqrt{c_{w1}}^k\det(Q)^{\frac{1}{2}} } \frac{1}{\sqrt{2\pi}} 
   \int_z
\lp \tilde{I}_1^{(2)}(z) \rp^{\frac{k}{k_1}} e^{-\frac{z^2}{2}} dz.
    \end{eqnarray} 
    
For $h^{(2)}_{\q}(\k)$ we can still utilize the program from (\ref{eq:ult1kogpeq15a1b3})
 \begin{eqnarray}\label{eq:ult2kogpeq15a1b3}
h^{(2)}_{\q}(\k) = 1 + \max_{\a} & & h_v(\a) 
\nonumber \\
\mbox{subject to} & &  A^{(\k)}_{eq}\a =\b^{(\k)}_{eq}
\nonumber \\
  & &  \sum_{i=1}^{l_{c,k}+1}\a_i =1 .
    \end{eqnarray}
A combination of  (\ref{eq:ult23ogpeq15a1}),  (\ref{eq:ult23ogpeq15a1b0}), and (\ref{eq:ult23ogpeq15a3})--(\ref{eq:ult2kogpeq15a1b3})  provides all ingredients needed to utilize (\ref{eq:ult2eq7}) and obtain the following upper bound
\begin{eqnarray}\label{eq:ult23ogpeq16}
 \alpha_{ult_2}(\kappa) &  \leq  & \min_{\q} \lp  - \frac{h^{(2)}_{\q}(\k) \log(2) }{\log\lp p^{(2)}_{\q}(\k) \rp } \rp  \triangleq  \bar{\alpha}_{ult_2}(\kappa;\k).
    \end{eqnarray}
        
To be able to practically work with the above machinery, one again needs to determine $A^{(\k)}_{eq}$ and $\b^{(\k)}_{eq}$ in program (\ref{eq:ult2kogpeq15a1b3}).  However, the most complicated part, matrix $A^{(\k)}_{eq}$, actually remains identical as for the first level of uiltrametricity, i.e., one again has $A^{(\k)}_{eq}=A^{(k)}_{eq}$. In other words, this matrix will depend only on $k$. The only thing that needs to be adjusted is $\b^{(\k)}_{eq}$. However, this is a vector and much easier to handle than $A^{(\k)}_{eq}$. One just needs to be a bit careful as the adjustments depend on $\k$. We below proceed by considering particular choices for $\k$ and along the way discuss needed $\b^{(\k)}_{eq}$ adjustments.

\subsubsection{$ult_2$-OGP -- $\k=[1,2,k]$}
\label{sec:ult2k21}

Since on the second level of ultrametricity  we have both $q_1$ and $q_2$, we need to slightly adjust (\ref{eq:ult13ogpeq9}). To that end we write
\begin{eqnarray}\label{eq:ult23ogpeq9}
q_{sx1} & = & (1+q_1)/2 \nonumber \\
q_{sx2} & = & (1+q_2)/2 \nonumber \\
q_{s1} & = &  q_{sx1}
 \nonumber \\ 
    q_{s2} &  = &  q_{sx2} - (1-q_{s1})/2 .
\end{eqnarray}

We start considerations of various $\k$ sequences with $k_1=2$ (which is the first nontrivial scenario) and run $k$ over the first few minimal allowed values.

\subsubsubsection{$ult_2$-OGP -- $\k=[1,2,4]$}
\label{sec:ult2421}
 
Keeping the adjustment from (\ref{eq:ult23ogpeq9}) in mind, we immediately write analogously to (\ref{eq:ult14ogpeq15a1b1})
\begin{eqnarray}\label{eq:ult24ogpeq15a1b1}
\b^{([1,2,4])}_{eq} =\begin{bmatrix}
           q_{sx2} & q_{sx2} & q_{sx1} & q_{s1}-q_{s2} & 1/2(1-q_{s1})  & 1/2(1-q_{s1}) 
\end{bmatrix}^T.
\end{eqnarray}
To see where the above $\b^{([1,2,4])}_{eq}$ adjustment comes from, one observes that the first level of ultrametricity scenario with $k=4$ is $4$-OGP. For $4$-OGP one has that the fourth vector has identical $q_{sx1}n$  sign overlap with the remaining three vectors. Here, this needs to be adjusted to reflect that the overlap with the first two is $q_{sx2}n$ and with the third is $q_{sx1}n$. The choice given in (\ref{eq:ult24ogpeq15a1b1}) follows after one observes that the first three elements of $\b^{([1,2,4])}_{eq}$ are precisely these overlaps scaled by $n$. Everything else remains unchanged as it reflects partitioning intervals summing-up structure which is unaffected by the change of overlaps values. One then has that the program (\ref{eq:ult2kogpeq15a1b3}) is basically (\ref{eq:ult14ogpeq15a1b3}) with $\b^{(4)}_{eq}=\b^{([1,2,4])}_{eq}$. Results obtained through  numerical evaluations are given in Table \ref{tab:tab7}. One observes that this value is lower that any of the values from Table \ref{tab:tab6}, which might suggest that the union-bounding strategy might be applicable beyond the first level of ultrametricity. However, one would need much more to comfortably predict overall success.
\begin{table}[h]
\caption{Union upper bound $\bar{\alpha}_{ult_2}(1;[1,2,4])$ on $\alpha_{ult_2}(1)$  }\vspace{.1in}
\centering
\def\arraystretch{1.2}
 \begin{tabular}{||c||c||}\hline\hline
 \hspace{-0in}                                               $\k$     &  $[1,2,4]$     \\ \hline\hline
    $[q_1,q_2]$  & $[0.9932,0.9640 ]$    \\ \hline\hline 
   $\bar{\alpha}_{ult_2}(1;\k)$ &   \bl{$\mathbf{1.6444}$}     \\ \hline\hline
  \end{tabular}
\label{tab:tab7}
\end{table}

\subsubsubsection{$ult_2$-OGP -- $\k=[1,2,6]$}
\label{sec:ult2621}
 
In this scenario the adjustment from  (\ref{eq:ult23ogpeq9})  remains in place. Moreover, $A_{eq}^{([1,2,6])}= A_{eq}^{(6)}$. We also need to adequately adjust $\b^{([1,2,6])}_{eq}$ as it will not be exactly equal to $\b^{(6)}_{eq}$. First, from (\ref{eq:ult1kogpeq15a1b1}),  we recall that $\b^{(6)}_{eq}$ is built recursively from $\b^{(5)}_{eq}$ in the following way
 \begin{eqnarray}\label{eq:ult2kogpeq15a1b1}
\b^{(6)}_{eq} =\begin{bmatrix}
 q_{sx1} \1_{1\times (6-1) } & 0_{1\times (2^{6-2}-1)} &
\lp\b^{(5)}_{eq}\rp^T             
\end{bmatrix}^T,
\end{eqnarray}
with $\b^{(5)}_{eq}$ as given in  (\ref{eq:ult15ogpeq15a1b1})
{\small\begin{eqnarray}\label{eq:ult25ogpeq15a1b2}
\b^{(5)}_{eq} =\begin{bmatrix}
 q_{sx1} & q_{sx1} & q_{sx1} &  q_{sx1} &  q_{s2} & 0 & 0 & 0 & 0 & 0 & 0 &
           q_{sx1} & q_{sx1} & q_{sx1} & q_{s1}-q_{s2} & \frac{1-q_{s1}}{2}  &  \frac{1-q_{s1}}{2}  
\end{bmatrix}^T.
\end{eqnarray}}
In the $\k=[1,2,6]$ scenario that we consider now, we first need to adapt $\b^{(5)}_{eq} $ so that it becomes

{\small\begin{equation}\label{eq:ult25ogpeq15a1b3}
\b^{([1,2,5])}_{eq} =\begin{bmatrix}
 q_{sx2} & q_{sx2} & q_{sx2} &  q_{sx2} &  q_{s2} & 0 & 0 & 0 & 0 & 0 & 0 &
           q_{sx2} & q_{sx2} & q_{sx1} & q_{s1}-q_{s2} & \frac{1-q_{s1}}{2}  &  \frac{1-q_{s1}}{2}  
\end{bmatrix}^T.
\end{equation}}

\noindent The reason is the following. The first four components reflect the sign overlaps of the fifth vector with the first four. On the other hand, components 12-14  reflect the sign overlaps of the fourth vector with the first three. In the  $\k=[1,2,6]$ scenario of interest here, these overlaps are precisely as stated in (\ref{eq:ult25ogpeq15a1b3}). One can then use (\ref{eq:ult25ogpeq15a1b3}) to adjust (\ref{eq:ult2kogpeq15a1b1})
\begin{eqnarray}\label{eq:ult25ogpeq15a1b4}
\b^{([1,2,6])}_{eq} =\begin{bmatrix}
 q_{sx2} &  q_{sx2} &  q_{sx2} &  q_{sx2} &  &  q_{sx1} & 0_{1\times (2^{6-2}-1)} &
\lp\b^{([1,2,5])}_{eq}\rp^T             
\end{bmatrix}^T,
\end{eqnarray}
The first five components of $\b^{([1,2,6])}_{eq}$ are the sign overlaps of the sixth vector with the first five. In $\k=[1,2,6]$ scenario they are as stated in (\ref{eq:ult25ogpeq15a1b4}). Everything else remains unchanged. One then utilizes the program (\ref{eq:ult2kogpeq15a1b3}) and after numerical evaluations obtains results shown in Table \ref{tab:tab8}. One observes that the bound is smaller than the one in $\k=[1,2,4]$ scenario which is again a favorable trend.
\begin{table}[h]
\caption{Union upper bound $\bar{\alpha}_{ult_2}(1;[1,2,6])$ on $\alpha_{ult_2}(1)$  }\vspace{.1in}
\centering
\def\arraystretch{1.2}
 \begin{tabular}{||c||c|c||}\hline\hline
 \hspace{-0in}                                               $\k$     &  $[1,2,4]$  &  $[1,2,6]$     \\ \hline\hline
    $[q_1,q_2]$  & $[0.9940,0.9550 ]$     & $[0.9976,0.9760 ]$    \\ \hline\hline 
   $\bar{\alpha}_{ult_2}(1;\k)$ &   \bl{$\mathbf{1.6444}$}   &   \bl{$\mathbf{1.6327}$}     \\ \hline\hline
  \end{tabular}
\label{tab:tab8}
\end{table}

\subsubsubsection{$ult_2$-OGP -- $\k=[1,2,8]$}
\label{sec:ult2821}
 
One again has  $A_{eq}^{([1,2,8])}= A_{eq}^{(8)}$ and an adequate adjustment is needed for $\b^{([1,2,8])}_{eq}$ as it will not be exactly equal to $\b^{(8)}_{eq}$. From (\ref{eq:ult1kogpeq15a1b1}),  we first recall that $\b^{(8)}_{eq}$ is built as
 \begin{eqnarray}\label{eq:ult2kogpeq15a1b1c0}
\b^{(8)}_{eq} & = & \begin{bmatrix}
 q_{sx1} \1_{1\times (8-1) } & 0_{1\times (2^{8-2}-1)} &
\lp\b^{(8-1)}_{eq}\rp^T             
\end{bmatrix}^T
\nonumber \\
\b^{(7)}_{eq} & = & \begin{bmatrix}
 q_{sx1} \1_{1\times (7-1) } & 0_{1\times (2^{7-2}-1)} &
\lp\b^{(7-1)}_{eq}\rp^T             
\end{bmatrix}^T.
\end{eqnarray}
Utilizing the reasoning shown above, we here have the following analogue
 \begin{eqnarray}\label{eq:ult2kogpeq15a1b1c1}
\b^{([1,2,8])}_{eq} & = & \begin{bmatrix}
 q_{sx2} \1_{1\times (7-1) } & q_{sx1} & 0_{1\times (2^{8-2}-1)} &
\lp\b^{(8-1)}_{eq}\rp^T             
\end{bmatrix}^T
\nonumber \\
\b^{([1,2,7])}_{eq} & = & \begin{bmatrix}
 q_{sx2} \1_{1\times (7-1) } & 0_{1\times (2^{7-2}-1)} &
\lp\b^{([1,2,6])}_{eq}\rp^T             
\end{bmatrix}^T,
\end{eqnarray}
where $\b^{([1,2,6])}_{eq}$ is determined in (\ref{eq:ult24ogpeq15a1b1}). The first six components of $\b^{([1,2,7])}_{eq}$ are the sign overlaps of the seventh vector with the first six. In $[1,2,8]$ scenario they all are $q_{sx2}$. The first seven components of $\b^{([1,2,8])}_{eq}$ are the sign overlaps of the eighth vector with the first seven. The first six of them are $q_{sx2}$ and the seventh is $q_{sx1}$. 

After utilizing the program (\ref{eq:ult2kogpeq15a1b3}) one obtains through  numerical evaluations  results given in Table \ref{tab:tab9}. This time, however, the bounding trend is reversed and the bound is larger than the one obtained for $\k=[1,2,6]$. This means that the bound is loose. Memory requirements are substantial and prevent to smoothly continue further with larger $k$. However, the indication from the first level of ultrametricity is that once the bound gets loose it remains looser and looser as one increases $k$. If such a tendency is also in place on the second level of ultrametricity, then increasing $k$ further is not even needed.

\begin{table}[h]
\caption{Union upper bound $\bar{\alpha}_{ult_2}(1;[1,2,8])$ on $\alpha_{ult_2}(1)$  }\vspace{.1in}
\centering
\def\arraystretch{1.2}
 \begin{tabular}{||c||c|c|c||}\hline\hline
 \hspace{-0in}                                               $\k$     &  $[1,2,4]$  &  $[1,2,6]$    &  $[1,2,8]$     \\ \hline\hline
    $[q_1,q_2]$  & $[0.9940,0.9550 ]$     & $[0.9976,0.9760 ]$   & $[0.9989,0.9830 ]$    \\ \hline\hline 
   $\bar{\alpha}_{ult_2}(1;\k)$ &   \bl{$\mathbf{1.6444}$}   &   \red{$\mathbf{1.6327}$}  &   \bl{$\mathbf{1.6345}$}     \\ \hline\hline
  \end{tabular}
\label{tab:tab9}
\end{table}

\subsubsection{$ult_2$-OGP -- $\k=[1,3,k]$}
\label{sec:ult2k31}

In the above second level of ultrametricity considerations, we started with the simplest choice $k_1=2$.
We remain on the second lvel of ultrametricity and now consider the first next choice, $k_1=3$.
Adjustment (\ref{eq:ult23ogpeq9}) is again in place as there are two $q$'s, $q_1$ and $q_2$.  We follow the trend from previous sections and separately consider a few smallest allowable values for $k$ when 
$k_1=3$.

\subsubsubsection{$ult_2$-OGP -- $\k=[1,3,6]$}
\label{sec:ult2631}
 
Since $A_{eq}^{([1,3,6])}= A_{eq}^{(6)}$ remains in place, we need to adjust $\b^{([1,2,6])}_{eq}$  from (\ref{eq:ult25ogpeq15a1b4}). Following the strategy discussed right after (\ref{eq:ult25ogpeq15a1b3}) and (\ref{eq:ult25ogpeq15a1b4}) and making sure that sign overlaps properly reflect $\k=[1,3,6]$ ultametric structure, we find

{\small\begin{equation}\label{eq:ult25ogpeq15a1b1c0}
\b^{([1,3,5])}_{eq} =\begin{bmatrix}
 q_{sx2} & q_{sx2} & q_{sx2} &  q_{sx1} &  q_{s2} & 0 & 0 & 0 & 0 & 0 & 0 &
           q_{sx2} & q_{sx2} & q_{sx2} & q_{s1}-q_{s2} & \frac{1-q_{s1}}{2}  &  \frac{1-q_{s1}}{2}  
\end{bmatrix}^T.
\end{equation}}
and
\begin{eqnarray}\label{eq:ult24ogpeq15a1b2c1}
\b^{([1,3,6])}_{eq} =\begin{bmatrix}
 q_{sx2} &  q_{sx2} &  q_{sx2} &  q_{sx1} &  &  q_{sx1} & 0_{1\times (2^{6-2}-1)} &
\lp\b^{([1,3,5])}_{eq}\rp^T             
\end{bmatrix}^T,
\end{eqnarray}
Utilizing program (\ref{eq:ult2kogpeq15a1b3}), numerical evaluations give results shown in Table \ref{tab:tab10}. One observes that the bound is smaller than any of the $\k=[1,2,4]$ and $\k=[1,2,6]$ bounds which is again a very favorable trend.
\begin{table}[h]
\caption{Union upper bound $\bar{\alpha}_{ult_2}(1;[1,3,6])$ on $\alpha_{ult_2}(1)$  }\vspace{.1in}
\centering
\def\arraystretch{1.2}
 \begin{tabular}{||c||c|c||}\hline\hline
 \hspace{-0in}                                               $\k$     &  $[1,3,6]$     \\ \hline\hline
    $[q_1,q_2]$  & $[0.9954,0.9598 ]$      \\ \hline\hline 
   $\bar{\alpha}_{ult_2}(1;\k)$ &   \bl{$\mathbf{1.6300}$}      \\ \hline\hline
  \end{tabular}
\label{tab:tab10}
\end{table}

\subsubsubsection{$ult_2$-OGP -- $\k=[1,3,9]$}
\label{sec:ult2931}
 
First we note $A_{eq}^{([1,3,9])}= A_{eq}^{(9)}$, where $A_{eq}^{(9)}$ is obtained via recursion in (\ref{eq:ult15ogpeq1ka1b2}). One then needs to adjust $\b^{([1,3,9])}_{eq}$.
Following the strategy that we have repeated many times above, we account for sign overlaps given the $\k=[1,3,9]$ ultrametric structure and find the following analogue to (\ref{eq:ult2kogpeq15a1b1c1})

 \begin{eqnarray}\label{eq:ult2kogpeq15a1b1c1d0}
\b^{([1,3,9])}_{eq} & = & \begin{bmatrix}
 q_{sx2} \1_{1\times (7-1) } & q_{sx1} & q_{sx1}  & 0_{1\times (2^{9-2}-1)} &
\lp\b^{([1,3,8])}_{eq}\rp^T             
\end{bmatrix}^T
\nonumber \\
\b^{([1,3,8])}_{eq} & = & \begin{bmatrix}
 q_{sx2} \1_{1\times (7-1) } & q_{sx1} & 0_{1\times (2^{8-2}-1)} &
\lp\b^{([1,3,7])}_{eq}\rp^T             
\end{bmatrix}^T\nonumber \\
\b^{([1,3,7])}_{eq} & = & \begin{bmatrix}
 q_{sx2} \1_{1\times (7-1) } & 0_{1\times (2^{7-2}-1)} &
\lp\b^{([1,3,6])}_{eq}\rp^T             
\end{bmatrix}^T,
\end{eqnarray}
where $\b^{([1,3,6])}_{eq}$ is determined in (\ref{eq:ult24ogpeq15a1b2c1}).
 Numerical evaluations via program (\ref{eq:ult2kogpeq15a1b3}) give results shown in Table \ref{tab:tab11}. One observes that the bound is smaller than  $\k=[1,3,6]$ bound.
\begin{table}[h]
\caption{Union upper bound $\bar{\alpha}_{ult_2}(1;[1,3,9])$ on $\alpha_{ult_2}(1)$  }\vspace{.1in}
\centering
\def\arraystretch{1.2}
 \begin{tabular}{||c||c|c||}\hline\hline
 \hspace{-0in}                                               $\k$     &  $[1,3,6]$  &  $[1,3,9]$     \\ \hline\hline
    $[q_1,q_2]$  & $[0.9954,0.9598 ]$  & $[0.9984,0.9746 ]$      \\ \hline\hline 
   $\bar{\alpha}_{ult_2}(1;\k)$ &   \bl{$\mathbf{1.6300}$} &   \bl{$\mathbf{1.6236}$}      \\ \hline\hline
  \end{tabular}
\label{tab:tab11}
\end{table}

\subsubsubsection{$ult_2$-OGP -- $\k=[1,3,12]$}
\label{sec:ult21231}
 
We have $A_{eq}^{([1,3,12])}= A_{eq}^{(12)}$ with $A_{eq}^{(12)}$ as in (\ref{eq:ult15ogpeq1ka1b2}). The needed  adjustment  for $\b^{([1,3,12])}_{eq}$ follows by writing analogously to (\ref{eq:ult2kogpeq15a1b1c1d0})

 \begin{eqnarray}\label{eq:ult2kogpeq15a1b1c1d0e0}
\b^{([1,3,12])}_{eq} & = & \begin{bmatrix}
 q_{sx2} \1_{1\times (10-1) } & q_{sx1} & q_{sx1}  & 0_{1\times (2^{12-2}-1)} &
\lp\b^{([1,3,12])}_{eq}\rp^T             
\end{bmatrix}^T
\nonumber \\
\b^{([1,3,11])}_{eq} & = & \begin{bmatrix}
 q_{sx2} \1_{1\times (9-1) } & q_{sx1} & 0_{1\times (2^{11-2}-1)} &
\lp\b^{([1,3,10])}_{eq}\rp^T             
\end{bmatrix}^T\nonumber \\
\b^{([1,3,10])}_{eq} & = & \begin{bmatrix}
 q_{sx2} \1_{1\times (8-1) } & 0_{1\times (2^{10-2}-1)} &
\lp\b^{([1,3,9])}_{eq}\rp^T             
\end{bmatrix}^T,
\end{eqnarray}
where $\b^{([1,3,9])}_{eq}$ is determined in (\ref{eq:ult2kogpeq15a1b1c1d0}).
 Numerical evaluations via program (\ref{eq:ult2kogpeq15a1b3}) give results shown in Table \ref{tab:tab12}. We now note that the bounding trend reverses compared to  $\k=[1,3,9]$ which means that the bound is loose. Memory requirements make further increase in $k$ challenging. However, given that the bounding trend reversed that might not be needed.
\begin{table}[h]
\caption{Union upper bound $\bar{\alpha}_{ult_2}(1;[1,3,12])$ on $\alpha_{ult_2}(1)$  }\vspace{.1in}
\centering
\def\arraystretch{1.2}
 \begin{tabular}{||c||c|c|c||}\hline\hline
 \hspace{-0in}                                               $\k$     &  $[1,3,6]$  &  $[1,3,9]$    &  $[1,3,12]$     \\ \hline\hline
    $[q_1,q_2]$  & $[0.9954,0.9598 ]$  & $[0.9984,0.9746 ]$  & $[0.9989,0.9755 ]$      \\ \hline\hline 
   $\bar{\alpha}_{ult_2}(1;\k)$ &   \bl{$\mathbf{1.6300}$} &   \red{$\mathbf{1.6236}$}     &   \bl{$\mathbf{1.6308}$}      \\ \hline\hline
  \end{tabular}
\label{tab:tab12}
\end{table}

\subsubsection{$ult_2$-OGP -- $\k=[1,4,k]$}
\label{sec:ult2k41}

After considering $k_1=2$ and $k_1=3$ choices in previous sections, we now focus on the first next $k_1=4$. Adjustment (\ref{eq:ult23ogpeq9})  continuous to be in place. We keep paralleling what we presented so far and separately consider a couple of smallest allowable values of $k$ when 
$k_1=4$.

\subsubsubsection{$ult_2$-OGP -- $\k=[1,4,8]$}
\label{sec:ult2841}

Analogously to (\ref{eq:ult2kogpeq15a1b1c1d0e0}), we now have

 \begin{eqnarray}\label{eq:ult2kogpeq15a1b1c1d0e0f0}
\b^{([1,4,8])}_{eq} & = & \begin{bmatrix}
 q_{sx2} \1_{1\times 4 } & q_{sx1} & q_{sx1} & q_{sx1}  & 0_{1\times (2^{8-2}-1)} &
\lp\b^{([1,4,7])}_{eq}\rp^T             
\end{bmatrix}^T
\nonumber \\
\b^{([1,4,7])}_{eq} & = & \begin{bmatrix}
 q_{sx2} \1_{1\times 4} & q_{sx1}  & q_{sx1}  & 0_{1\times (2^{7-2}-1)} &
\lp\b^{([1,4,6])}_{eq}\rp^T             
\end{bmatrix}^T\nonumber \\
\b^{([1,4,6])}_{eq} & = & \begin{bmatrix}
 q_{sx2} \1_{1\times 4 } & q_{sx1} & 0_{1\times (2^{6-2}-1)} &
\lp\b^{([1,4,5])}_{eq}\rp^T             
\end{bmatrix}^T,
\end{eqnarray}
where  
{\small\begin{equation}\label{eq:ult25eq50}
\b^{([1,4,5])}_{eq} =\begin{bmatrix}
 q_{sx2} & q_{sx2} & q_{sx2} &  q_{sx2} &  q_{s2} & 0 & 0 & 0 & 0 & 0 & 0 &
           q_{sx1} & q_{sx1} & q_{sx1} & q_{s1}-q_{s2} & \frac{1-q_{s1}}{2}  &  \frac{1-q_{s1}}{2}  
\end{bmatrix}^T.
\end{equation}}

Results obtained through numerical evaluations via (\ref{eq:ult2kogpeq15a1b3}) are shown in Table \ref{tab:tab13}. The bound is smaller than all except one of the earlier bounds, suggesting that further improvement might be possible. 

\begin{table}[h]
\caption{Union upper bound $\bar{\alpha}_{ult_2}(1;[1,4,8])$ on $\alpha_{ult_2}(1)$  }\vspace{.1in}
\centering
\def\arraystretch{1.2}
 \begin{tabular}{||c||c|c||}\hline\hline
 \hspace{-0in}                                               $\k$     &  $[1,4,8]$      \\ \hline\hline
    $[q_1,q_2]$  & $[0.9967,0.9583 ]$        \\ \hline\hline 
   $\bar{\alpha}_{ult_2}(1;\k)$ &   \bl{$\mathbf{1.6268}$}      \\ \hline\hline
  \end{tabular}
\label{tab:tab13}
\end{table}

\subsubsubsection{$ult_2$-OGP -- $\k=[1,4,12]$}
\label{sec:ult21241}

Analogously to (\ref{eq:ult2kogpeq15a1b1c1d0e0f0}), we now have

\begin{eqnarray}\label{eq:ult2kogpeq15a1b1c1d0e0f0g0}
\b^{([1,4,12])}_{eq} & = & \begin{bmatrix}
 q_{sx2} \1_{1\times 8 } & q_{sx1} & q_{sx1} & q_{sx1}  & 0_{1\times (2^{12-2}-1)} &
\lp\b^{([1,4,11])}_{eq}\rp^T             
\end{bmatrix}^T
\nonumber \\
\b^{([1,4,11])}_{eq} & = & \begin{bmatrix}
 q_{sx2} \1_{1\times 8} & q_{sx1}  & q_{sx1}  & 0_{1\times (2^{11-2}-1)} &
\lp\b^{([1,4,10])}_{eq}\rp^T             
\end{bmatrix}^T
\nonumber \\
\b^{([1,4,10])}_{eq} & = & \begin{bmatrix}
 q_{sx2} \1_{1\times 8 } & q_{sx1} & 0_{1\times (2^{10-2}-1)} &
\lp\b^{([1,4,9])}_{eq}\rp^T             
\end{bmatrix}^T\nonumber \\
\b^{([1,4,9])}_{eq} & = & \begin{bmatrix}
 q_{sx2} \1_{1\times 8 }  & 0_{1\times (2^{9-2}-1)} &
\lp\b^{([1,4,8])}_{eq}\rp^T             
\end{bmatrix}^T,
\end{eqnarray}
where  $\b^{([1,4,8])}_{eq}$ is determined in (\ref{eq:ult2kogpeq15a1b1c1d0e0f0}). Numerical evaluations via (\ref{eq:ult2kogpeq15a1b3}) give results in Table \ref{tab:tab14}. The bound is smaller than all of the earlier bounds. All of the bounds obtained on the second level of ultrametricity are visualized in Figure \ref{fig:fig2}.

Memory requirements render further increase in $k$ impractical. Given that for $k_1=2$ and $k_1=3$ we observed a bounding reversal trend after the first two smallest allowed $k$, it might be reasonable to expect that similar thing would appear here as well, effectively making further increase in $k$ of no use. 

\begin{table}[h]
\caption{Union upper bound $\bar{\alpha}_{ult_2}(1;[1,4,12])$ on $\alpha_{ult_2}(1)$  }\vspace{.1in}
\centering
\def\arraystretch{1.2}
 \begin{tabular}{||c||c|c||}\hline\hline
 \hspace{-0in}                                               $\k$     &  $[1,4,8]$      &  $[1,4,12]$      \\ \hline\hline
    $[q_1,q_2]$  & $[0.9967,0.9583 ]$     & $[0.9989,0.9745 ]$        \\ \hline\hline 
   $\bar{\alpha}_{ult_2}(1;\k)$ &   \bl{$\mathbf{1.6268}$}    &   \red{$\mathbf{1.6219}$}      \\ \hline\hline
  \end{tabular}
\label{tab:tab14}
\end{table}

 \begin{figure}[h]
\centering
\centerline{\includegraphics[width=.7\linewidth]{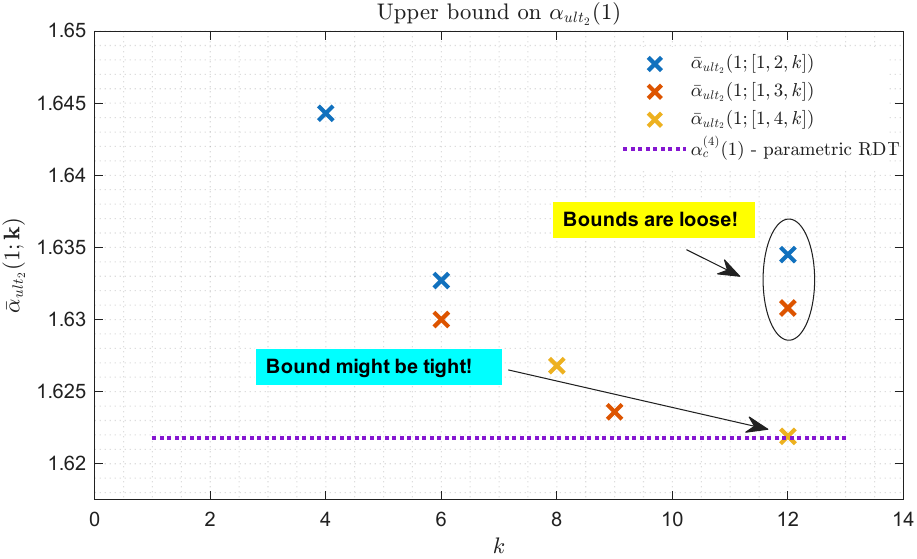}}
\caption{Upper bounds $\bar{\alpha}_{ult_2}(1;\k)$ on ultrametric OGP critical constraint  density $\alpha_{ult_2}(1)$.}
\label{fig:fig2}
\end{figure}

In any event, we note the following remarkable property of the obtained bounds. Namely, our best bound, $\approx 1.6219$, on the second level of ultrametricity is very close to the $\approx 1.6218$ prediction obtained on the third partial level of lifting via parametric RDT in \cite{Stojnicalgsbp26}. Moreover, this parallels the trend that we observed on the first level of ultrametricity, where our best bound $\approx 1.6578$ is also remarkably close to second partial lifting level prediction $\approx 1.6576$ obtained via parametric RDT in \cite{Stojnicalgsbp26}.

These developments could be a coincidence. However, they can also point to a potential connection between the ultrametric OGP and parametric RDT and statistical computational gaps. Checking what happens on higher levels of ultrametricity might shed some light on whether there is indeed such a connection. Proceeding in that direction at present is not a smooth process as overcoming memory requirements appears rather challenging. Nonetheless, we conducted some limited evaluations on the third ultrametric level and we discuss the obtained results next.

\subsection{Third level of ultrametricity -- $ult_3$-OGP}
\label{sec:ult3}

We start by noting that  $s=3$ on the third level of ultametricity. Specializing (\ref{eq:ultogp6}) and (\ref{eq:ultogp7}) to $s=3$ we further have  $\q=[q_0,q_1,q_2,q_3,q_4]$ and $\k=[k_0,k_1,k_2,k_3]$ with
  \begin{equation}\label{eq:ult3ogp6}
1  =q_0 > q_1> q_2 > q_3>q_4=0, \quad \mbox{and } \quad 1 =k_0 <  k_1 < k_2 < k_3 = k, \mbox{ where } \frac{k_{3}}{k_2},\frac{k_{2}}{k_1},\frac{k_{1}}{k_0}\in\mZ.
\end{equation}
  Analogously to (\ref{eq:ult2eq7}), we write for a $Q\in\cQ_{ult_3}(\k)$
 \begin{eqnarray}\label{eq:ult3eq7}
& &  \hspace{-.4in}\mP_G(\exists X\in\cX(\k;Q)\mbox{ such that } |GX_{:,i}|\leq \kappa \1, i=1,2,\dots,k) 
\leq 
   \lp2^{h_q(\k)} p_q(\k)^{\alpha}\rp^n,
    \end{eqnarray}
    where, as in (\ref{eq:ult2eq8}) 
\begin{eqnarray}\label{eq:ult3eq8}
h^{(3)}_{\q}(\k) &  = & \frac{\log_2\lp  |\cX(\k;Q)|  \rp}{n} \nonumber \\
p^{(3)}_{\q}(\k) & = &         \mP\lp \forall i, |G_{1,:}X_{:,i}|\leq \kappa \mbox{ and } X^TX=Q  \rp .
    \end{eqnarray}    
Following the methodology of previous sections, we then write
\begin{eqnarray}\label{eq:ult33ogpeq15}
 p^{(3)}_{\q}(\k) & = &   \frac{1}{\lp 2\pi\rp^{\frac{k}{2}} \det(Q)^{\frac{1}{2}} }\int_{|\g|\leq \kappa} e^{-\frac{1}{2}\g^TQ^{-1}\g}d\g,
    \end{eqnarray}    
where 
\begin{eqnarray}\label{eq:ult33ogpeq14}
\g  =  \begin{bmatrix}
                  g_1 \\ g_2 \\\vdots \\g_k        \end{bmatrix} 
\quad                  \mbox{ and } \quad
              Q  =  (1-q_1)I +(q_1-q_2)\cI^{(1)}   + (q_2-q_3) \cI^{(2)}   + q_3 \cI^{(3)}  ,
               \end{eqnarray}
               and
\begin{eqnarray}\label{eq:ult33ogpeq14a0}               
               \cI^{(1)} =  I_{\frac{k}{k_1}\times\frac{k}{k_1}} \1_{k_1}\1_{k_1}^T   \quad                  \mbox{ and } \quad              
               \cI^{(2)} =  I_{\frac{k}{k_2}\times\frac{k}{k_2}} \1_{k_2}\1_{k_2}^T   \quad                  \mbox{ and } \quad              
              \cI^{(3)} =\1_k\1_k^T .
\end{eqnarray}
After setting
 \begin{eqnarray}\label{eq:ult33ogpeq15a1}
c_{w1} & =  &\frac{1}{1-q_1} \nonumber \\
 c_{w2} &  = & -(c_{w1})^2\lp \frac{1}{q_1-q_2} + k_1 c_{w1} \rp^{-1} 
\nonumber \\
c_{w3} & = & -(c_{w1}+k_1 c_{w2})^2\lp \frac{1}{q_2-q_3} +k_2( c_{w1} + k_1 c_{w2} ) \rp^{-1}
\nonumber \\
c_{w4} & = & -(c_{w1}+k_1 c_{w2}+k_2 c_{w3})^2\lp \frac{1}{q_3} +k( c_{w1} + k_1 c_{w2}+ k_2 c_{w3} ) \rp^{-1},
\end{eqnarray}
one finds
\begin{eqnarray}\label{eq:ult33ogpeq15a0b0}
Q^{-1}  &  = &   c_{w1}I + c_{w2}\cI^{(1)} +  c_{w3} \cI^{(2)}+  c_{w4} \cI^{(3)}.
\end{eqnarray}
We then partition $\g$ 
\begin{eqnarray}\label{eq:ult33ogpeq15a1b1}
 \g  =  \begin{bmatrix}
                  \bar{\bar{\g}}_{1} \\ \bar{\bar{\g}}_{2} \\\vdots \\\bar{\bar{\g}}_{\frac{k}{k_2}}        \end{bmatrix} ,
                  \mbox{ where }
 \bar{\bar{\g}}_j  =  \begin{bmatrix}
                  \bar{\g}_{(j-1)\frac{k_2}{k_1} +1} \\ \bar{\g}_{(j-1)\frac{k_2}{k_1} +2} \\\vdots \\ \bar{\g}_{(j-1)\frac{k_2}{k_1} + \frac{k_2}{k_1}}        \end{bmatrix} ,
                  \mbox{ and }
\bar{\g}_{(j-1)\frac{k_2}{k_1} +i}  =  \begin{bmatrix}
                   g_{(j-1)k_2+(i-1)k_1+1} \\ g_{(j-1)k_2+(i-1)k_1+2} \\\vdots \\ g_{(j-1)k_2+(i-1)k_1+k_1}        \end{bmatrix} ,  
\end{eqnarray}
and further transform (\ref{eq:ult33ogpeq15}) 
\begin{eqnarray}\label{eq:ult33ogpeq15a2}
 p^{(3)}_{\q}(\k) & = &   \frac{1}{\lp 2\pi\rp^{\frac{k}{2}} \det(Q)^{\frac{1}{2}} }\int_{|\g|\leq \kappa} e^{-\frac{1}{2}\g^T\lp    c_{w1} I + c_{w2}\cI^{(1)}  + c_{w3}\cI^{(2)}  + c_{w4}\cI^{(3)}  \rp \g}d\g
 \nonumber \\
 & = &   \frac{1}{\lp 2\pi\rp^{\frac{k}{2}} \det(Q)^{\frac{1}{2}} }\int_{|\g|\leq \kappa} e^{-\frac{c_{w1}}{2}\g^T\g  -\frac{ c_{w2}}{2}\g^T \cI^{(1)} \g -\frac{ c_{w3}}{2}\g^T \cI^{(2)} \g  -\frac{ c_{w4}}{2}\g^T \cI^{(3)} \g  }d\g
 \nonumber \\
 & = &   \frac{1}{\lp 2\pi\rp^{\frac{k}{2}} \det(Q)^{\frac{1}{2}} }
 \nonumber \\
& & \times
 \int_{|\g|\leq \kappa} e^{-\frac{c_{w1}}{2}\g^T\g  }  
\prod_{i=1}^{\frac{k}{k_2}}
 \frac{1}{\sqrt{2\pi}}\int_{\bar{\bar{z}}_i} e^{\sqrt{-c_{w3}}\1^T \bar{\bar{\g}}_i \bar{\bar{z}}_i -\frac{\bar{\bar{z}}_i^2}{2}} d\bar{\bar{z}}_i  
  \nonumber \\
& &
 \times
 \prod_{i=1}^{\frac{k_2}{k_1}}
 \frac{1}{\sqrt{2\pi}}\int_{\bar{z}_i} e^{\sqrt{-c_{w2}}\1^T \bar{\g}_i \bar{z}_i -\frac{\bar{z}_i^2}{2}} d\bar{z}_i 
 \frac{1}{\sqrt{2\pi}}\int_z e^{\sqrt{-c_{w4}}\1^T \g z -\frac{z^2}{2}} dz d\g
 \nonumber \\
 & = &   \frac{1}{\det(Q)^{\frac{1}{2}} } \frac{1}{\sqrt{2\pi}^{1+\frac{k_2}{k_1}+\frac{k}{k_2}}} 
  \int_z
 \prod_{i=1}^{\frac{k}{k_2}}
\int_{\bar{\bar{z}}_i}
\nonumber \\
& & \times
 \prod_{i=1}^{\frac{k_2}{k_1}}
\int_{\bar{z}_i}
\lp
\frac{1}{\lp 2\pi\rp^{\frac{k_1}{2}} }
 \int_{|\bar{\g}_i|\leq \kappa} e^{-\frac{c_{w1}}{2}\bar{\g}_i^T\bar{\g}_i  +\sqrt{-c_{w2}}\1^T \bar{\g}_i \bar{z}_i  
  +\sqrt{-c_{w3}}\1^T \bar{\g}_i \bar{\bar{z}}_i  
  +\sqrt{-c_{w4}}\1^T \bar{\g}_i z  } d\bar{\g}_i 
 \rp
 e^{-\frac{\bar{z}_i^2}{2}} d\bar{z}_i
 \nonumber \\
 & & \times
 e^{-\frac{\bar{\bar{z}}_i^2}{2}} d\bar{\bar{z}}_i
 e^{-\frac{z^2}{2}} dz
\nonumber \\
 & = &   \frac{1}{\det(Q)^{\frac{1}{2}} } \frac{1}{\sqrt{2\pi}^{1+\frac{k_2}{k_1}+\frac{k}{k_2}}} 
  \int_z
 \prod_{i=1}^{\frac{k}{k_2}}
\int_{\bar{\bar{z}}_i}
\nonumber \\
& & \times
 \prod_{i=1}^{\frac{k_2}{k_1}}
\int_{\bar{z}_i}
\lp
\frac{1}{\lp 2\pi\rp^{\frac{1}{2}} }
 \int_{|g_1|\leq \kappa} e^{-\frac{c_{w1}}{2}g_1^2  +\sqrt{-c_{w2}}  g_1 \bar{z}_i  
   +\sqrt{-c_{w3}}  g_1 \bar{\bar}{z}_i  +\sqrt{-c_{w4}} g_1 z  } dg_1 
 \rp^{k_1}
 e^{-\frac{\bar{z}_i^2}{2}} d\bar{z}_i
\nonumber \\
& & \times
 e^{-\frac{\bar{\bar{z}}_i^2}{2}} d\bar{\bar{z}}_i
 e^{-\frac{z^2}{2}} dz
 \nonumber \\
 & = &   \frac{1}{\det(Q)^{\frac{1}{2}} } \frac{1}{\sqrt{2\pi}^{1+\frac{k_2}{k_1}+\frac{k}{k_2}}} 
  \int_z
 \prod_{i=1}^{\frac{k}{k_2}}
\int_{\bar{\bar{z}}_i}
 \prod_{i=1}^{\frac{k_2}{k_1}}
\int_{\bar{z}_i}
\lp
I_0^{(3)}(\bar{z}_i,\bar{\bar{z}}_i,z) \rp^{k_1}
 e^{-\frac{\bar{z}_i^2}{2}} d\bar{z}_i
 e^{-\frac{\bar{\bar{z}}_i^2}{2}} d\bar{\bar{z}}_i
 e^{-\frac{z^2}{2}} dz
 \nonumber \\
 & = &   \frac{1}{\det(Q)^{\frac{1}{2}} } \frac{1}{\sqrt{2\pi}^{1+\frac{k}{k_2}}} 
   \int_z
 \prod_{i=1}^{\frac{k}{k_2}}
\int_{\bar{\bar{z}}_i}
 \lp
 \int_{\bar{z}_1}
 \frac{1}{\sqrt{2\pi} }
\lp
I_0^{(3)}(\bar{z}_1,\bar{\bar{z}}_1,z) \rp^{k_1}
 e^{-\frac{\bar{z}_1^2}{2}} d\bar{z}_1
 \rp^{\frac{k_2}{k_1}}
 e^{-\frac{\bar{\bar{z}}_i^2}{2}} d\bar{\bar{z}}_i
 e^{-\frac{z^2}{2}} dz
 \nonumber \\
 & = &   \frac{1}{\det(Q)^{\frac{1}{2}} } \frac{1}{\sqrt{2\pi}} 
   \int_z
   \lp
\lp I_1^{(3)}(\bar{\bar{z}}_i,z) \rp^{\frac{k_2}{k_1}} 
 e^{-\frac{\bar{\bar{z}}_i^2}{2}} d\bar{\bar{z}}_i
 \rp^{\frac{k}{k_2}}
e^{-\frac{z^2}{2}} dz \nonumber \\
 & = &   \frac{1}{\det(Q)^{\frac{1}{2}} } \frac{1}{\sqrt{2\pi}} 
   \int_z
   \lp
I_2^{(3)}(z) \rp^{\frac{k}{k_2}}
e^{-\frac{z^2}{2}} dz,
    \end{eqnarray}    
    where
\begin{eqnarray}\label{eq:ult33ogpeq15a2b0}
  I_0^{(3)}(\bar{z}_1,\bar{\bar{z}}_1,z) 
& = &
\frac{1}{\lp 2\pi\rp^{\frac{1}{2}} }
 \int_{|g_1|\leq \kappa} e^{-\frac{c_{w1}}{2}g_1^2  +\sqrt{-c_{w2}}  g_1 \bar{z}_1   
 +\sqrt{-c_{w3}}  g_1 \bar{\bar{z}}_1  
 +\sqrt{-c_{w4}} g_1 z  } dg_1 
 \nonumber \\
   I_1^{(3)}(\bar{\bar{z}}_i,z) 
 & = &
 \frac{1}{\sqrt{2\pi} }
 \int_{\bar{z}_1}
\lp
  I_0^{(3)}(\bar{z}_1,\bar{\bar{z}}_1,z) 
 \rp^{k_1}
 e^{-\frac{\bar{z}_1^2}{2}} d\bar{z}_1 
 \nonumber \\
   I_2^{(3)}(z) 
 & = &
 \frac{1}{\sqrt{2\pi} }
 \int_{z}
\lp
  I_1^{(3)}(\bar{\bar{z}}_1,z) 
 \rp^{\frac{k_2}{k_1}}
 e^{-\frac{\bar{\bar{z}}_1^2}{2}} d\bar{\bar{z}}_1.
      \end{eqnarray}    
After setting
\begin{eqnarray}
\label{eq:ult33ogpeq15a3}
D^{(3)} = \frac{\sqrt{-c_{w2}} }  {\sqrt{c_{w1}} } \bar{z}_1
+ \frac{\sqrt{-c_{w3}} }  {\sqrt{c_{w1}} } \bar{\bar{z}}_1
+
 \frac{\sqrt{-c_{w4}} }  {\sqrt{c_{w1}} } z 
  \quad \mbox{ and } \quad  E^{(3)}  = \sqrt{c_{w1}}\kappa,
\end{eqnarray}
and
\begin{eqnarray}\label{eq:ult33ogpeq15a4}
 \tilde{I}_0^{(3)}(\bar{z}_1,\bar{\bar{z}}_1,z) 
 & = &
 -\frac{{\mathrm{e}}^{\frac{\lp D^{(3)}\rp^2}{2}}\,\left(\mathrm{erf}\left(\frac{\sqrt{2}\,D^{(3)}}{2}
 -\frac{\sqrt{2}\,E^{(3)}}{2}\right)-\mathrm{erf}\left(\frac{\sqrt{2}\,D^{(3)}}{2}
 +\frac{\sqrt{2}\,E^{(3)}}{2}\right)\right)}{2},
 \nonumber \\
   \tilde{I}_1^{(3)}(z) 
 & = &
 \frac{1}{\sqrt{2\pi} }
 \int_{\bar{z}_1}
\lp
\tilde{I}_0^{(3)}(\bar{z}_1,\bar{\bar{z}}_1,z) \rp^{k_1}
 e^{-\frac{\bar{z}_1^2}{2}} d\bar{z}_1 \nonumber \\
   \tilde{I}_2^{(3)}(z) 
 & = &
 \frac{1}{\sqrt{2\pi} }
 \int_{\bar{\bar{z}}_1}
\lp
\tilde{I}_1^{(3)}(\bar{z}_1,\bar{\bar{z}}_1,z) \rp^{\frac{k_2}{k_1}}
 e^{-\frac{\bar{\bar{z}}_1^2}{2}} d\bar{\bar{z}}_1,
      \end{eqnarray}    
and finds
\begin{eqnarray}\label{eq:ult33ogpeq15a5}
 p^{(3)}_{\q}(\k) & = &   
  \frac{1}{\sqrt{c_{w1}}^k\det(Q)^{\frac{1}{2}} } \frac{1}{\sqrt{2\pi}} 
   \int_z
\lp \tilde{I}_2^{(3)}(z) \rp^{\frac{k}{k_2}} e^{-\frac{z^2}{2}} dz.
    \end{eqnarray} 
    
For $h^{(3)}_{\q}(\k)$ we again utilize  
 \begin{eqnarray}\label{eq:ult3kogpeq15a1b3}
h^{(3)}_{\q}(\k) = 1 + \max_{\a} & & h_v(\a) 
\nonumber \\
\mbox{subject to} & &  A^{(\k)}_{eq}\a =\b^{(\k)}_{eq}
\nonumber \\
  & &  \sum_{i=1}^{l_{c,k}+1}\a_i =1 .
    \end{eqnarray}
A combination of  (\ref{eq:ult33ogpeq15a1})  and (\ref{eq:ult33ogpeq15a3})--(\ref{eq:ult3kogpeq15a1b3})  is sufficient to utilize (\ref{eq:ult3eq7}) and obtain the following upper bound
\begin{eqnarray}\label{eq:ult33ogpeq16}
 \alpha_{ult_3}(\kappa) &  \leq  & \min_{\q} \lp  - \frac{h^{(3)}_{\q}(\k) \log(2) }{\log(p^{(3)}_{\q}(\k))} \rp  \triangleq  \bar{\alpha}_{ult_3}(\kappa;\k).
    \end{eqnarray}
        
As was the case on the second level of ultrametricity, to be able to put the above machinery into practical use, one needs to determine $A^{(\k)}_{eq}$ and $\b^{(\k)}_{eq}$ in program (\ref{eq:ult3kogpeq15a1b3}).  One again has $A^{(\k)}_{eq}=A^{(k)}_{eq}$ and the only thing that needs to be adjusted is $\b^{(\k)}_{eq}$.  Below we consider a few particular $\k$ choices together with associated needed $\b^{(\k)}_{eq}$ adjustments.

Before getting into the concrete $\k$ choices, we note that on the third level of ultrametricity  we have $q_1$, $q_2$, and $q_3$ wchih implies that (\ref{eq:ult23ogpeq9}) needs a slight adjustment as well. To that end we write
\begin{eqnarray}\label{eq:ult33ogpeq9}
q_{sx1} & = & (1+q_1)/2 \nonumber \\
q_{sx2} & = & (1+q_2)/2 \nonumber \\
q_{sx3} & = & (1+q_3)/2 \nonumber \\
q_{s1} & = &  q_{sx1}
 \nonumber \\ 
    q_{s2} &  = &  q_{sx2} - (1-q_{s1})/2 .
\end{eqnarray}

\subsubsection{$ult_3$-OGP -- $\k=[1,2,4,8]$}
\label{sec:ult38421}

We start considerations with simplest possible $\k$ sequences. On the third level of ultrametricity that is precisely $\k=[1,2,4,8]$. As mentioned above, the only thing that needs an additional attention is $\b^{(\k)}_{eq} = \b^{([1,2,4,8])}_{eq}$. We continue to apply the strategy developed above. Analogously to (\ref{eq:ult2kogpeq15a1b1c1d0e0f0}) we write

 \begin{eqnarray}\label{eq:ult3exmp1}
\b^{([1,2,4,8])}_{eq} & = & \begin{bmatrix}
 q_{sx3}  & q_{sx3}  & q_{sx3} & q_{sx3} &  q_{sx2}  & q_{sx2} & q_{sx1} & 0_{1\times (2^{8-2}-1)} &
\lp\b^{([1,2,4,7])}_{eq}\rp^T             
\end{bmatrix}^T
\nonumber \\
\b^{([1,2,4,7])}_{eq} & = & \begin{bmatrix}
 q_{sx3}  & q_{sx3}  & q_{sx3} & q_{sx3} &  q_{sx2}  & q_{sx2} & 0_{1\times (2^{7-2}-1)} &
\lp\b^{([1,2,4,6])}_{eq}\rp^T             
\end{bmatrix}^T\nonumber \\
\b^{([1,2,4,6])}_{eq} & = & \begin{bmatrix}
  q_{sx3}  & q_{sx3}  & q_{sx3} & q_{sx3} &  q_{sx1}  &  & 0_{1\times (2^{6-2}-1)} &
\lp\b^{([1,2,4,5])}_{eq}\rp^T             
\end{bmatrix}^T,
\end{eqnarray}
where  
{\small\begin{equation}\label{eq:ult3exmp2}
\b^{([1,2,4,5])}_{eq} =\begin{bmatrix}
 q_{sx3} & q_{sx3} & q_{sx3} &  q_{sx3} &  q_{s2} & 0 & 0 & 0 & 0 & 0 & 0 &
           q_{sx2} & q_{sx2} & q_{sx1} & q_{s1}-q_{s2} & \frac{1-q_{s1}}{2}  &  \frac{1-q_{s1}}{2}  
\end{bmatrix}^T.
\end{equation}}

\noindent We have done on many occasions adjustments similar to the one given above. Since this is the first $\k$ sequence on the third level of ultrametricity, we find it useful to briefly recall on the key reasoning points that allow to  construct (\ref{eq:ult3exmp1}). 

First, we recall that the ultrametric structure governed by sequence  $\k=[1,2,4,8]$ has four $\frac{8}{4}=2$ clusters of size four (the first four vectors, $(X_{:,1},X_{:,2},X_{:,3},X_{:,4})$, belong to the first big cluster whereas the remaining four $(X_{:,5},X_{:,6},X_{:,7},X_{:,8})$ belong to the second). Each of these clusters has $\frac{4}{2}=2$ smaller clusters of size two (they are $(X_{:,1},X_{:,2})$ and  $(X_{:,3},X_{:,4})$ for the first big cluster and $(X_{:,5},X_{:,6})$ and  $(X_{:,7},X_{:,8})$  for the second one).

Now we look at vectors $\b_{eq}$. First seven components of $\b^{([1,2,4,8])}_{eq} $ relate to sign overlaps between the eighth vector $X_{:,8}$ and the first seven $X_{:,i},i=1,2,\dots,7$. The first four denote ($n$ scaled) sign overlaps with the vectors in the most distant cluster and are equal to $q_{sx3}$. Fifth and sixth denote the ($n$ scaled) sign overlaps  with the elements in the  closer  cluster. They are equal to $q_{sx2}$. Finally, the seventh element, $q_{sx1}$, is the sign overlap with the element of the same cluster. 

The reasoning for $\b^{([1,2,4,7])}_{eq} $ and $\b^{([1,2,4,6])}_{eq} $ is analogous. The first six elements of $\b^{([1,2,4,7])}_{eq} $ denote the $n$ scaled sign overlaps between the  seventh vector $X_{:,7}$  and the first six $X_{:,i},i=1,2,\dots,6$. Since this vector belongs to the same cluster as the eighth one, the first six elements match the first six elements of $\b^{([1,2,4,8])}_{eq} $. The first five elements of $\b^{([1,2,4,6])}_{eq} $ denote the $n$ scaled sign overlaps between the sixth vector $X_{:,6}$ and the first five $X_{:,i},i=1,2,\dots,5$. Since the sixth vector belongs to the same big cluster as the seventh and eighth, its first four elements are equal to the first four elements of $\b^{([1,2,4,7])}_{eq}$ and $\b^{([1,2,4,8])}_{eq}$. However, one should note that the fifth element is different as the sixth vector belongs to a different smaller cluster.

All other parts of vectors $\b^{([1,2,4,6])}_{eq} $, $\b^{([1,2,4,7])}_{eq} $, and $\b^{([1,2,4,8])}_{eq} $ have the form that automatically follows from the recursion established for $k>5$ in earlier section. One also needs $\b^{([1,2,4,5])}_{eq} $. The first four elements of $\b^{([1,2,4,5])}_{eq} $ are the $n$ scaled sign overlaps between the fifth vector $X_{:,5}$ and the first four $X_{:,i},i=1,2,3,4$ and they are $q_{sx3}$. Elements 12-14 correspond to the same type of overlaps between the fourth vector and the first three. All other elements  of $\b^{([1,2,4,5])}_{eq} $ relate to the overlaps between the first three vectors and these have fixed structure discussed on a multitude of occasions in earlier sections.

The above is then sufficient to conduct numerical evaluations via (\ref{eq:ult3kogpeq15a1b3}). The obtained results are shown in Table \ref{tab:tab15}. The bound is smaller than all of the earlier bounds
obtained on the first and second ultrametric level. This suggest that the bounding mechanism might have the right trend.

\begin{table}[h]
\caption{Union upper bound $\bar{\alpha}_{ult_3}(1;[1,2,4,8])$ on $\alpha_{ult_3}(1)$  }\vspace{.1in}
\centering
\def\arraystretch{1.2}
 \begin{tabular}{||c||c||}\hline\hline
 \hspace{-0in}                                               $\k$     &  $[1,2,4,8]$      \\ \hline\hline
    $[q_1,q_2,q_3]$  & $[0.9986,0.9928,0.9590]$          \\ \hline\hline 
   $\bar{\alpha}_{ult_3}(1;\k)$ &   \bl{$\mathbf{1.6193}$}       \\ \hline\hline
  \end{tabular}
\label{tab:tab15}
\end{table}

In the above scenario in any of the clusters the number of largest subcluster is exactly two which is minimal possible. Below we consider three scenarios $\k=[1,2,6,12]$, $\k=[1,3,6,12]$, and $\k=[1,2,2,12]$  where some of the clusters have three largest subclusters.

\subsubsection{$ult_3$-OGP -- $\k=[1,2,6,12]$}
\label{sec:ult312621}

As usual, we need to adjust $\b^{(\k)}_{eq} = \b^{([1,2,6,12])}_{eq}$.  Analogously to (\ref{eq:ult2kogpeq15a1b1c1d0e0f0g0}) we write
\begin{eqnarray}\label{eq:ult3exmp3}
\b^{([1,2,6,12])}_{eq} & = & \begin{bmatrix}
 q_{sx3} \1_{1\times 6 } & q_{sx2} & q_{sx2} & q_{sx2}  & q_{sx2} & q_{sx1}  & 0_{1\times (2^{12-2}-1)} &
\lp\b^{([1,2,6,11])}_{eq}\rp^T             
\end{bmatrix}^T
\nonumber \\
\b^{([1,2,6,11])}_{eq} & = & \begin{bmatrix}
 q_{sx3} \1_{1\times 6 } & q_{sx2} & q_{sx2} & q_{sx2}  & q_{sx2}  & 0_{1\times (2^{11-2}-1)} &
\lp\b^{([1,2,6,10])}_{eq}\rp^T             
\end{bmatrix}^T
\nonumber \\
\b^{([1,2,6,10])}_{eq} & = & \begin{bmatrix}
 q_{sx3} \1_{1\times 6 } & q_{sx2} & q_{sx2} & q_{sx1}  & 0_{1\times (2^{10-2}-1)} &
\lp\b^{([1,2,6,9])}_{eq}\rp^T             
\end{bmatrix}^T
\nonumber \\
\b^{([1,2,6,9])}_{eq} & = & \begin{bmatrix}
 q_{sx3} \1_{1\times 6 } & q_{sx2} & q_{sx2} & 0_{1\times (2^{9-2}-1)} &
\lp\b^{([1,2,6,8])}_{eq}\rp^T             
\end{bmatrix}^T
\nonumber \\
\b^{([1,2,6,8])}_{eq} & = & \begin{bmatrix}
 q_{sx3} \1_{1\times 6 } & q_{sx2}  & 0_{1\times (2^{8-2}-1)} &
\lp\b^{([1,2,6,7])}_{eq}\rp^T             
\end{bmatrix}^T
\nonumber \\
\b^{([1,2,6,7])}_{eq} & = & \begin{bmatrix}
 q_{sx3} \1_{1\times 6 }  & 0_{1\times (2^{7-2}-1)} &
\lp\b^{([1,2,6,6])}_{eq}\rp^T             
\end{bmatrix}^T
\nonumber \\
\b^{([1,2,6,6])}_{eq} & = & \begin{bmatrix}
 q_{sx2}  & q_{sx2}  & q_{sx2}  & q_{sx2}  & q_{sx1}  & 0_{1\times (2^{6-2}-1)} &
\lp\b^{([1,2,6,5])}_{eq}\rp^T             
\end{bmatrix}^T,
\end{eqnarray}
where  
{\small\begin{equation}\label{eq:ult3exmp4}
\b^{([1,2,6,5])}_{eq} =\begin{bmatrix}
 q_{sx2} & q_{sx2} & q_{sx2} &  q_{sx2} &  q_{s2} & 0 & 0 & 0 & 0 & 0 & 0 &
           q_{sx2} & q_{sx2} & q_{sx1} & q_{s1}-q_{s2} & \frac{1-q_{s1}}{2}  &  \frac{1-q_{s1}}{2}  
\end{bmatrix}^T.
\end{equation}}

Conducting numerical evaluations via (\ref{eq:ult3kogpeq15a1b3}), we obtained results shown in Table \ref{tab:tab16}. We observe that the bound is smaller than for $\k=[1,2,4,8]$.

\begin{table}[h]
\caption{Union upper bound $\bar{\alpha}_{ult_3}(1;[1,2,6,12])$ on $\alpha_{ult_3}(1)$  }\vspace{.1in}
\centering
\def\arraystretch{1.2}
 \begin{tabular}{||c||c|c||}\hline\hline
 \hspace{-0in}                                               $\k$     &  $[1,2,4,8]$   &  $[1,2,6,12]$      \\ \hline\hline
    $[q_1,q_2,q_3]$  & $[0.9986,0.9928,0.9590]$ & $[0.9995,0.9960,0.9560]$         \\ \hline\hline 
   $\bar{\alpha}_{ult_3}(1;\k)$ &   \bl{$\mathbf{1.6193}$}  &   \bl{$\mathbf{1.6148}$}       \\ \hline\hline
  \end{tabular}
\label{tab:tab16}
\end{table}

\subsubsection{$ult_3$-OGP -- $\k=[1,3,6,12]$}
\label{sec:ult312631}

We now need adjustment for  $\b^{(\k)}_{eq} = \b^{([1,3,6,12])}_{eq}$. Following (\ref{eq:ult3exmp3}) we write
\begin{eqnarray}\label{eq:ult3exmp5}
\b^{([1,3,6,12])}_{eq} & = & \begin{bmatrix}
 q_{sx3} \1_{1\times 6 } & q_{sx2} & q_{sx2} & q_{sx2}  & q_{sx1} & q_{sx1}  & 0_{1\times (2^{12-2}-1)} &
\lp\b^{([1,3,6,11])}_{eq}\rp^T             
\end{bmatrix}^T
\nonumber \\
\b^{([1,3,6,11])}_{eq} & = & \begin{bmatrix}
 q_{sx3} \1_{1\times 6 } & q_{sx2} & q_{sx2} & q_{sx2}  & q_{sx1}  & 0_{1\times (2^{11-2}-1)} &
\lp\b^{([1,3,6,10])}_{eq}\rp^T             
\end{bmatrix}^T
\nonumber \\
\b^{([1,3,6,10])}_{eq} & = & \begin{bmatrix}
 q_{sx3} \1_{1\times 6 } & q_{sx2} & q_{sx2} & q_{sx2}  & 0_{1\times (2^{10-2}-1)} &
\lp\b^{([1,3,6,9])}_{eq}\rp^T             
\end{bmatrix}^T
\nonumber \\
\b^{([1,3,6,9])}_{eq} & = & \begin{bmatrix}
 q_{sx3} \1_{1\times 6 } & q_{sx1} & q_{sx1} & 0_{1\times (2^{9-2}-1)} &
\lp\b^{([1,3,6,8])}_{eq}\rp^T             
\end{bmatrix}^T
\nonumber \\
\b^{([1,3,6,8])}_{eq} & = & \begin{bmatrix}
 q_{sx3} \1_{1\times 6 } & q_{sx1}  & 0_{1\times (2^{8-2}-1)} &
\lp\b^{([1,3,6,7])}_{eq}\rp^T             
\end{bmatrix}^T
\nonumber \\
\b^{([1,3,6,7])}_{eq} & = & \begin{bmatrix}
 q_{sx3} \1_{1\times 6 }  & 0_{1\times (2^{7-2}-1)} &
\lp\b^{([1,3,6,6])}_{eq}\rp^T             
\end{bmatrix}^T
\nonumber \\
\b^{([1,3,6,6])}_{eq} & = & \begin{bmatrix}
 q_{sx2}  & q_{sx2}  & q_{sx2}  & q_{sx1}  & q_{sx1}  & 0_{1\times (2^{6-2}-1)} &
\lp\b^{([1,3,6,5])}_{eq}\rp^T             
\end{bmatrix}^T,
\end{eqnarray}
where  
{\small\begin{equation}\label{eq:ult3exmp6}
\b^{([1,2,6,5])}_{eq} =\begin{bmatrix}
 q_{sx2} & q_{sx2} & q_{sx2} &  q_{sx1} &  q_{s2} & 0 & 0 & 0 & 0 & 0 & 0 &
           q_{sx2} & q_{sx2} & q_{sx2} & q_{s1}-q_{s2} & \frac{1-q_{s1}}{2}  &  \frac{1-q_{s1}}{2}  
\end{bmatrix}^T.
\end{equation}}

Numerical evaluations via (\ref{eq:ult3kogpeq15a1b3}) give the results shown in Table \ref{tab:tab17}. We observe that the bound is smaller than bounds for both $\k=[1,2,4,8]$ and  $\k=[1,2,6,12]$.

\begin{table}[h]
\caption{Union upper bound $\bar{\alpha}_{ult_3}(1;[1,3,6,12])$ on $\alpha_{ult_3}(1)$  }\vspace{.1in}
\centering
\def\arraystretch{1.2}
 \begin{tabular}{||c||c|c|c||}\hline\hline
 \hspace{-0in}                                               $\k$     &  $[1,2,4,8]$   &  $[1,2,6,12]$    &  $[1,3,6,12]$      \\ \hline\hline
    $[q_1,q_2,q_3]$  & $[0.9986,0.9928,0.9590]$ & $[0.9995,0.9960,0.9560]$     & $[0.9992,0.9920,0.9580]$         \\ \hline\hline 
   $\bar{\alpha}_{ult_3}(1;\k)$ &   \bl{$\mathbf{1.6193}$}  &   \bl{$\mathbf{1.6148}$}   &   \bl{$\mathbf{1.6131}$}       \\ \hline\hline
  \end{tabular}
\label{tab:tab17}
\end{table}

\subsubsection{$ult_3$-OGP -- $\k=[1,2,4,12]$}
\label{sec:ult312421}

We now need adjustment for  $\b^{(\k)}_{eq} = \b^{([1,2,4,12])}_{eq}$. Following (\ref{eq:ult3exmp3})  and (\ref{eq:ult3exmp5}) we write
\begin{eqnarray}\label{eq:ult3exmp7}
\b^{([1,2,4,12])}_{eq} & = & \begin{bmatrix}
 q_{sx3} \1_{1\times 8 } & q_{sx2} & q_{sx2}  & q_{sx1}  & 0_{1\times (2^{12-2}-1)} &
\lp\b^{([1,2,4,11])}_{eq}\rp^T             
\end{bmatrix}^T
\nonumber \\
\b^{([1,2,4,11])}_{eq} & = & \begin{bmatrix}
 q_{sx3} \1_{1\times 8 } & q_{sx2} & q_{sx2}  & 0_{1\times (2^{11-2}-1)} &
\lp\b^{([1,2,4,10])}_{eq}\rp^T             
\end{bmatrix}^T
\nonumber \\
\b^{([1,2,4,10])}_{eq} & = & \begin{bmatrix}
 q_{sx3} \1_{1\times 8 } & q_{sx1}   & 0_{1\times (2^{10-2}-1)} &
\lp\b^{([1,2,4,9])}_{eq}\rp^T             
\end{bmatrix}^T
\nonumber \\
\b^{([1,2,4,9])}_{eq} & = & \begin{bmatrix}
 q_{sx3} \1_{1\times 8 } & 0_{1\times (2^{9-2}-1)} &
\lp\b^{([1,2,4,8])}_{eq}\rp^T             
\end{bmatrix}^T
\nonumber \\
\b^{([1,2,4,8])}_{eq} & = & \begin{bmatrix}
 q_{sx3} \1_{1\times 4 } & q_{sx2} & q_{sx2} & q_{sx1}  & 0_{1\times (2^{8-2}-1)} &
\lp\b^{([1,2,4,7])}_{eq}\rp^T             
\end{bmatrix}^T
\nonumber \\
\b^{([1,2,4,7])}_{eq} & = & \begin{bmatrix}
 q_{sx3} \1_{1\times 4 }  & q_{sx2}  & q_{sx2}  & 0_{1\times (2^{7-2}-1)} &
\lp\b^{([1,2,4,6])}_{eq}\rp^T             
\end{bmatrix}^T
\nonumber \\
\b^{([1,2,4,6])}_{eq} & = & \begin{bmatrix}
q_{sx3} \1_{1\times 4 }   & q_{sx1}  & 0_{1\times (2^{6-2}-1)} &
\lp\b^{([1,2,4,5])}_{eq}\rp^T             
\end{bmatrix}^T,
\end{eqnarray}
where  
{\small\begin{equation}\label{eq:ult3exmp8}
\b^{([1,2,4,5])}_{eq} =\begin{bmatrix}
 q_{sx3} & q_{sx3} & q_{sx3} &  q_{sx3} &  q_{s2} & 0 & 0 & 0 & 0 & 0 & 0 &
           q_{sx2} & q_{sx2} & q_{sx1} & q_{s1}-q_{s2} & \frac{1-q_{s1}}{2}  &  \frac{1-q_{s1}}{2}  
\end{bmatrix}^T.
\end{equation}}

Conducting once again numerical evaluations via (\ref{eq:ult3kogpeq15a1b3}) we obtain results shown in Table \ref{tab:tab18}. This time, however,  the bound is smaller than some but not smaller than all bounds obtained on the third ultrametric level. We also note that our smallest bound $\approx 1.6131$ is fairly close to  $\approx 1.6093$ that was obtained via parametric RDT on the fourth partial lifting   level. All bounds obtained on the third level of ultrametriciy are visualized in Figure \ref{fig:fig3}. To ensure a systematic view and easiness of comparison, in Figure \ref{fig:fig3} we also complement the bounds from the third ultrametric level with the bounds from the first and second level.

\begin{table}[h]
{\small
\caption{Union upper bound $\bar{\alpha}_{ult_3}(1;[1,2,4,12])$ on $\alpha_{ult_3}(1)$  }\vspace{.1in}
\centering
\def\arraystretch{1.2}
 \begin{tabular}{||c||c|c|c|c||}\hline\hline
 \hspace{-0in}                                               $\k$     &  $[1,2,4,8]$   &  $[1,2,6,12]$    &  $[1,3,6,12]$    &  $[1,2,4,12]$      \\ \hline\hline
    $[q_1,q_2,q_3]$  & $[0.9986,0.9928,0.9590]$ & $[0.9995,0.9960,0.9560]$     & $[0.9992,0.9920,0.9580]$   & $[0.9994,0.9970,0.980]$         \\ \hline\hline 
   $\bar{\alpha}_{ult_3}(1;\k)$ &   \bl{$\mathbf{1.6193}$}  &   \bl{$\mathbf{1.6148}$}   &   \red{$\mathbf{1.6131}$}  &  \bl{$\mathbf{1.6178}$}       \\ \hline\hline
  \end{tabular}
  \label{tab:tab18}}
\end{table}

 \begin{figure}[h]
\centering
\centerline{\includegraphics[width=.8\linewidth]{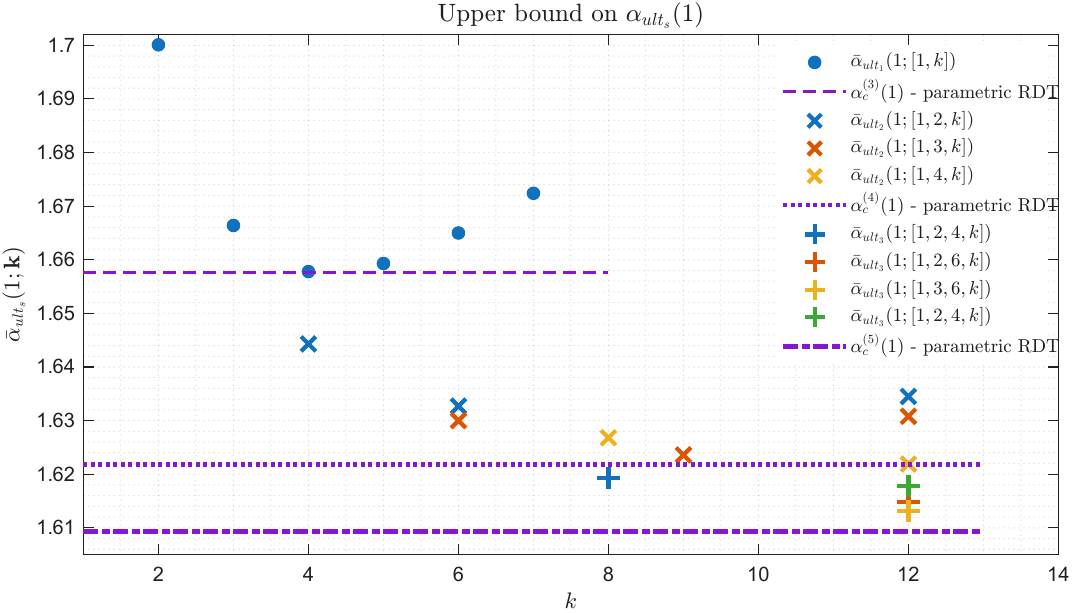}}
\caption{Upper bounds $\bar{\alpha}_{ult_s}(1;\k)$ on ultrametric OGP critical constraint  density $\alpha_{ult_s}(1)$.}
\label{fig:fig3}
\end{figure}

We should also add that as the ultrametricity level, $s$, increases the evaluations are more cumbersome which often renders achieving a full numerical precision as rather difficult. Consequently, the parametric results given in this and previous section (i.e., Tables \ref{tab:tab7}-\ref{tab:tab18} may not be fully accurate. However,
after conducting extensive numerical work, our experience is that $\bar{\alpha}_{ult}(1;\k)$  estimates may not be overly sensitive to parameters changes. This further suggests that even if some parameters  may need tiny adjustments, we would not expect such adjustments to significantly impact given $\bar{\alpha}_{ult}(1;\k)$ bounds estimates.

\subsection{$s$-th level of ultrametricity -- $ult_s$-OGP}
\label{sec:ults}

The above results for the second and third level of ultrametricity easily generalize to the $s$-th level as well. For the completeness, we formalize the above procedure so that it is readily applicable for any $s$. 

We first note that for general $s$  $\q=[q_0,q_1,q_2,\dots,q_{s+1}]$ and $\k=[k_0,k_1,k_2,\dots,k_s]$ with
  \begin{equation}\label{eq:ultseq1}
1  =q_0 > q_1> q_2>\dots > q_s>q_{s+1}=0 \quad \mbox{and } \quad 1 =k_0 <  k_1 < k_2 <\dots < k_s = k, \end{equation}
where 
 \begin{equation}\label{eq:ultseq2}
\frac{k_{s}}{k_{s-1}},\frac{k_{s-1}}{k_{s-2}},\dots,\frac{k_{2}}{k_1},\frac{k_{1}}{k_0}\in\mZ.
\end{equation}
Set 
 \begin{eqnarray}\label{eq:ultseq3}
c_{w1} & =  &\frac{1}{1-q_1} \nonumber \\
  c_{w(i+1)} &  = &
  - \lp \sum_{j=1}^{i}k_{j-1}c_{wj}  \rp^2 \lp \frac{1}{q_i-q_{i+1}} + k_i \sum_{j=1}^{i}k_{j-1}c_{wj} \rp^{-1} , i=1,2,\dots,s, 
  \nonumber \\
  \z & = & [z_1,z_2,\dots,z_s] 
  \nonumber \\  
D^{(s)} & = & \sum_{i=1}^{s}\frac{\sqrt{-c_{w(i+1)}} }  {\sqrt{c_{w1}} } z_i
  \nonumber \\  
  E^{(s)} &  =  & \sqrt{c_{w1}}\kappa,
\end{eqnarray}
and
\begin{eqnarray}\label{eq:ultseq4}
 \tilde{I}_0^{(s)}(\z) 
 & = &
 -\frac{{\mathrm{e}}^{\frac{\lp D^{(s)}\rp^2}{2}}\,\left(\mathrm{erf}\left(\frac{\sqrt{2}\,D^{(s)}}{2}
 -\frac{\sqrt{2}\,E^{(s)}}{2}\right)-\mathrm{erf}\left(\frac{\sqrt{2}\,D^{(s)}}{2}
 +\frac{\sqrt{2}\,E^{(s)}}{2}\right)\right)}{2},
 \nonumber \\
    \tilde{I}_i^{(s)}(\z) 
 & = &
 \frac{1}{\sqrt{2\pi} }
 \int_{\bar{\bar{z}}_1}
\lp
\tilde{I}_{i-1}^{(s)}(\z) \rp^{\frac{k_i}{k_{i-1}}}
 e^{-\frac{z_i^2}{2}} d z_i,i=1,2,\dots, (s-1).
      \end{eqnarray}    
For 
\begin{eqnarray}\label{eq:ultseq4a0}
 Q = \sum_{i=0}^{s} (q_{i}-q_{i+1}) I_{\frac{k}{k_{i}}\times\frac{k}{k_{i}}} \1_{k_{i}}\1_{k_{i}}^T,
               \end{eqnarray}
 and  $ p^{(s)}_{\q}(\k)$ and  $h^{(s)}_{\q}(\k)$ given by
\begin{eqnarray}\label{eq:ultseq5}
 p^{(s)}_{\q}(\k) & = &   
  \frac{1}{\sqrt{c_{w1}}^{k_s}\det(Q)^{\frac{1}{2}} } \frac{1}{\sqrt{2\pi}} 
   \int_z
\lp \tilde{I}_{s-1}^{(s)}(\z) \rp^{\frac{k_s}{k_{s-1}}} e^{-\frac{z_s^2}{2}} dz_s.
    \end{eqnarray} 
and 
 \begin{eqnarray}\label{eq:ultseq6}
h^{(s)}_{\q}(\k) = 1 + \max_{\a} & & h_v(\a) 
\nonumber \\
\mbox{subject to} & &  A^{(\k)}_{eq}\a =\b^{(\k)}_{eq}
\nonumber \\
  & &  \sum_{i=1}^{l_{c,k}+1}\a_i =1 .
    \end{eqnarray}
the following theorem summarizes union-bound strategy for $\alpha_{ult_s}(\kappa)$.

\begin{theorem}  
\label{thm:thms} 
Consider a statistical SBP $\mathbf{\mathcal S} \lp G,\kappa,\alpha \rp$ from (\ref{eq:ex1a0}) and let the associated ultrametric OGP critical constraint density, $\alpha_{ult_s}(\kappa)$, be as introduced in (\ref{eq:ultogp10})  and (\ref{eq:ultogp10a0}). Also, let $p^{(s)}_{\q}(\k)$ and $h^{(s)}_{\q}(\k) $ be as in (\ref{eq:ultseq5}) and (\ref{eq:ultseq6}), respectively. Then 
 \begin{eqnarray}\label{eq:thmseq1}
 \alpha_{ult_s}(\kappa) &  \leq  & \min_{\q} \lp  - \frac{h^{(s)}_{\q}(\k) \log(2) }{\log(p^{(s)}_{\q}(\k))} \rp  \triangleq  \bar{\alpha}_{ult_s}(\kappa;\k).
    \end{eqnarray}
\end{theorem}
\begin{proof}
  Follows by trivial adjustments of the results obtained on the second and third level of ultrametricity.
\end{proof}

\subsection{Feasibility of numerical evaluations}
\label{sec:pracnumeval}

We would like to emphasize several aspects of the numerical evaluations that may otherwise go unnoticed.

Our framework provides a generic strategy for evaluating upper bounds on the maximal allowed constraint densities, $\alpha_{ult_s}(\kappa)$, beyond which ultrametric OGP exists. We demonstrated this through systematic union-bounding, starting with the second and third levels of ultrametricity before generalizing to any level $s$. This strategy involves managing two primary components in (\ref{eq:thmseq1}): 1) The combinatorial factor $h^{(s)}_{\q}(\k)$ and 2) The probabilistic factor $p^{(s)}_{\q}(\k)$. While we have proven convenient analytical forms that allow these components to be handled in principle, their practical application remains a separate consideration.

Determining concrete values for $h^{(s)}{\q}(\k)$ and $p^{(s)}{\q}(\k)$ requires numerical evaluations that present two primary computational challenges: 1) $p^{(s)}_{\q}(\k)$ Evaluations: These are obtained as $s$-fold nested integrals. As $s$ increases, memory requirements can make the calculations computationally intractable; and 2) $h^{(s)}_{\q}(\k)$ Evaluations: These are framed as convex optimization programs with linear constraints. Despite convexity, they are difficult to solve for two reasons: (i) Scale: The number of unknowns (partition intervals) grows exponentially with $k$ and $s$, leading to significant memory constraints that currently limit applicability to lower values.
(ii) Sensitivity: Optimal $\q$ sequences often fall near the boundaries (close to 1), which can hinder formal optimization procedures. In our experience, manual ad-hoc procedures have proven more effective than formal optimization. However, it remains unclear if these strategies will be as successful when translated to larger $k$ and $s$.

At present, numerical evaluations on $s > 3$ levels appear infeasible. Nonetheless, there are several favorable features to keep in mind as potentially helpful in  further explorations. As we will see below, we actually believe that for relatively small values of $s$ (e.g., $s < 10$), the current machinery should reach the $\alpha$-range of approximately $1.58 - 1.59$. This would closely approach the local entropy \cite{Bald20} and parametric RDT \cite{Stojnicalgsbp26} predictions. While our results are currently only bounds—meaning the true $\lim_{s \rightarrow \infty} \alpha_{ult_s}(\kappa)$ values could be lower—we also believe these estimates will eventually collapse to the algorithmic threshold, $\alpha_a(\kappa)$.

The obtained bounds appear most useful for small ratios of $\frac{k_i}{k_{i-1}}$. However, as $s$ grows, the product of these ratios, $k_s$, still reaches a range that is currently computationally prohibitive. Even though the number of unknowns and constraints is large in such scenarios, most constraints are extremely sparse. Determining whether this sparsity can be leveraged in numerical solving remains a promising avenue for future research.

\subsection{$ult$-OGP -- parametric RDT connection}
\label{sec:uogpparrdt}

We now point to some striking parallels between the results that we obtained in previous sections and those related to parametric RDT given in \cite{Stojnicalgsbp26}. In particular we observe connections between the $s$-th level of ultrametricity and the$r=(s+2)$-th lifted level of parametric RDT, i.e., we observe connections between $ult_s$-OGP and $r$-fl-RDT.

We recall that \cite{Stojnicalgsbp26} put forth the concept of \emph{parametric RDT} where the co-called $\c$ sequence  -- a key RDT parameter --  is allowed to have arbitrary ordering, thereby deviating from the physically natural decreasing one. Parametric RDT is then connected to algorithmic thresholds and the following conjecture is formulated.

\begin{conjecture}\cite{Stojnicalgsbp26} [SBP algorithmic threshold] 
\label{thm:conj1} 
Consider a statistical SBP $\mathbf{\mathcal S} \lp G,\kappa,\alpha \rp$ from (\ref{eq:ex1a0}). Define its algorithmic threshold as
\begin{eqnarray}\label{eq:alphaa}
  \alpha_a (\kappa) \triangleq   \max   \left \{\alpha |\hspace{.05in}  \lim_{n\rightarrow \infty}\mP\left (   \mbox{$\mathbf{\mathcal S} \lp G,\kappa,\alpha \rp$  is solvable in polynomial time}   \right ) =1 \right \}.
 \end{eqnarray}
 Let $\alpha_c^{(r)} (\kappa)$ be $r$-th level parametric RDT capacity estimate. One then has for the SBP's statistical computational gap (SCG)
\begin{eqnarray}\label{eq:scg}
 SCG = \alpha_c (\kappa) - \alpha_a (\kappa) = \alpha_c^{(2)} (\kappa) -  \lim_{r\rightarrow \infty}\alpha_c^{(r)} (\kappa) .
 \end{eqnarray}
\end{conjecture}

The second equality follows automatically since $\alpha_c(\kappa)=\alpha_c^{(2)}(\kappa)$. This means that the first interesting developments are expected starting from the third lifting level. We therefore start out our $ult$-OGP -- parametric RDT connections discussion precisely from the third RDT lifting level.

\subsubsection{Connecting $ult_1$-OGP and $3$-fl-RDT}
\label{sec:uogpparrdt1}

Utilizing parametric RDT up to the third level of lifting (i.e., utilizing 3-fl-RDT), \cite{Stojnicalgsbp26} arrived at the following collection of estimates given in Table \ref{tab:tab19}. $\alpha^{(r)}_c(1)$ values given in Table \ref{tab:tab19} are the first three values from Table \ref{tab:tab1}. However, now they are also supplemented with $\c$ and $\p$ sequences which are the key RDT parameters of our interest here. In parallel with RDT estimates, we in Table  \ref{tab:tab19} show the very best bound on $\alpha_{ult_1}(1)$ that we obtained in previous sections. That is $\alpha_{4ogp} = \bar{\alpha}_{ult_1}(1;[1,4])\approx 1.6578$ obtained for $4$-OGP and corresponds to the first level of ultrametricity characterized by $\k$ sequence $\k=[1,4]$.
 
\begin{table}[h]
\caption{Parallel view of key $ult_1$-OGP and $3$-fl-RDT parameters estimates}\vspace{.1in}
\centering
\def\arraystretch{1.2}
 \begin{tabular}{c|c||c|c}
 \hline\hline
\multicolumn{2}{c||}{$ult_1$-OGP}    & \multicolumn{2}{c}{3-fl-RDT} \\
 \hline\hline
  \hspace{-0in}$s$                                              &  $1$       
& \hspace{-0in}$r$                                              &  $3$       
 \\ \hline\hline
   $\q=[q_0,q_1]$    &  $[1,0.984]$    
&      $\hat{\p}=[\hat{\p}_1,\hat{\p}_2]$    &  $[1,0.9852]$   
     \\ \hline\hline
   $\k=[k_0,k_1]$  &    $[1,4]$      
 &
   $\hat{\c}=[\hat{\c}_2^{(s)},\hat{\c}^{(s)}_3]$  &    $[1,4.2629]$      
 \\ \hline\hline
 $\bar{\alpha}_{ult_s}(1;[1,4])$    & \red{$\mathbf{1.6578}$} 
 &  $\alpha_c^{(r)}(1)$   & \red{$\mathbf{1.6576}$} 
  \\ \hline\hline
  \end{tabular}
\label{tab:tab19}
\end{table}

Looking at the last row of Table \ref{tab:tab19} one observes a remarkable similarity between the best upper bound 
$\bar{\alpha}_{ult_1}(1;[1,4])\approx 1.6578$ on the minimal constraint density for which  $\mathbf{\mathcal S} \lp G,\kappa,\alpha \rp$ exhibits $ult_1$-OGP on the one side, and constraint density $\alpha_c^{(3)}(1)\approx 1.6576$ associated with the third level of parametric RDT lifting on the other side. This is precisely what was noted earlier when we discussed the numerical evaluation of $4$-OGP and in particular results  presented in Table \ref{tab:tab4}. However, Table \ref{tab:tab19} allows to draw an even stronger parallel: Not only are the $\alpha$ estimates almost identical, but the same seems to apply to $\q$ and $\hat{\p}$ and $\k$ and $\hat{\c}$ sequences as well. While this could be a coincidence, we actually believe that it points to the existence of an intrinsic connection between these parameters. If such a connection indeed exists then it also likely points towards an intriunsic connection between $ult$-OGP and parametric RDT.

\subsubsection{Connecting $ult_2$-OGP and $4$-fl-RDT}
\label{sec:uogpparrdt2}

We complement the above observations (that relate to the first level of ultrametricity) with a set of similar ones from the second level. To that end, in Table \ref{tab:tab20}, we summarize 4-fl-RDT results obtained in \cite{Stojnicalgsbp26}.  In parallel with the RDT estimates, we also show the very best bound  that we obtained on $\alpha_{ult_2}(1)$. That is $\bar{\alpha}_{ult_2}(1;[1,4,12])\approx 1.6219$ obtained for ultrametric structure characterized by $\k$ sequence $\k=[1,4,12]$.

Looking at the last row of Table \ref{tab:tab20}, one again observes a remarkable similarity between the best upper bound  $\bar{\alpha}_{ult_2}(1;[1,4,12])\approx 1.6219$ on  $\alpha_{ult_2}(1)$ on the one side, and constraint density $\alpha_c^{(4)}(1)\approx 1.6218$ associated with the fourth level of parametric RDT lifting on the other side. Similarity between corresponding sequences $\q$ and $\hat{\p}$ and $\k$ and $\hat{\c}$ might be even more striking.

\begin{table}[h]
\caption{Parallel view of key $ult_s$-OGP and $r$-fl-RDT parameters estimates; $s\in\{1,2\}$, $r\in\{3,4\}$. }\vspace{.1in}
\centering
\def\arraystretch{1.2}
 \begin{tabular}{c|c|c||c|c|c}
 \hline\hline
\multicolumn{3}{c||}{$ult_s$-OGP}    & \multicolumn{3}{c}{r-fl-RDT} \\
 \hline\hline
  \hspace{-0in}$s$                                              &  $1$   &  $2$       
& \hspace{-0in}$r$                                              &  $3$   &  $4$      
 \\ \hline\hline
   $\q=[q_0,\dots,q_s]$    &  $[1,0.984]$  &  $[1,0.9989,0.9745]$    
&      $\hat{\p}=[\hat{\p}_1,\dots,\hat{\p}_{r-1}]$    &  $[1,0.9852]$     &  $[1,0.9988,0.9729]$   
     \\ \hline\hline
   $\k=[k_0,\dots,k_s]$  &    $[1,4]$   &    $[1,4,12]$      
 &
   $\hat{\c}=[\hat{\c}_2^{(s)},\dots,\hat{\c}^{(s)}_{r}]$  &    $[1,4.2629]$   &    $[1,4.1522,12.0786]$      
 \\ \hline\hline
 $\alpha_{ult_s}(1;\k)$    & \red{$\mathbf{1.6578}$}   & \red{$\mathbf{1.6219}$} 
 &  $\alpha_c^{(r)}(1)$   & \red{$\mathbf{1.6576}$}   & \red{$\mathbf{1.6218}$} 
  \\ \hline\hline
  \end{tabular}
\label{tab:tab20}
\end{table}

\subsubsection{Connecting $ult_s$-OGP and $(s+2)$-fl-RDT}
\label{sec:uogpparrdts}

Given the above similarities on the first two levels of ultrametricity, one wonders if they extend to higher levels. As we have emphasized on multiple occasions, continuing numerical evaluations beyond $s=3$ seems very challenging. Nonetheless, we have done some limited  computations for a few $\k$ that can be handled. Among those where $l\leq 12$ the best bound $\bar{\alpha}_{ult_3}(1;[1,3,6,12]) \approx 1.6131$ is obtained for $k=[1,3,6,12]$. The associated parametes estimates are shown in Table \ref{tab:tab21} as well. Also, we show in parallel $5$-fl-RDT prediction $\alpha^{(5)}_c(1)$ from \cite{Stojnicalgsbp26} together with its own associated parameters. The $\alpha$ estimates are fairly close but there seems to be a space for further improvement as well.

\begin{table}[h]
\caption{Parallel view of key $ult_s$-OGP and $r$-fl-RDT parameters estimates; $s=3$, $r=5$.}\vspace{.1in}
\centering
\def\arraystretch{1.2}
 \begin{tabular}{c|c||c|c}
 \hline\hline
\multicolumn{2}{c||}{$ult_s$-OGP}    & \multicolumn{2}{c}{r-fl-RDT} \\
 \hline\hline
  \hspace{-0in}$s$                                              &  $3$       
& \hspace{-0in}$r$                                              &  $5$   
 \\ \hline\hline
   $\q$    &     $[1,0.9992, 0.9920, 0.9580]$    
&      $\hat{\p} $    &   $[1,0.9985, 0.9965, 0.9627]$   
     \\ \hline\hline
   $\k $      &    $[1,3,6,12]$      
 &
   $\hat{\c} $       &    $[ 1,4.3528,12.7310, 29.6479 ]$      
 \\ \hline\hline
 $\alpha_{ult_s}(1;\k)$   & \red{$\mathbf{1.6131}$} 
 &  $\alpha_c^{(r)}(1)$     & \red{$\mathbf{1.6093}$} 
  \\ \hline\hline
  \end{tabular}
\label{tab:tab21}
\end{table}

While at present memory requirements restrict continuing evaluations with larger $k$, we believe that there are $\k$ configurations that indeed allow for further improvements to practically materialize. In fact, choices similar to $\k=[1,4,12,24]$ might already be good candidates (while this is still fairly small example (compared to all available options for $\k$, including those of $\infty$ size), a gigantic amount of memory might be needed to handle it).

The above reasoning and results from Tables \ref{tab:tab19}-\ref{tab:tab21} allow us to formulate the following conjecture.

\begin{conjecture} [$ult$-OGP -- parametric fl-RDT connection (\underline{weak sense})] 
\label{thm:conj2} 
Consider a statistical SBP $\mathbf{\mathcal S} \lp G,\kappa,\alpha \rp$ from (\ref{eq:ex1a0}). Let $\bar{\alpha}_{uls_s}(\kappa)$ be as in (\ref{eq:thmseq1}) (and in (\ref{eq:ult1kogpeq16}), (\ref{eq:ult23ogpeq16}), and (\ref{eq:ult33ogpeq16})) and let $\alpha^{(r)}_c(\kappa)$ be the $r$-th level parametric fl-RDT estimate of $\alpha_c(\kappa)$ as introduced in \cite{Stojnicalgsbp26}. One then has
\begin{eqnarray}\label{eq:conj2eq1}
\lim_{s\rightarrow \infty} \alpha_{uls_s}(\kappa) 
= \lim_{s\rightarrow \infty} \min_{\k}  \bar{\alpha}_{uls_s} (\kappa;\k) 
= \lim_{r\rightarrow \infty} \alpha^{(r)}_c(\kappa).
\end{eqnarray}
Together with Conjecture \ref{thm:conj1} one also has for the SBP's statistical computational gap (SCG)
\begin{equation}\label{eq:conj2scg}
 SCG = \alpha_c (\kappa) - \alpha_a (\kappa) 
 = \alpha_c^{(2)} (\kappa) -  \lim_{r\rightarrow \infty}\alpha_c^{(r)} (\kappa) 
 = \alpha_c^{(2)} (\kappa) -  \lim_{s\rightarrow \infty} \min_{\k} \bar{\alpha}_{ult_s} (\kappa;\k) 
 = \alpha_c^{(2)} (\kappa) -  \lim_{s\rightarrow \infty}\alpha_{ult_s}(\kappa) .
 \end{equation}
\end{conjecture}

The above conjecture effectively predicts that $ult_s$-OGP and parametric $r$-fl-RDT in the limit of large $s$ and $r$ have the same associated critical constraint density. Moreover, such a critical constraint density is postulated as precisely the algorithmic threshold. Somewhat remarkably, the conjecture also predicts that it is achieved by the critical constraint density associated with union-bounding analytical methodology.

The above can be viewed as a \emph{weak sense} analogy/equivalence between $ult$-OGP and parametric fl-RDT. The strong sense counterpart goes even a step further.

\begin{conjecture} [$ult$-OGP -- parametric fl-RDT connection (\underline{strong sense})] 
\label{thm:conj3} 
With the setup of Conjecture \ref{thm:conj2} one also has for any $s\geq 1$
\begin{eqnarray}\label{eq:conj3eq1}
  \alpha_{uls_s}(\kappa)  \leq  \alpha^{(s+2)}_c(\kappa).
\end{eqnarray}
(The strongest sense would actually allow even an equality in (\ref{eq:conj3eq1}) for some $s$ (including possibly all $s$).)
 \end{conjecture}

In the light of the above conjectures, we can now revisit our guess that $\k=[1,4,12,24]$ might already be a good candidate to approach $\alpha^{(5)}_c(1)\approx 1.6093$ on the third level of ultrametricity. The guess is basically motivated by the  recognition that the parametric RDT gives on the fifth lifting level $\approx 1.6093$ with $\c$ parameters roughly $\hat{\c}^{(s)} = [1, 4.3528, 12.7310, 29.6479]$. If our conjectures are correct in the $s\rightarrow\infty$ limit, then they might not be that far away for finite $s$ scenarios. One can then establish a possible 
\begin{equation} \label{eq:connect1a0}
\bl{\textbf{\emph{ult-OGP -- parametric fl-RDT isomorphism:}}}  \quad   \quad   \quad \frac{\hat{\c}^{(s)}_{i+1}}{\hat{\c}^{(s)}_{i}}\approx\frac{k_i}{k_{i-1}} \quad \mbox{ and } \quad 
\hat{\p}= \q.
\end{equation}
Keeping in mind the above mentioned $\hat{\c}^{(s)}$ obtained on the fifth level of RDT lifting, (\ref{eq:connect1a0}) would then give estimates for presumably close to optimal $\k$ sequence of the following type
\begin{equation} \label{eq:connect1a1}
\frac{k_1}{k_0}\approx \frac{\hat{\c}^{(s)}_2}{\hat{\c}^{(s)}_1}\approx 4.3528 \quad \mbox{ and } \quad 
\frac{k_2}{k_1}\approx \frac{\hat{\c}^{(s)}_3}{\hat{\c}^{(s)}_2}\approx 2.9248 \quad \mbox{ and } \quad 
\frac{k_3}{k_2}\approx \frac{\hat{\c}^{(s)}_4}{\hat{\c}^{(s)}_3}\approx 2.3288.
\end{equation}
Since elements of $\k$ are restricted to integers rounding further gives
\begin{equation} \label{eq:connect1a2}
\frac{k_1}{k_0}\approx \left [\frac{\hat{\c}^{(s)}_2}{\hat{\c}^{(s)}_1}\right ] \approx [4.3528] =4 \quad \mbox{ and } \quad 
\frac{k_2}{k_1}\approx \left [\frac{\hat{\c}^{(s)}_3}{\hat{\c}^{(s)}_2} \right ] \approx [2.9248 ] = 3 \quad \mbox{ and } \quad 
\frac{k_3}{k_2}\approx \left [\frac{\hat{\c}^{(s)}_4}{\hat{\c}^{(s)}_3}\right ]\approx [2.3288 ] =2,
\end{equation}
which precisely matches our suggestion $\k=[1,4,12,24]$. This is on the third ultrametric level. One can then proceed in a similar fashion on higher levels as well utilizing parametric RDT results. We believe that the obtained $\k$ choices will be fairly close to the optimal ones.

In the $s\rightarrow \infty$ limit, one might expect a 
\begin{equation} \label{eq:connect1a3}
\bl{\textbf{\emph{\red{complete} ult-OGP -- parametric fl-RDT isomorphism:}}}  \quad   \quad   \quad \frac{\hat{\c}^{(s)}_{i+1}}{\hat{\c}^{(s)}_{i}} = \frac{k_i}{k_{i-1}} \quad \mbox{ and } \quad 
\hat{\p}= \q.
\end{equation}
This, on the other hand, is highly unlikely to be achievable for integer sequences $\k$ that are rather small. One may be tempted to believe that scaled sequences
\begin{equation} \label{eq:connect1a4}
\k=[k_0,k_1,\dots,k_s]k_{sc} \mbox{ where } k_{sc}\gg 1,
\end{equation}
could be the fully optimal ones. Then 
\begin{equation} \label{eq:connect1a5}
\frac{\hat{\c}^{(s)}_{i+1}}{\hat{\c}^{(s)}_{i}} \rightarrow \lim_{k_{sc}\rightarrow \infty}\frac{k_ik_{sc}}{k_{i-1}k_{sc}}, 
\end{equation}
can be satisfied, by simply choosing
\begin{equation} \label{eq:connect1a6}
\forall i, \quad  k_{i}\rightarrow \hat{\c}^{(s)}_{i+1} \quad \mbox{ and }\quad  k_{i}k_{sc}\in\mZ. 
\end{equation}
This would practically mean that the single vectors that we took as the basic cluster unit in all our OGP consideration, should be replaced by a cluster of infinite length. This may then almost fully resemble the parametric fl-RDT picture where $\hat{\c}^{(s)}$ is also obtained after scaling the original $\hat{\c}$ by $\infty$. The only difference will be that fl-RDT predict partial form of lifting which implies orthogonal smallest clusters, but the effect of that might be negligible in the $s\rightarrow\infty$ and $r\rightarrow\infty$ limits.

Finally, we should point out another property that seemingly must follow from the above conjectures. An arbitrary ordering of the $\c$ sequence is allowed within the parametric RDT. On the other hand, within the OGP realm one has (by the definition of the ultrametric construction) that $\k$ sequence is increasing. Since the above $ult$-OGP -- parametric RDT isomorphism predicts that $\k$ sequence might match RDT's $\c$-sequence, one then has that the $\c$-sequence should be increasing. This is remarkable as it is directly opposite from the natural (physical) decreasing ordering. In other words, allowing for arbitrary ordering in parametric RDT might be even more than what is needed. Instead just flipping from decreasing to increasing might suffice. We complement this observation by noting that in all the problems where we have applied the parametric RDT to estimate algorithmic thresholds (including the negative Hopfield model \cite{Stojniccluphop25}, ABP \cite{Stojnicalgbp25}, and the SBP itself \cite{Stojnicalgsbp26}) we indeed observed increasing $\c$ sequences.

\subsubsection{The powerful role of advanced OGPs here and in other problems}
\label{sec:ogpother}

The early OGP like considerations appeared in replica methods and analyses of satisfiability problems \cite{AchlioptasCR11,HMMZ08,MMZ05} where the importance of understanding the geometry of the solution space was emphasized as a potentially key component in explaining efficient algorithmic solvability or lack thereof. 

Starting with \cite{GamarSud14,GamarSud17,GamarSud17a}, the interest in studying these concepts on mathematically rigorous grounds rejuvenated. Various forms of OGP have been considered throughout the literature over the last decade. An excellent introduction to basic OGP principles together with rationale for their potential algorithmic importtance/relevance can be found in an early survey \cite{Gamar21}. The early considerations \cite{Gamar21,GamarSud14,GamarSud17,GamJag21,GamarnikJW20,CGPR19} typically relied on 
classical $2$-OGPs. However, it soon became clear that they might not be general enough to always capture all algorithmically related intricacies of a potentially complex solution space. 

More advanced OGP variants were then pursued, including the so-called $m$-OGP, ensemble-OGP, or star-OGP \cite{RahVir17,GamarSud17a,GamKizPerXu22,GamKizPerXu23,Kiz23} and ladder-OGP \cite{Wein22,BreHuang21}.
To introduce and analyze a generic form of Lipshitz algorithmic hardness, \cite{HuangS22,HuangS22a} considered algorithms that satisfy overlap concentration property. Along the way they significantly extended previous OGP notions to encompass a wider branching-OGP (BOGP) class. These concepts are then utilized in a variety of ways to study statistical computational  gaps and different types of associated algorithmic hardness in various spherical \cite{HuangSell24,HuangSell23,HuangS22,HuangS22a} and Ising  \cite{HuangS22,HuangS22a} glasses. While it remains notoriously hard to prove generic algorithmic hardness, \cite{HuangSell24,HuangSell23,HuangS22,HuangS22a} consider certain types of  Lipshitz algorithms  and prove that for many spin glass problems they exhibit barriers and computational gaps (for analogous low degree polynomial hardness, albeit with a different OGP approach or through a reduction from the Lipshitz hardness, see e.g.,  \cite{Sellke25,HuangSell25}).

The ultrametric OGP that we consider here is closely related to BOGP. While BOGP is typically associated with continuous forms of overlap functionals and PDE formulations, we here utilize a discrete $ult_s$-OGP form where $s$ is an integer that governs the level of discreteness. Such a discretization is critically important component of our approach as it further allows to utilize union-bounding strategy and uncover that it potentially can be remarkably successful. Moreover, in an approximate fashion, it seems to work very well already for fairly small $s$ values and proportionally small sizes of the set of ultrametric vectors $k_s$ (smaller even than a few tens). In return, the observed success of union-bounding provides a key link that enables to connect the $ult$-OGP and parametric RDT. Since the parametric RDT has already been postulated to accurately predict the algorithmic threshold $\alpha_a(\kappa)$, the newly uncovered union-bounding $ult$-OGP -- parametric RDT potential connection allows to incorporate OGP into the predicated statistical computational gaps picture designed in \cite{Stojnicalgsbp26}. Combining this with the fact that  \cite{Stojnicalgsbp26} also connected parametric RDT to local entropy (LE), we actually believe that all three concepts, $ult$-OGP, parametric RDT, and LE ultimately (in their own limiting renderings) may very well converge to the same analytical structure at the very same critical constraint density corresponding to the algorithmic threshold $\alpha_a(\kappa)$.  This also fairly well concurs with  what was observed in $\alpha\rightarrow 0$ regime in \cite{Stojnicalgsbp26}. Namely,  on the 3rd lifting level RDT finds a convenient relation  $\kappa\approx 1.2385\sqrt{\frac{\alpha_a}{-\log\lp \alpha_a \rp}}$ which scaling-wise matches \cite{GamKizPerXu22}'s OGP based predictions and identically (including the precise values of associated constants) matches local entropy based predictions of \cite{BarbAKZ23}. In other words, our results here extend the OGP-LE connection beyond $\alpha\rightarrow 0$ regime and position parametric RDT as a mathematical glue that connects the two.

\section{Conclusion}
\label{sec:conc}

We studied SBP's ultrametric OGPs  and their potential connection to statistical computational gaps (SCGs). For $s$-level ultrametric OGPs  ($ult_s$-OGP) we rigorously upper-bound associated constraint densities $\alpha_{ult_s}(\kappa)$. The analytical methodology consists of utilizing union-bounding and adequately formulating its combinatorial and  probabilistic components as convex and nested integration programs that can be solved numerically. Conducting numerical evaluations for concrete canonical margin $\kappa=1$, we obtained bounds on the first two ultrametric levels  $\bar{\alpha}_{ult_1}(1)\approx 1.6578$ and $\bar{\alpha}_{ult_2}(1)\approx 1.6219$ and observed that they are remarkably close to the  constraint density estimates $\alpha_c^{(3)}(1)\approx 1.6576$ and $\alpha_c^{(4)}(1)\approx 1.6218$ obtained on the 3rd and 4th lifting level of parametric RDT in \cite{Stojnicalgsbp26}. Moreover, an excellent agrement of overlap values and the ultrametric clusters' relative sizes is recognized as well. Keeping these observations in mind, we propose a couple of conjectures to link $ult$-OGP, parametric RDT, and algorithmic thresholds/computational gaps: (i) $\alpha_a(\kappa)=\lim_{s\rightarrow\infty} \alpha_{ult_s}(\kappa) =\lim_{s\rightarrow\infty} \bar{\alpha}_{ult_s}(\kappa) = \lim_{r\rightarrow\infty} \alpha_{c}^{(r)} (\kappa) $; and (ii) $ \alpha_{ult_s}(\kappa) \leq \alpha_{c}^{(s+2)}(\kappa)$, for any $s\geq 1$ with possible equality for some (maybe even all) $s$.  In other words, in the limit of large ultrametricity and lifting levels, the constraint densities associated with the $ult$-OGP and parametric RDT meet at the value of the algorithmic threshold. In a stronger sense they might even meet for some finite levels as well.

We also believe that the proposed ultrametric OGP -- parametric RDT -- computational gaps connection might be a generic property that extends beyond SBPs. Along the same lines, identifying other optimization problems where the presented methodologies may be applicable seems worth of further exploration.

%
%
%
%
%
%
%

\begin{singlespace}
\bibliographystyle{plain}
\bibliography{nflgscompyxRefs}
\end{singlespace}

\end{document}